\documentclass[Afour,sageh,times]{sagej}

\newcommand\BibTeX{{\rmfamily B\kern-.05em \textsc{i\kern-.025em b}\kern-.08em
T\kern-.1667em\lower.7ex\hbox{E}\kern-.125emX}}

\usepackage{amsmath,amsfonts}
\usepackage{textcomp}
\usepackage{array}
\usepackage{stfloats}
\usepackage{url}
\usepackage{verbatim}
\usepackage{graphicx}
\usepackage[ruled,linesnumbered,vlined]{algorithm2e}
\SetKw{Continue}{continue}
\SetKw{Break}{break}
\SetKw{Return}{return}
\SetKwInOut{Parameter}{Param}
\SetKwRepeat{Do}{do}{while}%
\usepackage{acro}[=v3]

\DeclareAcronym{ipp}{
  short = {IPP},
  long  = {informative path planning}
}

\DeclareAcronym{dbs}{
  short = {DBS},
  long  = {depth-wise beam search}
}

\DeclareAcronym{nbs}{
  short = {NBS},
  long  = {node-wise beam search}
}

\DeclareAcronym{tsp}{
  short = {TSP},
  long  = {traveling salesman problem} 
}

\DeclareAcronym{spt}{
  short = {SPT},
  long  = {shortest path tree}
}

\DeclareAcronym{rrt}{
  short = {RRT},
  long  = {rapidly-exploring random tree}
}

\DeclareAcronym{rrg}{
  short = {RRG},
  long  = {rapidly-exploring random graph}
}

\DeclareAcronym{rrag}{
  short = {RRAG},
  long  = {rapidly-exploring random annulus graph}
}

\DeclareAcronym{rrat}{
  short = {RRAT},
  long  = {rapidly-exploring random annulus tree}
}

\DeclareAcronym{nbv}{
  short = {NBV}, 
  long  = {next-best view}
}

\DeclareAcronym{prm}{
  short = {PRM},
  long  = {probabilistic roadmap}
}

\DeclareAcronym{dep}{
  short = {DEP},
  long  = {dynamic exploration planner}
}

\DeclareAcronym{rig}{
  short = {RIG},
  long  = {robotic information gathering}
}

\DeclareAcronym{cpu}{
  short = {CPU},
  long  = {central processing unit}
}

\DeclareAcronym{bfs}{
  short = {BFS},
  long  = {breadth-first search}
}

\DeclareAcronym{dfs}{
  short = {DFS},
  long  = {depth-first search}
}

\DeclareAcronym{np}{
  short = {NP},
  long  = {nondeterministic polynomial}
}

\DeclareAcronym{esdf}{
  short = {ESDF},
  long  = {Euclidean signed distance field} 
}

\DeclareAcronym{hssd}{
  short = {HSSD},
  long  = {Habitat Synthetic Scenes Dataset}
}

\usepackage{booktabs}
\usepackage{multirow}
\usepackage{adjustbox}
\usepackage[table]{xcolor}
\usepackage{siunitx}
\usepackage{amssymb}
\usepackage{enumitem}
\usepackage{amsthm}
\usepackage{bm}
\newtheorem{theorem}{Theorem}
\newtheorem*{theorem*}{Theorem}
\newtheorem{lemma}{Lemma}
\newtheorem*{lemma*}{Lemma}
\newtheorem{definition}{Definition}
\newtheorem*{definition*}{Definition}

\newtheorem*{observation*}{Observation}
\newtheorem{corollary}{Corollary}
\newtheorem*{corollary*}{Corollary}
\DeclareMathOperator*{\argmax}{arg\,max}

\newcommand{\clearance}{\xi}
\newcommand{\openapero}{\texttt{OpenApero}\xspace}
\newcommand{\mathset}[1]{\mathcal{#1}}

\newcommand{\method}[1]{\textnormal{\textsc{#1}}}

\definecolor{gold}{HTML}{FFD700}    
\definecolor{silver}{HTML}{C0C0C0}  
\definecolor{bronze}{HTML}{B87333}  

\usepackage{etoolbox} 

\usepackage{subcaption}
\usepackage[colorlinks,bookmarksopen,bookmarksnumbered,citecolor=red,urlcolor=red]{hyperref}
\usepackage{cleveref}

\def\volumeyear{2026}

\begin{document}

\runninghead{Qu \textit{et~al.}}


\title{An Efficient Beam Search Algorithm for Active Perception in Mobile Robotics}

\author{Kaixian Qu, Han Wang, Victor Klemm, Cesar Cadena, Marco Hutter}

\affiliation{Robotic Systems Lab, ETH Zurich}

\corrauth{Kaixian Qu, Robotic Systems Lab, ETH Zurich}

\email{kaixqu@ethz.ch}

\begin{abstract}
Active perception is a fundamental problem in autonomous robotics in which the robot must decide where to move and what to sense in order to obtain the most informative observations for accomplishing its mission. Existing approaches either solve a computationally expensive traveling salesman problem over heuristically selected informative nodes, or adopt a more efficient but overly constrained shortest path tree formulation. To address these limitations, we explore beam search algorithms as scalable alternatives. While the standard beam search provides scalability by preserving the top-$B$ paths at each depth level, it is prone to local optima and exhibits parameter sensitivity. Our first contribution is a node-wise beam search (NBS) algorithm, which maintains top-$B$ candidates per node to enable more effective exploration of the solution space. Systematic benchmarking on graphs shows that NBS consistently outperforms other baselines and maintains strong performance even at low beam widths. As a second contribution, we integrate the concept of frontiers into the path selection criterion, introducing the expected gain metric, which better balances exploration and exploitation compared to existing alternatives. Our third contribution proposes the rapidly-exploring random annulus graph (RRAG), a novel graph construction method that preserves full orientation sampling and ensures connectivity in cluttered environments through a fallback local sampling-based planner. Extensive experiments demonstrate that NBS combined with RRAG achieves the highest performance across all three representative active perception tasks, outperforming state-of-the-art algorithms by at least 20\% in one or more tasks. We further validate the approach on real robotic platforms in different scenarios.
\end{abstract}


\maketitle

\section{Introduction}
\label{sec:Introduction}

\begin{figure*}[!t]
    \centering
    \includegraphics[width=0.99\textwidth]{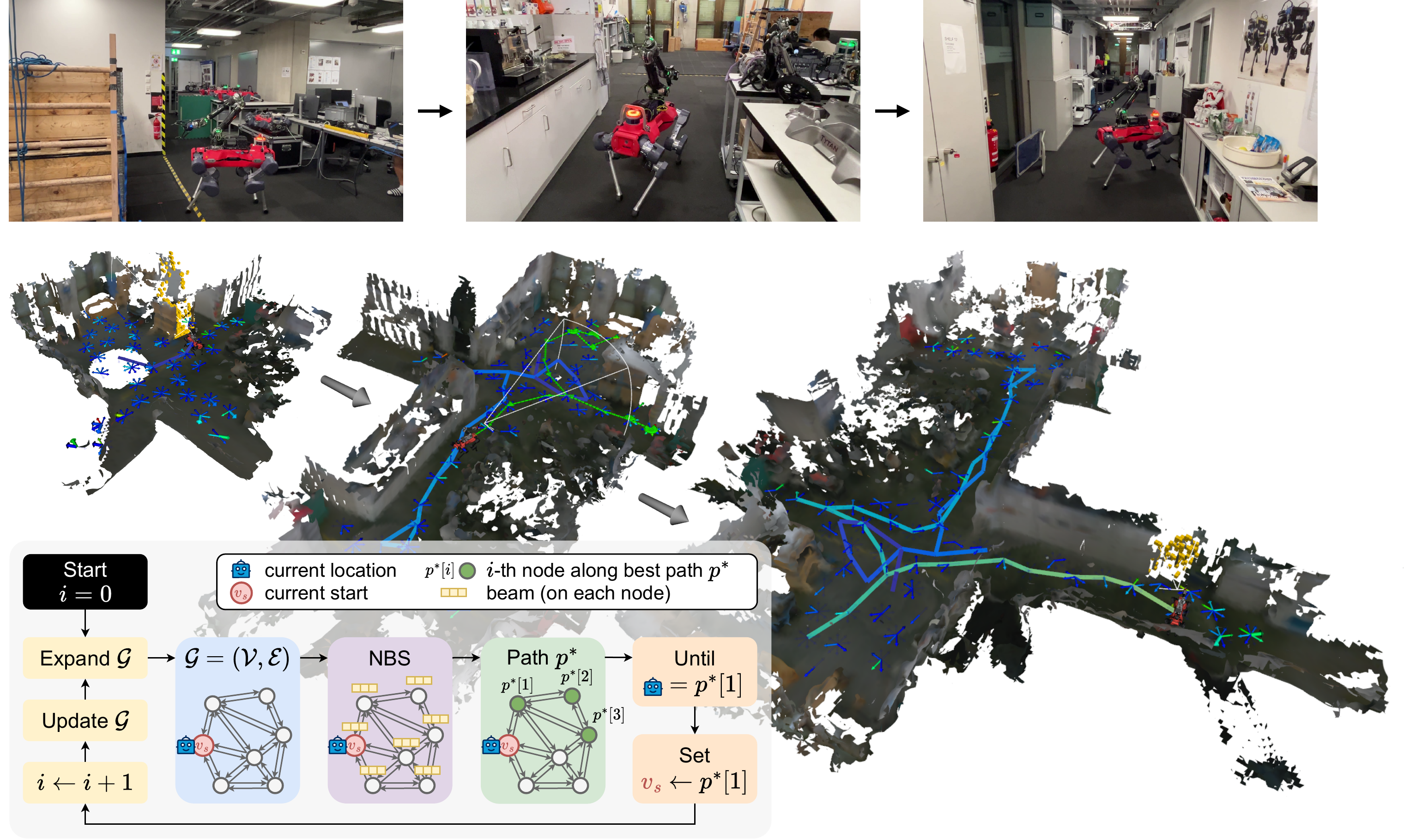}
    \caption{
    \textbf{Surface reconstruction in hardware experiments.}
    The figure illustrates the robot's navigation sequence alongside the corresponding reconstructed scene. The completed trajectory is visualized with a color gradient from dark blue to bright yellow, while the currently planned path is shown as a sequence of light green arrows. The camera's view frustum is depicted, and yellow voxels represent surface frontiers visible within it. Each node is represented by a color-coded arrow: low-gain nodes are blue, high-gain nodes are red. The bottom left panel illustrates the planning pipeline: the planner identifies the optimal path using NBS and, upon reaching the next node, updates and expands the graph with new information before computing the next-best path.
    }
    \label{fig:test_result}
\end{figure*}

Autonomous robots are increasingly deployed for tasks requiring intelligent information gathering, such as environment monitoring~\citep{singh2009efficient, cao2013multi, yu2016correlated, ghaffari2019sampling, popovic2020informative}, autonomous exploration~\citep{bircher2016receding, meng2017two, zhou2021fuel, cao2021tare}, and surface reconstruction~\citep{bircher2015structural, yoder2016autonomous, schmid2020efficient, kompis2021informed}. Unlike conventional shortest path planning, which aims to minimize traversal cost to reach a goal, these applications pose a more complex challenge: how to enable robots to autonomously plan sensor trajectories that maximize accumulated information under resource constraints such as time or energy.

Furthermore, robots often operate under partial observability, with neither the environment nor the information distribution known a priori. This motivates active perception scenarios, where the robot must simultaneously explore the environment and plan informative trajectories in the free space. Sampling-based planners such as \ac{rrt} and RRT*~\citep{kuffner2000rrt, karaman2010incremental} have been adapted to this setting~\citep{bircher2016receding, moon2022tigris}. However, tree-based structures have inherent limitations: each node stores only one incoming path, and extensive rewiring or pruning is required as the robot moves, often leading to degraded solution quality or high computational overhead.

Graph-based representations offer greater flexibility through multi-query planning capabilities. Since planning the most informative path on a graph is NP-hard~\citep{ahuja1983combinatorial}, prior work has employed simplifying heuristics, such as restricting the search to a \ac{spt}~\citep{dang2020graph, xu2021autonomous}. While computationally efficient, this severely limits the search space. Another common approach decouples node selection from routing: informative nodes are selected or sampled heuristically and then connected via shortest paths using a \ac{tsp} solver~\citep{meng2017two, arora2017randomized, zhou2021fuel}. However, this separation often leads to suboptimal paths and high computational cost.

In this work, we adopt the term \ac{ipp} to refer broadly to graph-based path planning problems that seek a path maximizing a given gain metric subject to a cost budget. Importantly, we do not address information-theoretic objectives, such as those in~\citep{singh2009efficient, lim2016adaptive, bai2016information, chen2024adaptive}, but instead treat \ac{ipp} as a path optimization problem defined over a graph.

Given the combinatorial nature of \ac{ipp}, we turn to beam search---a well-established algorithm in sequential decision-making~\citep{rush2013optimal, vijayakumar2018diverse, meister2020if}. It strikes a balance between the efficiency of greedy search and the optimality of exact inference by iteratively pruning the search space to focus only on the most promising partial solutions. We first explore the standard variant, termed \ac{dbs}, which maintains the top $B$ paths at each search depth. Although more flexible than \ac{spt}- or \ac{tsp}-based methods, this approach remains sensitive to parameter choices and prone to local optima. To address these limitations, we propose \ac{nbs}, which maintains the top $B$ paths per node. This design enhances robustness to parameter variation and enables more effective exploration of the search space.

We first benchmark path planning algorithms in two abstract graph settings: planning informative paths on a fully known graph and on a graph incrementally discovered from sensor data. Most algorithms require a path selection criterion, i.e., a metric that generalizes across both settings for ranking candidate paths. The two most commonly used metrics are path gain and path ratio (gain-to-cost ratio). To better balance exploration of new information against exploitation of known gains, we introduce the expected gain metric, which incorporates frontier~\citep{yamauchi1997frontier} information to estimate potential future gains. Through extensive evaluations, we show that \ac{nbs} consistently outperforms \ac{tsp}, \ac{spt}, and \ac{dbs} on graphs. Our proposed expected gain criterion achieves a superior balance between exploration and exploitation compared to existing metrics. We also observe that frequent replanning further improves performance, especially in online perception settings.

Transitioning from graph environments to active perception requires online and incremental graph construction. To this end, we propose \ac{rrag}, an incrementally built graph structure inspired by \ac{rrg}~\citep{karaman2011sampling}. The name annulus graph reflects its spacing rules: nodes are accepted only if sufficiently distant from existing ones, and edges are formed only when nodes are within a suitable proximity, preventing excessive node density and edge connectivity. Unlike previous incremental graph construction methods that retain a single orientation per position~\citep{schmid2020efficient, xu2021autonomous}, \ac{rrag} preserves multiple orientations, enhancing directional coverage. Moreover, while most approaches assume straight-line edges---which often fail in narrow or cluttered environments---\ac{rrag} introduces a fallback local sampling-based planner to ensure connectivity even in highly constrained spaces.

We validate our approach through comprehensive simulations on three representative active perception tasks: point collection, surface reconstruction, and volumetric exploration, across eight scenarios from the \ac{hssd} dataset~\citep{khanna2024habitat}. \ac{nbs} combined with \ac{rrag} achieves the highest performance in all three tasks, outperforming state-of-the-art methods by at least 20\% in one or more tasks. These results, complemented by real-world robotic experiments (see \Cref{fig:test_result} for an example), demonstrate the effectiveness and generalizability of our approach.
The project page is available at \url{https://efficient-beam-search.github.io/}. Alongside this work, we release \openapero (Open Active Perception Library for Robotics)\footnote{\url{https://openapero.github.io/}}, an open-source library implementing the algorithms presented in this paper and providing a modular framework for active perception research.

We summarize our main contributions as follows:
\begin{itemize}
    \item We propose \ac{nbs}, an efficient node-wise beam search algorithm that maintains a bounded set of top candidate paths per node. Extensive experiments demonstrate that NBS is robust to parameter choices and consistently outperforms state-of-the-art approaches in both graph environments and active perception tasks.
    \item We integrate frontier information into the path selection criterion by introducing the expected gain metric, which balances exploration of unknown space and exploitation of gains in free space. Empirical results show that this formulation surpasses other alternatives, such as path ratio, across a variety of graph environments.
    \item We introduce \ac{rrag}, a novel incremental graph construction method for active perception. \ac{rrag} preserves multiple orientations at each position, enabling richer path optimization. It also incorporates a fallback local sampling-based planner to ensure graph connectivity in cluttered environments.
\end{itemize}

\begin{figure*}[!t]
    \centering
    \includegraphics[width=\linewidth]{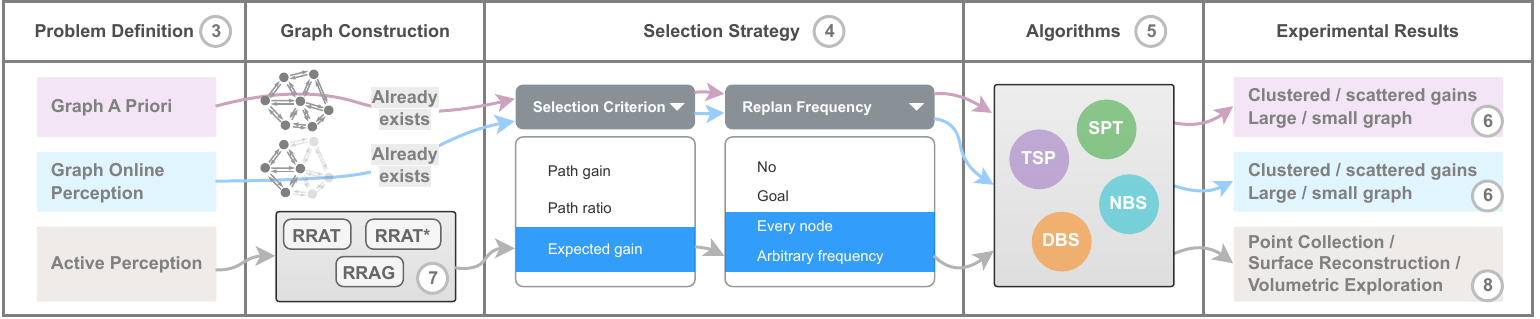}
    \caption{\textbf{Paper organization overview.} Circled numbers indicate the corresponding sections. In the graph a priori and online perception settings, the environment is represented as a graph, either fully known or partially known. In contrast, active perception necessitates graph construction from the collision-free state space. The path selection strategy encompasses both selection criteria and replanning frequencies, which are benchmarked in the graph-based settings. The best-performing combination is subsequently applied to the active perception tasks.}
    \label{fig:paper-organization}
\end{figure*}

This paper is structured to progressively build both intuition and formal understanding of the proposed approach. \Cref{sec:related-work} reviews relevant prior work in active perception, autonomous exploration, informative path planning, and sampling-based planning. \Cref{sec:problem-formulation} formulates the active perception problem across three increasingly realistic settings: planning on a fully known graph, planning on an incrementally discovered graph, and active perception with incremental environment exploration. \Cref{sec:preliminaries} introduces key components such as path selection criteria and replanning strategies, and discusses properties of optimal paths. \Cref{sec:algorithms} presents different \ac{ipp} algorithms, including the proposed \ac{nbs} algorithm, along with theoretical analysis of their computational complexity. \Cref{sec:benchmark} evaluates the proposed methods against established baselines in different graph environments. \Cref{sec:graph-construction} details \ac{rrag}, our incremental graph construction method that enables the application of \ac{ipp} algorithms to active perception. \Cref{sec:active-perception} presents empirical results from extensive simulations and real-world robotic experiments, demonstrating the effectiveness of our approach in active perception tasks. Finally, \Cref{sec:limitation} discusses limitations and future work, and \Cref{sec:conclusion} concludes the manuscript. An overview of the paper’s structure is provided in \Cref{fig:paper-organization}.

\section{Related Work}
\label{sec:related-work}

One of the most fundamental active perception tasks is autonomous exploration, in which a robot aims to explore as much of the environment as possible. A seminal approach is frontier-based exploration~\citep{yamauchi1997frontier}, where the robot repeatedly navigates to the boundaries between known free space and unknown regions to incrementally complete a map. Subsequent work has addressed the challenge of high-speed exploration on multi-rotor platforms~\citep{cieslewski2017rapid}. Building on the frontier concept, recent methods have integrated \ac{tsp}-style formulations to achieve more globally efficient exploration policies. FUEL~\citep{zhou2021fuel} incrementally maintains a frontier information structure (FIS) by computing axis-aligned bounding boxes for frontier clusters and then solving an asymmetric \ac{tsp} to generate a global tour that visits representative viewpoints. RACER~\citep{zhou2023racer} extends FUEL to multi-UAV settings and uses coverage paths to guide exploration while avoiding repeated visits to the same regions. Other works, however, argue against casting exploration directly as a \ac{tsp} problem. FAEL~\citep{huang2023fael} introduces heuristic strategies that solve a path optimization problem jointly considering movement distance, information gain, and coverage. \cite{yu2023echo} proposes ECHO, an efficient heuristic for next viewpoint determination that selects the next viewpoint via a lightweight cost function accounting for boundary awareness, environmental structure, and motion consistency, achieving both computational efficiency and improved performance.

Recently, many researchers have explored applying \ac{ipp} formulations to a wide range of active perception problems. The objective of \acl{ipp} is to compute a path that maximizes information gain while satisfying a given cost budget. Classical \ac{ipp} formulations typically assume a fully known environment represented as a weighted graph, where vertices encode information gains and edges represent traversal costs~\citep{singh2009efficient, lim2016adaptive}. A major application is environment monitoring, where Gaussian processes are commonly used to model uncertainty and guide planning to reduce terrain uncertainty~\citep{singh2009efficient, ma2018data, popovic2020informative, chen2024adaptive}. \cite{yu2016correlated} employ a mixed-integer quadratic program to solve the planning problem, while \cite{ott2024approximate} introduce an approximate sequential path optimization method that constructs trajectories segment by segment using dynamic programming. More recently, \cite{jakkala2025fully} propose a fully differentiable \ac{ipp} formulation that offers significant gains in speed and scalability. Reinforcement learning-based approaches have also attracted increasing attention~\citep{ruckin2022adaptive, cao2023catnipp}. Beyond known environments, several works extend \ac{ipp} to unknown settings for autonomous exploration and active mapping~\citep{bai2016information, ghaffari2019sampling, asgharivaskasi2022active, chen2024poam}.

To optimize information gain under a cost budget, one major line of work adopts a \ac{tsp}-based approach: sampling a set of candidate nodes and then using established \ac{tsp} solvers to determine an efficient visitation sequence. \cite{arora2017randomized} first applied \ac{tsp} to generate a route through all sampled nodes and then iteratively evaluated the effect of adding or removing nodes to further improve the solution. This method leverages mature \ac{tsp} solvers but assumes a predetermined endpoint, limiting its applicability in realistic scenarios. \cite{meng2017two} improved upon this by eliminating the dependency on a fixed endpoint and incorporating a heuristic cost function to accelerate the optimization. \cite{cao2021tare} adopted a hierarchical framework, applying \ac{tsp} both within a local planning horizon to sequence dense viewpoints and at the global level to visit sparse sub-spaces. Similar ideas appear in structural inspection and surface coverage tasks. \cite{bircher2015structural, bircher2016three} proposed generating a minimal set of viewpoints that collectively cover a structure (analogous to the Art Gallery Problem~\citep{lee1986computational}), then solved the \ac{tsp} to connect them. Rather than solving the Art Gallery Problem directly, they used multiple iterations of random sampling to refine the solution. The \ac{tsp}-based paradigm has two key limitations. First, it employs a two-stage optimization: candidate nodes are selected (or sampled) first, and only afterward is a near-optimal tour computed---sequencing and node selection are not jointly optimized under the cost constraint. Second, TSP quickly becomes computationally prohibitive as the number of candidate nodes increases, limiting its scalability.

Another mainstream approach to applying the \ac{ipp} formulation in active perception is via sampling-based planning. In this paradigm, a tree or graph is incrementally grown in the free space, with the information gain of each sampled node evaluated to ultimately extract the most informative path. Sampling-based planning has become the de facto standard for path planning in high-dimensional configuration spaces. These methods avoid the explicit construction of obstacles, instead relying on collision-checking modules to build a connectivity graph or tree. \cite{kavraki1996probabilistic} proposed the \ac{prm}, which constructs a persistent graph of the free space, enabling efficient multi-query planning. For single-query problems, the \ac{rrt} was developed~\citep{lavalle1998rapidly, lavalle2001randomized}, which utilizes a Voronoi bias to rapidly expand a tree from the start state toward the goal. Several extensions improve planning efficiency, such as RRT-Connect~\citep{kuffner2000rrt}, which grows two trees simultaneously (from the start and the goal), and expansive space trees (EST)~\citep{hsu1997path}, which bias expansion toward sparsely explored regions. To address the suboptimality of these early approaches, \cite{karaman2011sampling} introduced the \ac{rrg} and its tree-based counterpart, RRT*, both of which are asymptotically optimal. RRT* employs a rewiring mechanism to dynamically maintain the lowest-cost parent for each node while preserving a tree structure. RRT$^X$~\citep{otte2016rrtx}, another asymptotically optimal algorithm that grows from the goal toward the start, was proposed to enable real-time replanning in dynamic environments. However, these seminal sampling-based planners focus on finding the shortest path to a goal in a fully known environment and do not inherently address mapping or information-gathering objectives.

Recently, \acp{rrt} have been adapted to extend traditional motion planning algorithms to active perception tasks. \cite{bircher2016receding} adapted \acp{rrt} for exploration by sampling viewpoints to construct a tree and selecting the branch with the highest gain, with a weight affected by the cost, for execution. After each planning step, only the selected immediate branch is retained, and the tree is re-expanded. This approach has been further extended to surface inspection in unknown environments~\citep{bircher2018receding}. \cite{witting2018history} proposed sampling from pruned nodes to rapidly recover node density. \cite{schmid2020efficient} enhanced this method by rewiring the non-executed branch to the new root, enabling direct reuse of previously sampled nodes, following the connection strategy in RRT*~\citep{karaman2011sampling}. To address large and high-dimensional search spaces, \cite{kompis2021informed, moon2022tigris} introduced informed sampling strategies to sample nodes with high expected information gain. Tree-based planners inherently sacrifice optimality due to the limited number of candidate paths and the need for frequent tree updates. Additionally, shortest-path formulations typically exclude repeated nodes, whereas revisiting nodes can be beneficial.

Graph-based algorithms relax the constraint of maintaining a single path to each node by constructing a graph representation of the environment and searching for informative paths within it. Moreover, such approaches are inherently better suited for online replanning, as they do not rely on a rooted structure and thus eliminate the need for tree rewiring. \cite{xu2021autonomous, duberg2022ufoexplorer} propose incrementally constructing a graph during exploration. GBPlanner~\citep{dang2020graph} introduced a dual-graph structure for large-scale exploration, where local graphs are built within bounded regions and incrementally merged into a global graph. All these methods reduce the search space by applying shortest-path algorithms (e.g., Dijkstra) over selected nodes to generate shortest-path trees, effectively transforming the problem into a tree search embedded in an underlying graph. \cite{hollinger2014sampling} proposed RIG-roadmap and RIG-graph, in which path planning is integrated with graph construction under the assumption of prior knowledge of cost and gain. Although they incorporate a pruning strategy, they do not impose explicit bounds on graph size, which can result in poor scalability in large environments. Additionally, these methods lack fine-grained control over path selection, as graph growth and path planning are tightly coupled---i.e., path selection is performed only on the currently constructed subgraph rather than the complete graph. In this work, we focus on algorithms designed for fast online replanning, capable of handling online exploration with partial information while maintaining scalability in large environments.

\section{Problem Formulation}
\label{sec:problem-formulation}

We consider three progressively more realistic and complex formulations of this problem: (i) planning on a fixed, known graph; (ii) planning with online perception over a partially observed graph; and (iii) active perception scenarios in which the robot autonomously constructs a graph within the discovered free space to achieve its objective. We are interested in finding an algorithm that works well across all these three setups. Before delving into these formulations, we introduce the terminology used throughout the discussion.

Let $\mathset{G} = (\mathset{V}, \mathset{E})$ denote a directed, simple graph, where $\mathset{V}$ is the set of vertices and $\mathset{E} \subseteq \mathset{V} \times \mathset{V}$ is the set of directed edges, with each edge $e_{i, j} =  (v_i, v_j) \in \mathset{E}$ representing a feasible connection from vertex $v_i$ to vertex $v_j$. Throughout this manuscript, we use the terms vertex and node interchangeably.
A path $p$ is defined as a sequence of vertices $(v_0, v_1, \ldots, v_n)$ such that $(v_i, v_{i+1}) \in \mathset{E}$ for all $i$. A path is called simple if it visits each vertex at most once, i.e., all vertices in the sequence are distinct. Let $\mathset{P}$ denote the set of all paths in $\mathset{G}$.

Each edge $e_{i, j} = (v_i, v_j) \in \mathset{E}$ is assigned a positive traversal cost $c(e_{i, j}) > 0$, representing the cost of moving from vertex $v_i$ to vertex $v_j$. The cost of a path $p = (v_0, v_1, \ldots, v_n)$ is defined as the sum of the costs of its constituent edges:
\begin{equation}
    c(p) = \sum_{i = 0}^{n-1} c(e_{i, i+1}).
    \label{eq:path-cost-def}
\end{equation}
Each vertex $v \in \mathset{V}$ is associated with a non-negative information gain $g(v) \geq 0$. The total gain along a path $p = (v_0, v_1, \ldots, v_n)$ is defined by a set-based function $g_{set}$, which aggregates gains over the set of unique vertices visited:
\begin{equation}
    g(p) = g_{set}\left(\{ v_i \mid v_i \in p \}\right).
    \label{eq:path-gain-general}
\end{equation}
By definition, $g_{set}$ is invariant under repeated visits to the same vertex. With $g : \mathset{P} \to \mathbb{R}_{\geq 0}$ assigning path gains and $c : \mathset{P} \to \mathbb{R}_{> 0}$ assigning path costs, we extend the graph notation as:
\begin{equation}
    \mathset{G} = (\mathset{V}, \mathset{E}, g, c).
    \label{eq:graph-definition}
\end{equation}
An illustrative example is provided in \Cref{fig:graph_example}.

\begin{figure}[t]
    \centering
    \begin{subfigure}[b]{0.24\textwidth}
        \centering
        \includegraphics[width=\textwidth]{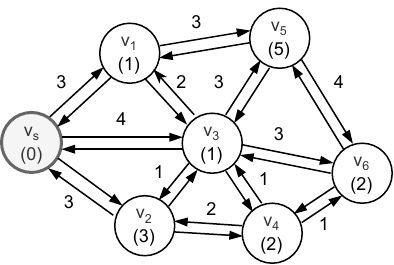}
        \caption{Constructed example graph}
        \label{fig:graph_example}
    \end{subfigure}
    \hfill
    \begin{subfigure}[b]{0.24\textwidth}
        \centering
        \includegraphics[width=\textwidth]{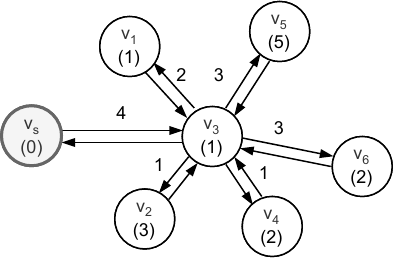}
        \caption{Example of a star graph}
        \label{fig:star_graph_example}
    \end{subfigure}
    \caption{\textbf{Examples of \ac{ipp} graphs.} Each node has an associated gain, each edge a cost, and $v_s$ denotes the start node. Revisiting nodes can be advantageous. For example, the optimal solution in (b) requires multiple traversals of~$v_3$.}
    \label{fig:illustration_graph}
\end{figure}

\subsection{Graph A Priori Formulation}
\label{sec:graph-a-priori-formulation}

We first consider the case in which the environment is fully known in advance and represented as a directed simple graph as in \Cref{eq:graph-definition}. In this formulation, the gain accumulated along a path is computed as the sum of gains from the unique vertices visited:
\begin{equation}
    g(p) = \sum_{v \in \{ v_i \mid v_i \in p \}} g(v).
    \label{eq:path-gain-sum}
\end{equation}
Note that this complies with \Cref{eq:path-gain-general}.

Let $p_k = (v_0, v_1, \dots, v_k)$ denote the sequence of vertices visited after $k$ steps. Given the fully known graph $\mathset{G} = (\mathset{V}, \mathset{E}, g, c)$, a starting vertex $v_s \in \mathset{V}$, and a cost budget $C > 0$, the objective is to determine a policy
\begin{equation}
    \pi^* : p_k \longmapsto e_{k, k+1} = (v_k, v_{k+1}) \in \mathset{E},
    \label{eq:graph-a-priori-policy}
\end{equation}
such that the resulting path $p^* = ( v_0 = v_s, v_1, \dots)$ maximizes the total accumulated gain under the budget constraint:
\begin{equation}
    p^* = \argmax_{p \in \mathset{P}} g(p) \; \text{subject to} \; c(p) \le C.
    \label{eq:optimization-objective}
\end{equation}

The initial vertex $v_0$ is fixed at $v_s$, while the final vertex is unconstrained. Paths may revisit vertices and edges, which can be advantageous (see \Cref{fig:star_graph_example}, where the robot must traverse $v_3$ multiple times to collect all gains). This formulation assumes complete a priori knowledge of the graph’s topology and associated cost and gain functions. Note that the policy $\pi^*$ may correspond to a one-shot plan computed at the beginning or a receding-horizon strategy that replans at every node. The notation here is general and accommodates all cases.

\subsection{Graph Online Perception Formulation}
\label{sec:graph-online-perception-formulation}

We next consider the case in which the graph is not known a priori and must be discovered incrementally through local perception. This online perception setting introduces significant challenges, particularly in balancing exploration and exploitation with partial information.

Let the environment be modeled as a directed simple graph $\mathset{G} = (\mathset{V}, \mathset{E}, g, c)$, where the path gain is defined as in \Cref{eq:path-gain-sum}. 
The robot perceives its surroundings within a fixed radius $\rho > 0$, assuming no occlusions. For any vertex $v \in \mathset{V}$, define its observable neighborhood as
\begin{equation}
    \mathset{N}_\rho(v) = \left\{ v_i \in \mathset{V} \mid \|v_i - v\| \le \rho \right\},
\end{equation}
where $\|v_i - v\|$ denotes the distance between $v$ and $v_i$. Similarly, we say an edge $(v_i, v_j) \in \mathset{E}$ is observable from $v$ if $v_i, v_j \in \mathset{N}_\rho(v)$.

Let $p_k = (v_0, v_1, \dots, v_k)$ denote the sequence of vertices visited after $k$ steps, with accumulated cost satisfying the budget constraint $c(p_k) \le C$ for all $k$. The cumulative discovered subgraph at step $k$ is denoted by
\begin{equation}
    \mathset{G}_k = (\mathset{V}_k, \mathset{E}_k, g_k, c_k),
\end{equation}
where
\begin{equation}
    \mathset{V}_k = \bigcup_{i=0}^{k} \mathset{N}_\rho(v_i), \;
    \mathset{E}_k = \left\{ (v_i, v_j) \in \mathset{E} \mid v_i, v_j \in \mathset{V}_k \right\},
\end{equation}
and $g_k$, $c_k$ are the restrictions of the global functions $g$ and $c$ to the current subgraph.

Given a perception radius $\rho > 0$, cost budget $C > 0$, and initial vertex $v_s \in \mathset{V}$, the objective is to determine a policy
\begin{equation}
    \pi^* : (\mathset{G}_k, p_k) \longmapsto e_{k, k+1} = (v_k, v_{k+1}) \in \mathset{E}_k,
    \label{eq:graph-online-perception-policy}
\end{equation}
that selects the next edge based on the currently discovered subgraph and executed path. The resulting path $p^*$ is chosen to maximize the total accumulated gain under the budget constraint, same as in~\Cref{eq:optimization-objective}.

This formulation reflects a setting in which the policy $\pi^*$ is inherently adaptive, continually updated with new local information. Note that when the perception radius $\rho$ exceeds the graph’s diameter, the entire graph $\mathset{G}$ becomes observable at the initial step ($k = 0$), effectively reducing this problem to the fully-known a priori case described in \Cref{sec:graph-a-priori-formulation}.

\subsection{Active Perception Formulation}
\label{sec:active-perception-formulation}

We generalize the online perception framework to a more realistic active perception setting, where a robot operates autonomously in a continuous, bounded 3D workspace using a depth sensor with limited range, completing tasks without prior knowledge of the environment. In contrast to prior formulations that assume a directed graph as the environment, in this formulation, the robot incrementally constructs a graph as an internal abstraction solely to support planning and navigation.

Let $\mathset{X} \subset \mathbb{R}^n$ denote the bounded state space and $\mathset{X}_{\mathrm{free}} \subseteq \mathset{X}$ its collision-free subset. The environment or world is represented as $\mathset{W}$, and the robot is equipped with a sensor model $\mathset{S}$ (with depth up to $\rho$). In this work, we consider an RGB-D camera, which typically has a limited field of view. For any configuration $x \in \mathset{X}_{\mathrm{free}}$, the observable portion of the environment is given by $\mathset{S}(x; \mathset{W}) \subseteq \mathset{W}$. 

The robot begins at an initial configuration $x_0 \in \mathset{X}_{\mathrm{free}}$ and executes a sequence of velocity commands under a velocity-controlled motion model. At discrete timestep $k$, the control input $u_k \in \mathset{U}$ specifies linear and angular velocities applied over a fixed duration $\Delta t$. The resulting state transition is
\begin{equation}
    x_{k+1} = x_k \oplus (u_k \cdot \Delta t),
\end{equation}
where $\oplus$ adds the integrated translation to the position and compounds the integrated rotation with the current orientation. At time step $k$, the trajectory $\tau_k$ is denoted by
\begin{equation}
    \tau_k = (x_0, u_0, x_1, u_1, \ldots, x_{k-1}, u_{k-1}, x_k).
\end{equation}
In this work, we use the term trajectory $\tau$ to denote a sequence of states and control actions generated under the robot’s velocity model with fixed timestep $\Delta t$, whereas a path $p$ refers to an abstract sequence of vertices on a graph. Each action incurs a traversal cost $c(x_i, u_i)$, leading to a cumulative trajectory cost:
\begin{equation}
    c(\tau_k) = \sum_{i=0}^{k-1} c(x_i, u_i).
\end{equation}
If the cost is time, then the cost is simply $c(x_i, u_i) = \Delta t$ and the trajectory cost is $k \Delta t$. At each step $k$, the robot integrates all previous observations for an estimated environment:
\begin{equation}
    \mathset{W}_k = \bigcup_{i=0}^{k} \mathset{S}(x_i; \mathset{W}),
\end{equation}
from which an estimated collision-free space ${\mathset{X}}_{\mathrm{free}}^k \subseteq \mathset{X}$ is inferred.
The true task-dependent objective $f(\tau; \mathset{W})$ is not available a priori; instead, the robot approximates it using an information gain function $g(\cdot)$ conditioned on $\mathset{W}_k$:
\begin{equation}
    g(\tau; \mathset{W}_k) = g_{set}(\{x_i \mid x_i \in \tau\}; \mathset{W}_k).
    \label{eq:active-perception-approximation}
\end{equation}
This function is sometimes instantiated as a sum over state-specific gains:
\begin{equation}
    g(\tau; \mathset{W}_k) = \sum_{x \in \{x_i \mid x_i \in \tau\}} g(x; \mathset{W}_k),
    \label{eq:active-perception-gain-sum}
\end{equation}
where $g(x; \mathset{W}_k) \ge 0$ denotes the estimated information gain at $x$ given the current map.

Given the sensor model $\mathset{S}$, a cost budget $C > 0$, and initial configuration $x_0$, the objective is to determine a policy
\begin{equation}
    \pi^* : (\mathset{W}_k, \mathset{X}_{\mathrm{free}}^k, \tau_k) \longmapsto u_k \in \mathset{U},
    \label{eq:active-perception-policy}
\end{equation}
that selects the next velocity informed by the current map and trajectory history. The resulting trajectory
$\tau^* = (x_0, u_0, x_1, u_1, \ldots)$
is chosen to maximize the true task-dependent objective under the cost constraint:
\begin{equation}
    \tau^* = \argmax_{\tau} f(\tau; \mathset{W}) \; \text{subject to} \; c(\tau) \le C.
    \label{eq:active-perception-objectives}
\end{equation}

In this paper, we focus exclusively on graph construction and navigation within the graph. This means that we grow the graph (both nodes and edges) within $\mathset{X}_{\mathrm{free}}^k$ and apply a velocity that helps the robot navigate on the graph. This setup is analogous to that in \Cref{sec:graph-online-perception-formulation}, wherein the trajectory cost $c(\tau)$ is computed as the corresponding path cost $c(p)$, and the trajectory gain $g(\tau)$ (\Cref{eq:active-perception-approximation}) is approximated by the path gain $g(p)$ of the traversed vertices.  Planning proceeds in discrete time, while control actions $u_k$ are selected from a continuous control space.

\section{Preliminaries}
\label{sec:preliminaries}

This section discusses the fundamental design choices of selection criteria and replanning strategies that underlie all subsequent algorithms for informative path planning. The choice of selection criterion is pivotal. In the graph a priori formulation (\Cref{sec:graph-a-priori-formulation}), the planner focuses solely on exploitation---maximizing gain within the cost constraint. In contrast, the graph online perception (\Cref{sec:graph-online-perception-formulation}) and active perception formulation (\Cref{sec:active-perception-formulation}) require balancing exploration (discovering new informative areas) and exploitation (maximizing collected gain). This work proposes a unified selection criterion designed to be effective across these settings. In addition, we investigate how the frequency of replanning influences performance, ranging from replanning at arbitrary frequency to single-shot planning (without replanning). Finally, we analyze the structure of optimal paths under symmetric graphs, offering insights for algorithmic design.

\subsection{Selection Criteria}
\label{sec:selection-criteria}

\begin{figure*}[!t]
    \centering
    \begin{subfigure}{0.5\columnwidth}
        \centering
        \includegraphics[width=\linewidth]{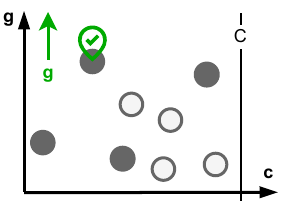}
        \caption{Path gain $g$}
        \label{fig:path_gain_illustration}
    \end{subfigure}
    \qquad
    \begin{subfigure}{0.5\columnwidth}
        \centering
        \includegraphics[width=\linewidth]{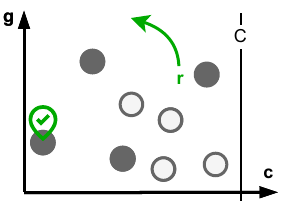}
        \caption{Path ratio $r$}
        \label{fig:path_ratio_illustration}
    \end{subfigure}
    \qquad
    \begin{subfigure}{0.5\columnwidth}
        \centering
        \includegraphics[width=\linewidth]{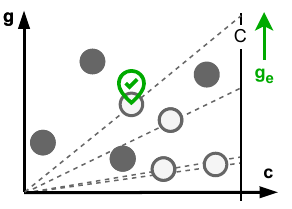}
        \caption{Expected gain $g_e$}
        \label{fig:expected_gain_illustration}
    \end{subfigure}
    \caption{\textbf{The illustration of path selection criteria.}
    Each path is represented by a circle: open circles denote paths leading to frontier nodes $\mathset{P}_{\mathset{F}}$, while solid circles represent non-frontier paths. The $x$-axis represents cost $c$, and the $y$-axis represents gain $g$. The optimal path under each criterion is highlighted with a green check mark. For the path gain criterion, the best path corresponds to the one with the highest $g$-value. Under the path ratio criterion, the optimal path is the one with the highest slope. For the expected gain criterion, the best path is the one whose extrapolated $g$-value is maximal.}
    \label{fig:objective_illustration}
\end{figure*}

The selection criterion defines a quality function of a path $ q(p) $, and the search algorithm returns the path with the highest quality. We consider three distinct path selection criteria, as illustrated in \Cref{fig:objective_illustration}: (i) path gain, (ii) gain-to-cost ratio (path ratio), and (iii) expected gain. In this work, we explicitly differentiate these selection criteria from the principal objective of informative path planning, which is formalized in \Cref{eq:path-gain-general}. The introduction of these criteria is motivated by the exploration–exploitation trade-off, particularly in settings where the graph structure or environmental model is partially or entirely unknown. In such cases, selecting paths based solely on the principal objective (path gain) may lead to suboptimal performance.

\subsubsection{Path Gain}

The most straightforward quality function is the path gain $g(p)$ as in \Cref{eq:path-gain-general}, aligning directly with the problem formulation.

\subsubsection{Path Ratio / Gain-to-Cost Ratio}

Since the full environment may not be known a priori, a commonly used heuristic is to optimize for path efficiency, defined as the gain collected per unit cost:
\begin{equation}
    r(p) = \frac{g(p)}{c(p)}.
\end{equation}
Note that this quality function implicitly assumes that the current gain-to-cost ratio provides a reliable estimate of the future gain over the entire budget, which can be formalized as:
\begin{equation}
    \argmax_{p \in \mathset{P}} r(p) = \argmax_{p \in \mathset{P}} r(p) \cdot C,
\end{equation}
where $r(p) \cdot C$ estimates the total gain achievable if the gain-to-cost efficiency remains constant.

This ratio has been proposed as the principal objective in several informative path planning frameworks~\citep{popovic2017online, schmid2020efficient, xu2021autonomous}. However, it does not constitute a well-defined optimization problem. Once the path with the highest ratio is identified, the robot has no incentive to continue beyond that path, as doing so would reduce the overall ratio. In contrast, practical scenarios typically impose a maximum cost budget, which should ideally be fully utilized for collecting maximum gains. 

\subsubsection{Expected Gain}
To explicitly account for exploration, we propose the expected gain, which incorporates frontier information:
\begin{equation} 
    g_e(p) =
    \begin{cases}
        r(p) \cdot C, & \text{if } p \in \mathset{P}_{\mathset{F}} \\
        g(p), & \text{otherwise}
    \end{cases}
    \label{eq:expected-gain}
\end{equation}
Here, $\mathset{F}$ represents the set of frontier nodes, i.e., nodes that lead to graph expansion, and $\mathset{P}_{\mathset{F}}$ is the set of paths ending at frontier nodes. For frontier paths, we approximate the potential future gain by assuming that after the current path is finished, the robot can continue from the frontier node into the newly discovered region, collecting gains with the current ratio $r(p)$ over the remaining cost budget. For all other paths, the expected gain defaults to the path gain $g(p)$. This quality function balances exploitation and exploration: it favors exploitation in fully explored graphs while incentivizing exploration when frontier regions remain. Note that maximizing \Cref{eq:expected-gain} is equivalent to
\begin{equation}
\max_{p \in \mathset{P}} g_e(p) = 
    \max \left\{ 
        \max_{p \in \mathset{P}_{\mathset{F}}} r(p) C,\ 
        \max_{p \notin \mathset{P}_{\mathset{F}}} g(p)
    \right\}.
\end{equation}

A natural question that arises is whether the extrapolation in \Cref{eq:expected-gain} should use the maximum gain-to-cost ratio across all frontier paths instead of individual path-specific ratios. The following theorem establishes that both approaches yield the same path selection. Consequently, the optimal path can be determined without precomputing the maximum ratio, reducing computational overhead while preserving optimality.
\begin{theorem}
\label{thm:ratio_equivalence}
Let $\tilde{\mathset{P}}$ be a set of candidate paths, and let $C > 0$ denote the available budget. Define $r^*$ to be the maximal gain-to-cost ratio among the candidate paths. Then
\begin{equation}
    \argmax_{p \in \tilde{\mathset{P}}} r(p)\,C
    = \argmax_{p \in \tilde{\mathset{P}}} g(p) + r^* \left(C - c(p) \right).
\end{equation}
\end{theorem}

\begin{proof}
We prove this theorem by showing the path that achieves the highest ratio $r^*$ maximizes both path quality functions. Define, for any $p \in \tilde{\mathset{P}}$,
\begin{equation}
    q_1(p) = r(p)\,C, \quad
    q_2(p) = g(p) + r^* \left(C - c(p) \right).
\end{equation}
Let $p^{*}$ maximize $r$, meaning $r(p^*) = r^*$.
Since $C > 0$ is constant,
\begin{equation}
    q_1(p) \leq r^* C,
\end{equation}
with equality if and only if $r(p) = r^*$. 
Rewriting $q_2(p)$,
\begin{equation}
    q_2(p) = r^* C + \left( g(p) - r^* c(p) \right).
\end{equation}
The term $r^* C$ is constant, so maximizing $q_2(p)$ is equivalent to maximizing $g(p) - r^* c(p)$. Observe that
\begin{equation}
    g(p) - r^* c(p) = c(p) \left( r(p) - r^* \right) \leq 0,
\end{equation}
with equality if and only if $r(p) = r^*$. Thus, $q_2(p) \leq r^* C$, with equality if and only if $r(p) = r^*$.
Therefore, both objectives are maximized by the same set of paths.\hfill \qedsymbol
\end{proof}

\begin{corollary}
Let $\mathset{P}_{\mathset{F}}$ denote the set of frontier paths, and let $r_{\mathset{F}}^* = \max_{p \in \mathset{P}_{\mathset{F}}} r(p)$ be the highest gain-to-cost ratio among these paths. Then,
\begin{equation}
    \argmax_{p \in \mathset{P}_{\mathset{F}}} r(p)\,C 
    = \argmax_{p \in \mathset{P}_{\mathset{F}}}g(p) + r_{\mathset{F}}^* \left( C - c(p) \right).
\end{equation}
This shows that extrapolating each frontier path using its own gain-to-cost ratio yields the same optimal path as using the maximum frontier ratio $r_{\mathset{F}}^*$.

Additionally, the maximum expected gain $\max_{p \in \mathset{P}} g_e(p)$ can be expressed as:
\begin{equation}
    \max \left\{ 
        \max_{p \in \mathset{P}_{\mathset{F}}} g(p) + r_{\mathset{F}}^* \left(C - c(p) \right),\ 
        \max_{p \notin \mathset{P}_{\mathset{F}}} g(p)
    \right\},
\end{equation}
or equivalently,
\begin{equation} 
    \max_{p \in \mathset{P}} g(p) + \mathbf{1}_{\mathset{P}_{\mathset{F}}}(p) \cdot r_{\mathset{F}}^* \left( C - c(p) \right),
    \label{eq:expected-gain-another}
\end{equation}
where $\mathbf{1}_{\mathset{P}_{\mathset{F}}}(p)$ is the indicator function that equals 1 if $p \in \mathset{P}_{\mathset{F}}$ and 0 otherwise.
\end{corollary}

Notably, our formulation \Cref{eq:expected-gain} allows the optimal path to be determined without explicitly computing $r_{\mathset{F}}^*$ in advance, thus reducing computational overhead.

\subsubsection{Complexity Class of Selection Criteria}

Consider the case where the path gain is defined as the sum of node gains along a path. Under this formulation, maximizing the path gain reduces to the classical longest path problem in a weighted graph, which is known to be NP-hard~\citep{hardgrave1962relation}. In contrast, maximizing the gain-to-cost ratio corresponds to the maximum ratio path problem~\citep{radzik2013fractional}. \cite{ahuja1983combinatorial} showed that the longest path problem can be reduced to a special case of the maximum ratio path problem. Consequently, any algorithm that solves the maximum ratio path problem can also solve the longest path problem, implying that the former is NP-hard. As the expected gain integrates both path gain and path ratio, its complexity subsumes theirs. Therefore, maximizing the expected gain is also NP-hard.

It is important to note that while the maximum ratio cycle problem is solvable in polynomial time~\citep{dantzig1966finding, lawler1972optimal}, this variant does not accommodate start-vertex constraints and is therefore not directly applicable to most informative path planning scenarios. That problem can be cast as detecting negative cycles~\citep{lawler1972optimal, cherkassky1999negative} or solved using linear programming~\citep{dantzig1966finding}, but it does not generalize to the path-based formulations considered in this work.

\subsection{Replanning Strategies}

When an underlying graph structure is available, the robot navigates by planning and executing paths over this graph. It may (i) follow a precomputed path without any replanning, (ii) replan only upon reaching the terminal node of the current path, or (iii) replan at each visited node along the trajectory, as illustrated in \Cref{eq:graph-a-priori-policy} and \Cref{eq:graph-online-perception-policy}. In active perception scenarios, where planning and navigation are grounded in an internally constructed graph, the same three replanning strategies are applicable. Moreover, since the robot operates under velocity control, replanning can, in principle, occur at any time step. For example, the robot may execute a partial trajectory (e.g., five steps along the current path) before reevaluating its policy at an intermediate configuration. The frequency of replanning introduces a trade-off between computational cost and responsiveness to new sensory inputs.

\subsubsection{No Replanning}
At the start node, the planner computes an optimal path 
\begin{equation}
p^* = (v_0, v_1, \ldots, v_{n-1}, v_n).
\end{equation}
The robot then executes this plan to completion. This approach is computationally efficient but does not incorporate newly acquired information during execution and may underutilize the available cost budget.

\subsubsection{Replanning at the Goal}
Once the robot completes the current optimal path $p^*$, the remaining cost budget is updated, and the gain from previously visited nodes is zeroed:
\begin{equation}
    C \gets C - c(p^*), \quad g(v_i) \gets 0, i = 1, 2, \ldots, n.
\end{equation}
The planner then computes a new optimal path based on the updated budget and updated graph. This strategy is typical in frontier-based exploration, where the robot selects the most promising frontier and replans upon reaching it~\citep{yamauchi1997frontier, xu2021autonomous, gervet2023navigating}.

\subsubsection{Replanning at Every Node}
Instead of executing the entire path, the robot moves to the next node $v_1$ along edge $e_{0, 1}$ and then replans. The cost budget is updated, and the gain from the visited node is zeroed:
\begin{equation}
    C \gets C - c(e_{0, 1}), \quad g(v_1) \gets 0.
\end{equation}
The planner then computes a new optimal path from $v_1$. This strategy enhances adaptability to new information but incurs a higher computational cost.

\subsubsection{Replanning at Arbitrary Frequency}

In active perception settings, the robot constructs and updates its internal graph representation online, based on sensory observations. At a given configuration, the planner computes an optimal path $p^* = (v_0, v_1,\ldots, v_n)$, or 
\begin{equation}
    p^* = (x_0 = v_0, x_1, x_2, \ldots, x_N),
\end{equation}
if we consider all the configurations along the edge. The time difference between consecutive configurations $x_{i}$ and $x_{i+1}$ is $\Delta t$. Suppose the robot replans every $k\Delta t$; it would then replan upon reaching configuration $x_k$. The cost budget is updated, and the gain from the traversed configurations is zeroed:
\begin{equation}
    C \gets C - \sum_{i=0}^{k-1} c(x_i, u_i), \quad g(x_i) \gets 0, i = 1,2,\ldots, k.
\end{equation}
This strategy enables replanning at arbitrary time steps, unconstrained by discrete node transitions. Such flexibility allows for the highest replanning frequency and, consequently, the most rapid adaptation to newly acquired information, albeit at the expense of maximal computational cost.

\subsection{On the Structure of Optimal Paths}

A key challenge in informative path planning is that the optimal path may revisit vertices (see \Cref{fig:star_graph_example}), leading to an infinite number of candidate paths. However, we show that it suffices to search only for paths that do not revisit edges in a symmetric directed graph, significantly reducing the search space.

\begin{theorem}
Let $\mathset{G} = (\mathset{V}, \mathset{E}, g, c)$ be a symmetric directed graph, meaning that for every $e_{i, j} = (v_i, v_j) \in \mathset{E} $, the reverse edge $e_{j,i} = (v_j, v_i) \in \mathset{E} $ also exists and satisfies $ c(e_{i, j}) = c(e_{j,i}) > 0 $. Suppose the gain of a path is determined solely by the set of nodes it visits (\Cref{eq:path-gain-general}),  and its cost is defined as the sum of edge costs along the path (\Cref{eq:path-cost-def}). Then, for any path that traverses an edge more than once, there exists an alternative path that achieves the same gain with strictly lower cost and does not revisit any edge.
\end{theorem}

\begin{proof}
Let $p$ be a path that traverses the edge $e_{m,n} = (v_m, v_n)$ more than once. Then $p$ can be decomposed as
\begin{equation}
    p = p_{a} + e_{m,n} +  p_{b} + e_{m,n} + p_{c},
\end{equation}
where $p_{a}$ and $p_{c}$ are (possibly empty) subpaths, and $p_{b}$ is a non-empty subpath from $v_n$ back to $v_m$. The $+$ operator here denotes path concatenation, defined only when the terminal node of the first segment matches the initial node of the second. For example, 
\begin{multline}
    (v_0, v_1, v_2) + (v_2, v_3) + (v_3, v_4, v_5) \\
    = (v_0, v_1, v_2, v_3, v_4, v_5),
\end{multline}
illustrates valid concatenations: the edge $(v_2, v_3)$ extends the path ending at $v_2$, and the subsequent path $(v_3, \dots)$ continues from $v_3$.

We construct a modified path $\tilde{p}$ by removing both occurrences of the edge $e_{m,n}$ and replacing the subpath $p_{b}$ with its reversal $\tilde{p_{b}}$. The resulting path is:
\begin{equation}
    \tilde{p} =p_{a} + \tilde{p_{b}} + p_{c},
\end{equation}
which remains a valid path since $\tilde{p_{b}}$ starts at $v_m$ and ends at $v_n$, and all reversed edges are present due to symmetry. This construction is illustrated in \Cref{fig:symmetric_directed_graph_proof}.

\begin{figure}[!t]
    \centering
    \includegraphics[width=0.48\textwidth]{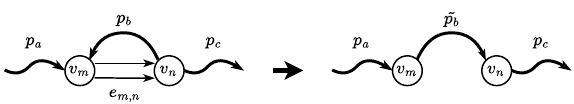}
    \caption{\textbf{Repeated edges can be removed to construct a lower-cost path with identical gain.}}
    \label{fig:symmetric_directed_graph_proof}
\end{figure}

Because $\tilde{p}$ visits the same set of nodes as $p$, its total gain remains unchanged. The cost, however, is strictly reduced due to the removal of both traversals of $e_{m,n}$:
\begin{equation}
    g(\tilde{p}) = g(p), \quad c(\tilde{p}) = c(p) - 2c(e_{m,n}) < c(p).
\end{equation}

By applying this transformation iteratively to each repeated edge, we progressively eliminate all edge repetitions. Although reversing a subpath may introduce new repetitions, each application strictly reduces the total number of edge traversals by two. Since the initial path is finite, the process terminates after a finite number of steps, yielding a path with no repeated edges.

Thus, any path that revisits an edge can be transformed into a trail (i.e., a path with no repeated edges) with identical gain and strictly lower cost, completing the proof.\hfill \qedsymbol
\end{proof}

This theorem directly leads to the following corollary.

\begin{corollary} \label{thm:no_revisit_edge}
In informative path planning over symmetric directed graphs, it suffices to restrict the search space to trails---that is, paths that do not traverse any edge more than once.
\end{corollary}

\section{Informative Path Planning Algorithms}
\label{sec:algorithms}

For informative path planning, we consider two widely used baseline algorithms: one based on the \ac{tsp} and the other on the \ac{spt}. To overcome their limitations and improve computational efficiency, we propose beam search algorithms for solving the \ac{ipp} problem. We first implement the standard beam search, referred to as \acl{dbs} (\ac{dbs}), which retains the top~$B$ partial solutions at each search depth for further expansion. We then introduce a modified variant, \acl{nbs} (\ac{nbs}), which retains up to~$B$ candidate paths per visited node, rather than per depth level. This modification enhances the planner’s ability to explore the graph more globally and avoid becoming trapped in local optima.

To compare the potential of different paths, we define a path-preference relation that first prioritizes the gain-to-cost ratio, then the path gain, and finally penalizes higher cost when all other factors are equal.

\begin{definition}[Path Preference Relation]
Let $p_1$ and $p_2$ be two candidate paths. 
We say that $p_{1}$ is preferred over $p_{2}$, denoted $p_{1} \succ p_{2}$, if one of the following conditions holds:
\begin{enumerate}
    \item $r(p_{1}) > r(p_{2})$; 
    \item $r(p_{1}) = r(p_{2})$, $g(p_{1}) > g(p_{2})$; 
    \item $r(p_{1}) = r(p_{2})$, $g(p_{1}) = g(p_{2})$, $c(p_{1}) < c(p_{2})$.
\end{enumerate}
\end{definition}

\subsection{Existing Algorithms}

\subsubsection{Traveling Salesman Problem Approach}
\label{sec:tsp}

A common foundation underlying various \ac{tsp}-based approaches~\citep{arora2017randomized, cao2021tare} is a two-stage pipeline: first, a subset of nodes is sampled or selected from the full set for traversal; second, the resulting \ac{tsp} instance is solved using an off-the-shelf solver. In the context of informative path planning, where the goal is to maximize information gain, the most natural strategy is to select nodes whose gain exceeds a threshold $g_{\mathrm{thr}}$. To explicitly encourage exploration, the selected set should further include frontier nodes:
\begin{equation}
    \mathset{V}_{s} = \mathset{F} \;\cup\;  \{ v \in \mathset{V} \mid g(v) > g_{\mathrm{thr}} \}.
    \label{eq:node-selection}
\end{equation}
The gain threshold is determined using a mid-range criterion:
\begin{equation}
    g_{\mathrm{thr}} = g_{\max} - \alpha \cdot (g_{\max} - g_{\min}),
    \label{eq:gain-threshold}
\end{equation}
where $\alpha$ specifies the fraction of top-gain nodes to retain, and $g_{\min}$ and $g_{\max}$ denote the minimum and maximum node gains.

Given the selected nodes from the previous stage, we first construct a cost matrix representing the shortest distances between every pair of nodes. We compute this matrix using the Floyd–Warshall algorithm~\citep{cormen2022introduction}, which finds all-pairs shortest paths with a time complexity of $\Theta(|\mathset{V}|^3)$ and a space complexity of $\Theta(|\mathset{V}|^2)$. Afterward, we solve the TSP to obtain the optimal tour $p^*$ using Google OR-Tools~\citep{ortools}. As TSP is NP-hard, exact dynamic‐programming approaches such as the
Bellman–Held–Karp algorithm~\citep{bellman1962dynamic,held1962dynamic}  become computationally intractable for large instances, requiring $\Theta(|\mathset{V}|^{2}\,2^{|\mathset{V}|})$ time and
$\Theta(|\mathset{V}|\,2^{|\mathset{V}|})$ space. Consequently, heuristic methods (e.g.\ Tabu Search~\citep{glover1998tabu}, Simulated Annealing~\citep{kirkpatrick1983optimization}, or Guided Local Search~\citep{voudouris1997guided}) are commonly employed to provide near-optimal solutions in practice. In our implementation, we allow OR-Tools to automatically select the solver. Since the complexity of the selected heuristic solver is not explicitly controlled, we report the lower bound given by the Floyd–Warshall preprocessing step and the
upper bound given by the exact Bellman–Held–Karp algorithm.

\subsubsection{Shortest Path Tree Approach}
\label{sec:spt}

The shortest path tree approach~\citep{xu2021autonomous} comprises two computational stages: (1) constructing the \ac{spt}, and (2) selecting the optimal path within the SPT. We employ the Bellman--Ford algorithm~\citep{cormen2022introduction} to compute the SPT that encodes the minimum-cost path from the source to every node. This step dominates the computational complexity, requiring $O(|\mathset{V}||\mathset{E}|)$ time, while the space complexity is $O(|\mathset{V}|)$ for storing distances and predecessors. Since the SPT is a tree, the number of paths to be evaluated equals $|\mathset{V}|$. This algorithm generalizes prior tree-based \ac{ipp} approaches using RRT~\citep{bircher2016receding, witting2018history} and RRT${}^*$~\citep{schmid2020efficient}.

\subsection{Depth-wise Beam Search}

Beam search is a variant of breadth-first search that selectively retains the most promising partial solutions at each depth level. In the standard formulation, a fixed parameter $B \in \mathbb{N}$, known as the beam width, determines the number of candidate paths preserved at each iteration. We refer to this method as \ac{dbs}, as it limits the number of paths expanded per depth. Unlike classical beam search, which typically terminates upon reaching a goal state, informative path planning lacks a well-defined endpoint since an optimal path can end at any node. To regulate the search, we introduce a search depth $D \in \mathbb{N}$ to constrain the search space, which also mirrors the look-ahead horizon in practical applications.

\begin{algorithm}
    \caption{Depth-wise Beam Search $(\textsc{dbs})$}
    \label{alg:dbs}
    \KwIn{$\mathset{G} = (\mathset{V}, \mathset{E})$, start node $v_s$, cost budget $C$}
    \Parameter{beam width $B$, search depth $D$, path quality $q$}
    $p^* \gets (v_s)$, $\bm{open} \gets \{(v_s)\}$\\
    \For {$d = 1 : D$} {
        $\bm{heap} \gets \emptyset$ \\
        \ForEach {$p \in \bm{open}$} {
            $v \gets p.\method{back}()$ \\
            \ForEach {$(v, v') \in \delta^+(v)$} {
                \If{$p.\method{traversed}(v, v')$}{
                    \Continue
                }
                $p' \gets p + (v, v')$ \\
                \If {$c(p') > C$} {
                    \Continue
                }
                \If {$q(p') > q(p^*)$} {
                    $p^* \gets p'$
                }
                \eIf {$\bm{heap}.\method{size}() < B$} {
                    $\bm{heap}.\method{push}(p')$
                }{
                    \If {$p' \succ \bm{heap}.\method{top}()$} {
                        $\bm{heap}.\method{pop}()$ \\
                        $\bm{heap}.\method{push}(p')$
                    }
                }
            }
        }
        $\bm{open} \gets \bm{heap}$ \\      
    }
    \Return $p^*$ \\
\end{algorithm}

\Cref{alg:dbs} outlines the DBS procedure for identifying the most informative path in a given graph. The search begins with a beam $\bm{open}$, which maintains the top $B$ paths at each depth level. During expansion, a min heap $\bm{heap}$---a priority queue structure---is used to store candidate paths, where the least preferred path (worst-ranked) is accessible via $\bm{heap}.\method{top}()$. Each path $p$ in the current beam is expanded using all outgoing edges $\delta^+(v)$ of its last vertex $v$. Expansion is performed through path concatenation, denoted by $+$. To prevent revisiting edges already traversed (as justified in~\Cref{thm:no_revisit_edge}), we check whether the edge $(v, v')$ is present in $p$. If the newly extended path $p'$ remains within the cost budget $C$, we evaluate its quality $q(p')$ under different selection criteria (\Cref{sec:selection-criteria}) and update the current best path $p^*$ if $p'$ is superior. The min heap is updated based on the path preference operator $\succ$, ensuring that at most $B$ paths are retained. If the heap has available capacity, the new path is directly inserted; otherwise, it replaces the worst-ranked path if it is preferred. This replacement operation maintains the min-heap property, guaranteeing that $\bm{heap}.\method{top}()$ always provides the least promising path for efficient pruning. Once all paths at the current depth are processed, the contents of $\bm{heap}$ are transferred to $\bm{open}$ to initialize the next beam. The search proceeds for $D$ iterations, ultimately returning the best path found.

\ac{dbs} exhibits a time complexity of $O(|\mathset{V}| D B (D + \log B))$, which stems from its iterative deepening structure and the operations executed at each search level. The algorithm performs $D$ iterations, and at each depth, it explores up to $B$ paths by extending them with a neighboring vertex (with at most $|\mathset{V}|$ possibilities), resulting in $O(|\mathset{V}| D B)$ path expansions. Each path extension involves revisitation checks and path construction, both incurring a cost of $O(D)$. The cost, gain, and ratio computations are performed in constant time, leveraging cached values for all paths in the beam $\bm{open}$. Heap operations ($\textsc{push}$, $\textsc{pop}$) for maintaining the beam require $O(\log B)$ time. Accounting for all these operations leads to the stated time complexity. 
The space complexity of \ac{dbs} is $O(DB)$, as it maintains a beam $\bm{open}$ of up to $B$ paths, each of length at most $D$.

\subsection{Node-wise Beam Search}

Depth-wise beam search prioritizes expanding the most promising paths at each depth level, but this makes it easily fall prey to local optima.  Once the entire beam converges on regions with high immediate gains, it will neglect other promising candidates that initially appear less favorable but could ultimately lead to better solutions. In the worst case, all paths in the beam may even terminate at the same node. To address this, we propose an alternative algorithm, \ac{nbs}. Instead of maintaining the top $B$ path candidates per depth, NBS stores up to $B$ paths for each visited node. This modification ensures that at least one promising path remains available at every visited node, encouraging exploration across the graph rather than focusing on immediate high-gain regions.

\begin{algorithm}
    \caption{Node-wise Beam Search $(\textsc{nbs})$}
    \label{alg:nbs}
    \KwIn{$\mathset{G} = (\mathset{V}, \mathset{E})$, start node $v_s$, cost budget $C$}
    \Parameter{beam width $B$, search depth $D$, path quality $q$}
    $p^* \gets (v_s)$, $\bm{open}[v_s] \gets \{(v_s)\}$ \\
    \For {$d = 1 : D$} {
        $\bm{heap} \gets \emptyset$ \\
        \ForEach {$v \in \mathset{V}$} {
            \ForEach {$p \in \bm{open}[v]$} {
                \ForEach {$(v, v') \in \delta^+(v)$} {
                    \If{$p.\method{traversed}(v, v')$}{
                        \Continue
                    }
                    $p' \gets p + (v, v')$ \\
                    \If {$c(p') > C$} {
                        \Continue
                    }
                    \If {$q(p') > q(p^*)$} {
                        $p^* \gets p'$
                    }
                    \eIf {$\bm{heap}[v'].\method{size}() < B$} {
                        $\bm{heap}[v'].\method{push}(p')$ \\
                    } {
                        \If {$p' \succ \bm{heap}[v'].\method{top}()$} {
                            $\bm{heap}[v'].\method{pop}()$ \\
                            $\bm{heap}[v'].\method{push}(p')$ \\
                        }
                    }
                }
            }
        }

        $\bm{open} \gets \bm{heap}$ \\      
    }
    \Return $p^*$ \\
\end{algorithm}

\Cref{alg:nbs} outlines the NBS procedure. The search maintains a per-node beam, $\bm{open}[v]$, storing up to $B$ paths that lead to node $v$. At each iteration, all nodes with stored paths are expanded: for each path $p$ ending at node $v$, we extend it to all successors $v'$, provided the edge $(v, v')$ has not been traversed in $p$. Throughout this process, we update the best path found according to the given selection criterion $q$. The structure $\bm{heap}$ serves as a temporary map to collect these extended paths; specifically, each $\bm{heap}[v']$ is a min-heap that tracks the least promising candidate (accessible via $\bm{heap}[v'].\method{top}()$) to enforce the beam limit. Other than this key difference, NBS follows a similar expansion strategy to \ac{dbs}, progressing in increasing depth from $1$ to $D$.

To better analyze the computational complexity, we can adjust the implementation by iterating the depth in the outer loop, then iterating over the edges $(v, v') \in \mathset{E}$, and finally iterating over the paths inside $\bm{open}[v]$ in the inner loop. In this way, we can see that there are in total $|\mathset{E}| D B$ iterations. Inside each iteration, the complexity is the same as that of \ac{dbs}; therefore, the total time complexity is $O(|\mathset{E}| D B (D + \log B))$. For NBS, $\bm{open}$ contains a list of up to $B$ paths for each vertex $v \in \mathset{V}$. Each path can have at most $D$ nodes, resulting in a space complexity of $O(|\mathset{V}| D B)$.

\begin{table*}[!t] 
  \centering
  \small
  \caption{\textbf{Asymptotic time and space complexity of the evaluated algorithms.}}
  \label{tab:complexity}
  \rowcolors{2}{white!100}{gray!15}
  \begin{tabular}{c|c|c|c|c|c}
    \toprule
    Alg.     & Time comp.           & Space comp.    & Num. of paths & Time comp. per path  & Space comp. per path  \\ 
    \midrule
    TSP   & $ \Omega(|\mathset{V}|^{3}), O(|\mathset{V}|^{2}\,2^{|\mathset{V}|})$   & $ \Omega(|\mathset{V}|^2)$, $O(|\mathset{V}|\,2^{|\mathset{V}|})$    &  -               &  -                   &  -                    \\
    SPT   & $O(|\mathset{V}||\mathset{E}|)$               & $O(|\mathset{V}|)$            & $|\mathset{V}|$            & $O(|\mathset{E}|)$             & $O(1)$                \\
    DBS   & $O(|\mathset{V}| D B (D + \log B))$ & $O(D B)$            & $|\mathset{V}| D B$        & $O(D + \log B)$      & $O(1 / |\mathset{V}|)$          \\
    NBS   & $O(|\mathset{E}| D B (D + \log B))$ & $O(|\mathset{V}| D B)$        & $|\mathset{E}| D B$        & $O(D + \log B)$      & $O(|\mathset{V}| / |\mathset{E}|)$        \\
    \bottomrule
  \end{tabular}
\end{table*}

\subsection{Computational Complexity}

\Cref{tab:complexity} presents a comparative summary of the computational characteristics of TSP, SPT, DBS, and NBS. For each algorithm, we report the time and space complexity, the maximum number of candidate paths considered during the search, and the average time and space complexity per path.

For the SPT approach, since the tree contains exactly one path to each vertex, the number of evaluated paths is $|\mathset{V}|$. This yields an average time complexity of $O(|\mathset{E}|)$ and a space complexity of $O(1)$ per path. In the case of TSP, as the number of enumerated candidate paths is not explicitly bounded, per-path statistics are not applicable. DBS exhibits a per-path time complexity of $O(D + \log B)$ and a per-path complexity of $O(1 / |\mathset{V}|)$. On the other hand, NBS exhibits the same per-path time complexity but incurs a different space complexity of $O(|\mathset{V}| / |\mathset{E}|)$.

Among all methods, TSP exhibits the highest asymptotic time and space complexity with the upper bound. While direct comparison between SPT and beam search methods in terms of overall computation time is challenging due to their dependence on beam width and search depth, their per-path complexities offer clearer insight. In most scenarios, $D$ and $\log B$ are expected to be significantly smaller than $|\mathset{E}|$, and thus both DBS and NBS exhibit lower per-path time and space complexity than SPT. Although NBS incurs higher total asymptotic complexity than DBS for a fixed beam width, it often demonstrates superior empirical performance under equivalent computational budgets. This advantage stems from its node-wise beam design, which promotes broader exploration and more effective graph coverage, yielding high-quality solutions with fewer path expansions (i.e., smaller $B$). We provide empirical support for this claim in \Cref{sec:benchmark}.

\section{Benchmarking on Graphs}
\label{sec:benchmark}

\begin{figure*}[!t]
\centering
\begin{subfigure}{0.45\textwidth}
\includegraphics[width=\textwidth]{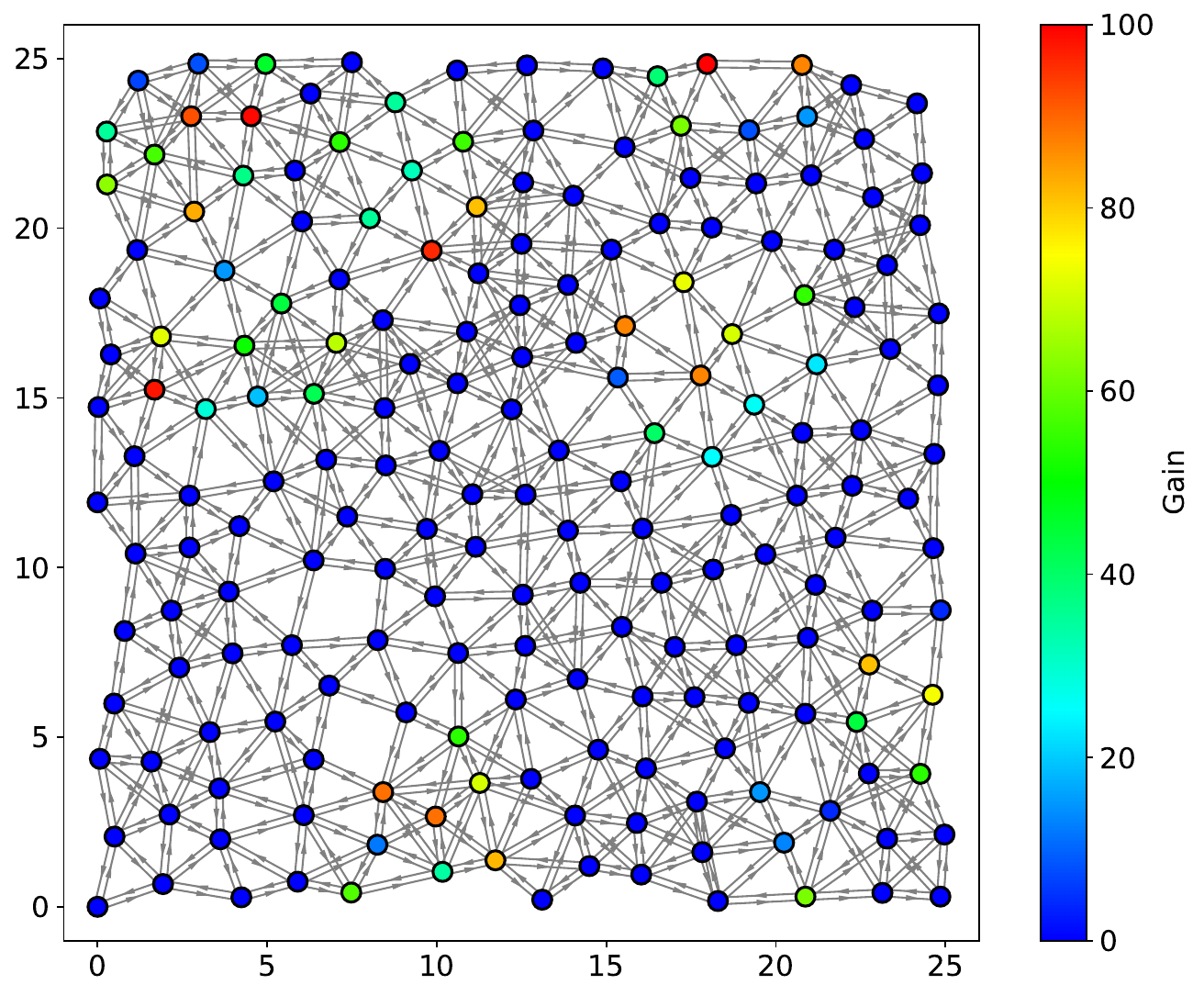}
\caption{Example of a clustered-gain graph of size $25 \times 25$}
\label{fig:small_clustered_gain_graph}
\end{subfigure}
\begin{subfigure}{0.45\textwidth}
\includegraphics[width=\textwidth]{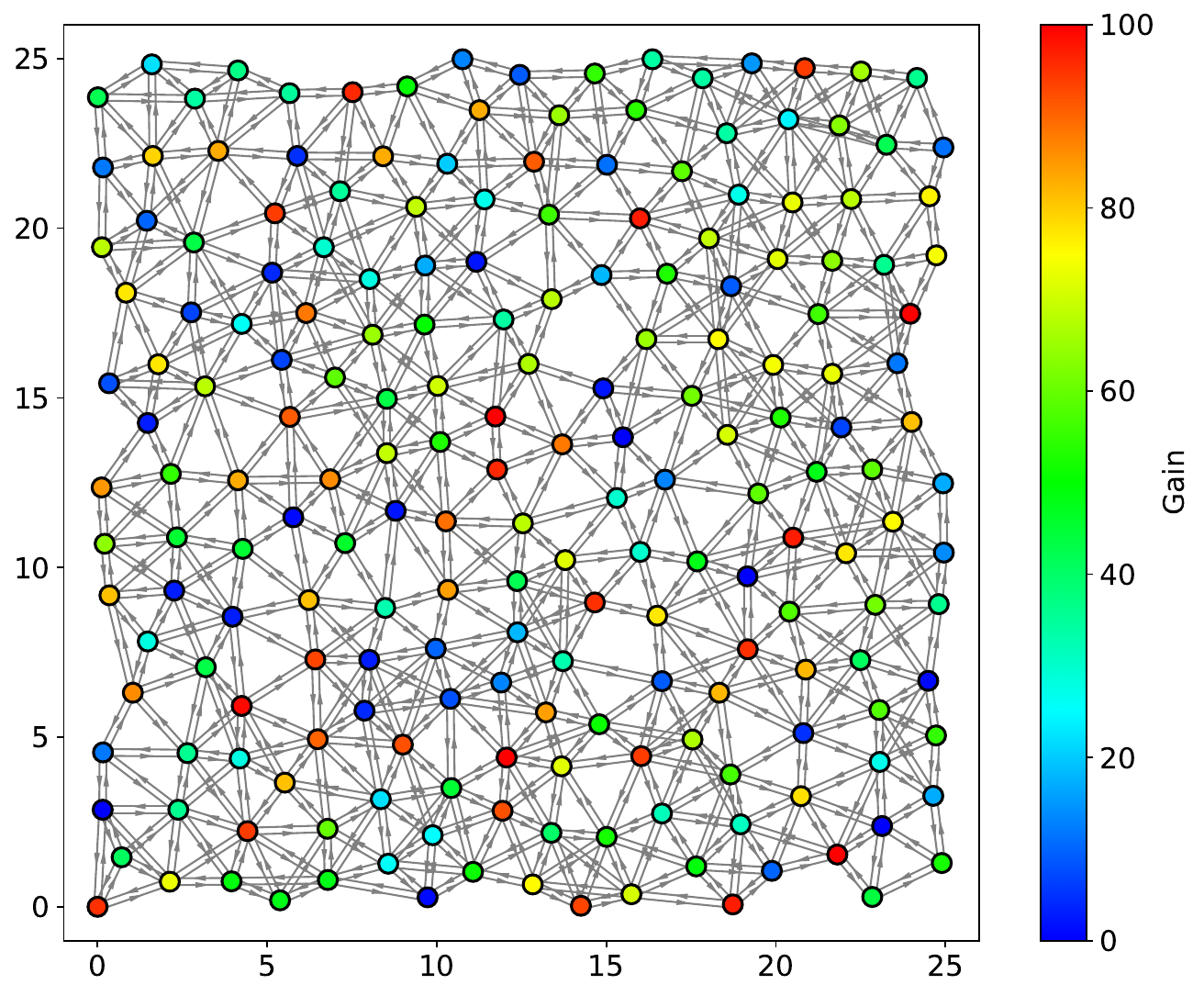}
\caption{Example of a scattered-gain graph of size $25 \times 25$}
\label{fig:small_scattered_gain_graph}
\end{subfigure}
\begin{subfigure}{0.45\textwidth}
\centering
\includegraphics[width=\textwidth]{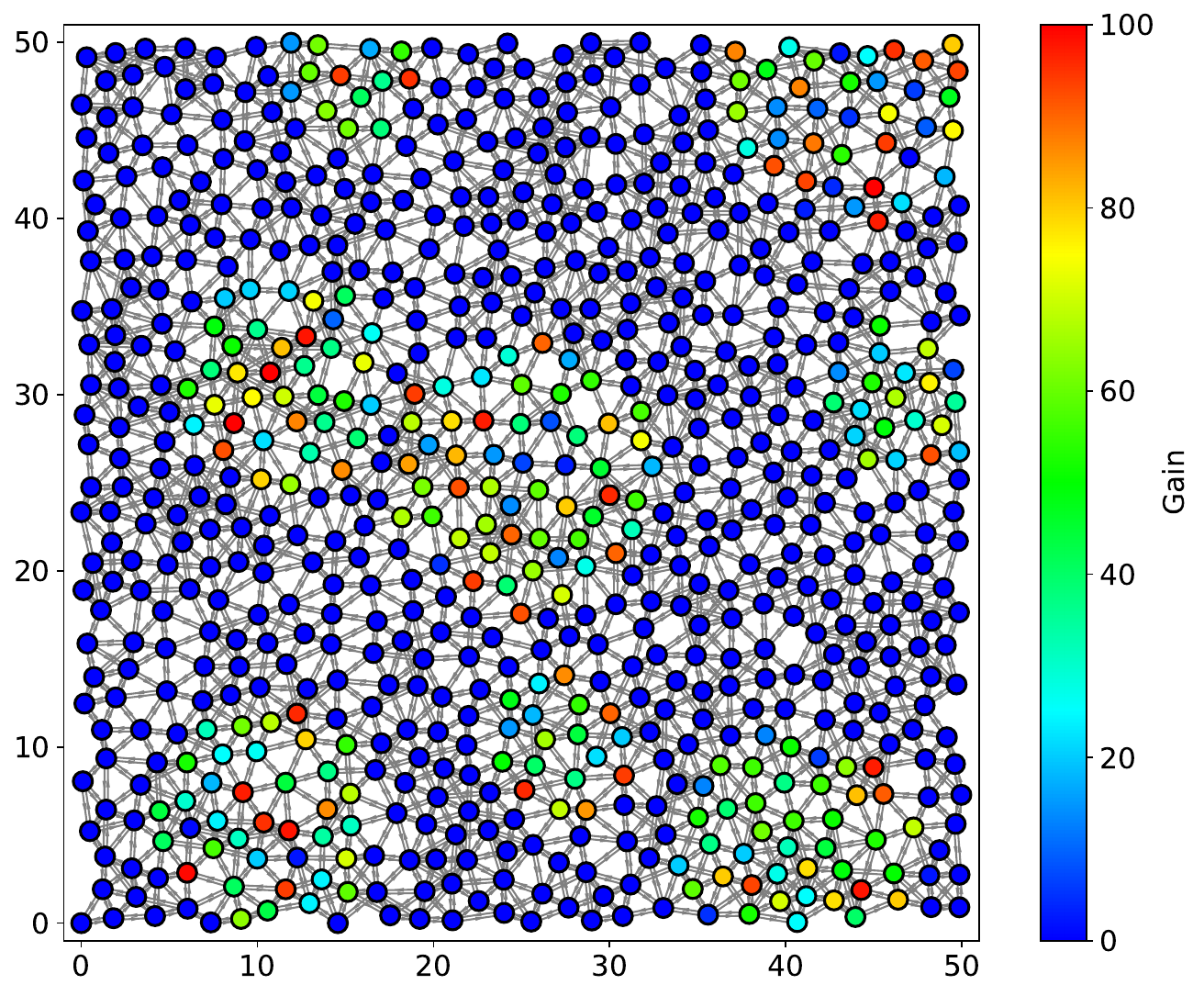}
\caption{Example of a clustered-gain graph of size $50 \times 50$}
\label{fig:large_clustered_gain_graph}
\end{subfigure}
\begin{subfigure}{0.45\textwidth}
\centering
\includegraphics[width=\textwidth]{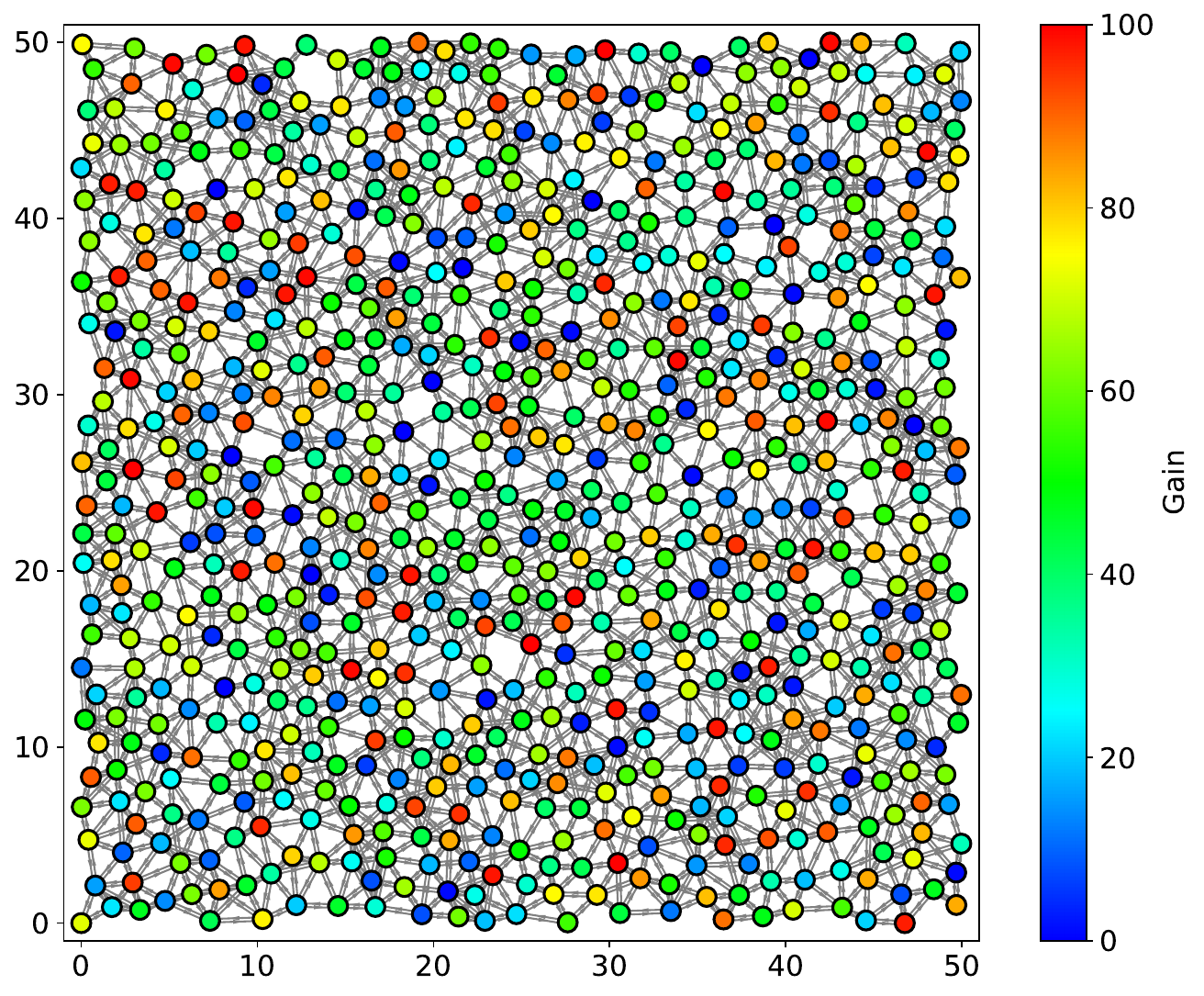}
\caption{Example of a scattered-gain graph of size $50 \times 50$}
\label{fig:large_scattered_gain_graph}
\end{subfigure}
\caption{\textbf{Illustrations of four types of graphs used for benchmarking.}}
\label{fig:graph_examples}
\end{figure*}

We evaluate all four algorithms, each with a range of key parameters. For beam search, the primary parameter is the beam width~$B$. The maximum beam width is chosen so that DBS and NBS have comparable computation times. For DBS, the beam widths are~$1$,~$100$, and~$10000$; for NBS, they are~$1$,~$10$, and~$100$. For TSP, the relevant parameter is~$\alpha$, as defined in~\Cref{eq:node-selection} and~\Cref{eq:gain-threshold}. While the shortest path tree typically does not involve tunable parameters, in~\cite{xu2021autonomous} the same parameter~$\alpha$ is used to reduce the search space and computation time. We set~$\alpha$ to~$0.5$,~$0.75$, and~$1.0$.  
All algorithms are implemented in C++, run in a single-threaded environment, and use the same graph representation to ensure a fair comparison. Experiments are conducted on an Intel i9-12900 CPU.

We evaluate each planner on two graph structures: clustered gains and scattered gains (examples in \Cref{fig:graph_examples}). Both structures are generated at two scales: a small graph spanning \SI{25}{\meter} in the $x$- and $y$-directions, and a large graph spanning \SI{50}{\meter}. In the scattered scenario, vertex gains are uniformly sampled from $[0,100]$. In the clustered scenario, these gains are assigned only to eight randomly placed clusters (all other vertices receive zero gain). Edge costs equal Euclidean distances, with the start vertex fixed at $(0,0)$. For each graph type, we generate five instances and average the results for gain and computation time.

We consider both sensing scenarios: the a priori case (\Cref{sec:graph-a-priori-formulation}), where the robot has full knowledge of the graph from the beginning, and online perception (\Cref{sec:graph-online-perception-formulation}), where the robot observes only nodes within a given perception radius~$\rho$, gradually discovering the graph as it navigates. We evaluate all applicable pairings of selection criteria and replanning strategies. In the a priori scenario, the expected gain equals the path gain ($g = g_e$) since no frontiers exist, and the ratio objective is inapplicable in the absence of replanning. Thus, five path selection combinations are feasible in this case. In the online perception scenario, replanning is necessary, making the no-replanning condition invalid; this results in six combinations per graph type.

In the plots, the $y$-axis represents different path selection pairings, while the $x$-axis shows either the range of computation time or the final accumulated gain across parameter settings. \Cref{fig:computation_time} presents the average computation time of various algorithms under different parameters. For this metric, we consider only the fully known graph case, as planning under exploration introduces significant uncertainties that can affect planning time. \Cref{fig:gains_cost_aprior} and \Cref{fig:gains_cost_onlineperception} present the results for the a priori and online perception scenarios, respectively. We evaluate algorithmic performance along two principal dimensions: (i) the maximum achievable gain, represented by the rightmost extent of each horizontal bar, and (ii) sensitivity to parameter settings, indicated by the bar's length. An algorithm is considered best if it consistently yields high gain with low sensitivity, corresponding to a short bar positioned to the far right. We prioritize the algorithm that achieves higher gain over those with merely lower sensitivity. Beyond identifying the top-performing algorithms, our analysis also aims to determine the most effective path selection combination across both formulations.

\subsection{Graph A Priori Benchmarking}

\begin{figure*}[!t]
\centering
\begin{subfigure}{0.49\textwidth}
    \centering
    \includegraphics[width=0.9\textwidth]{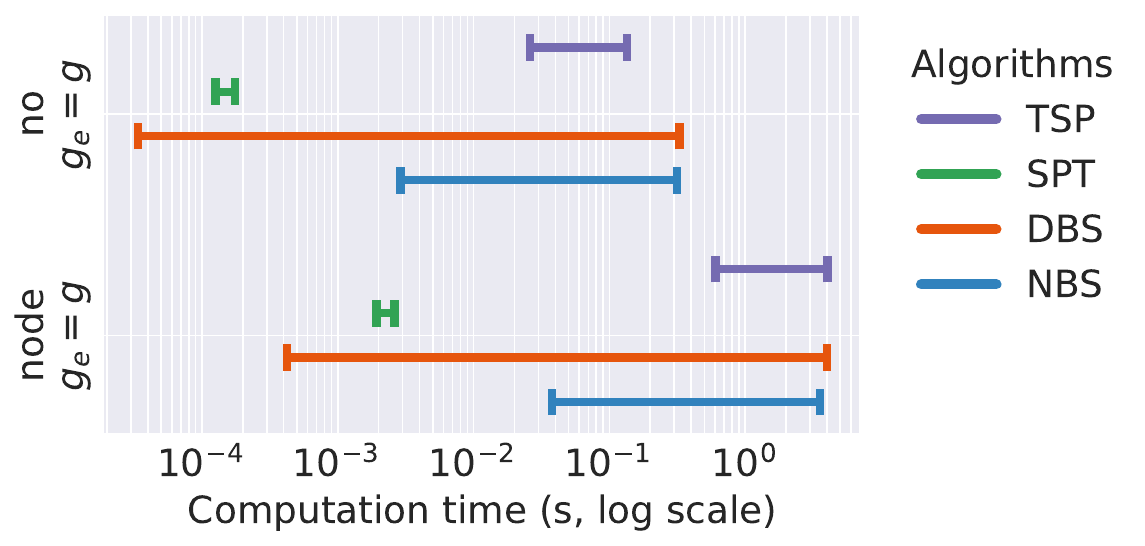}
    \caption{Average computation time in small scattered-gain graph}
    \label{fig:scattered_gain_computation_time_small_graph}
\end{subfigure}
\hfill
\begin{subfigure}{0.49\textwidth}
    \centering
    \includegraphics[width=0.9\textwidth]{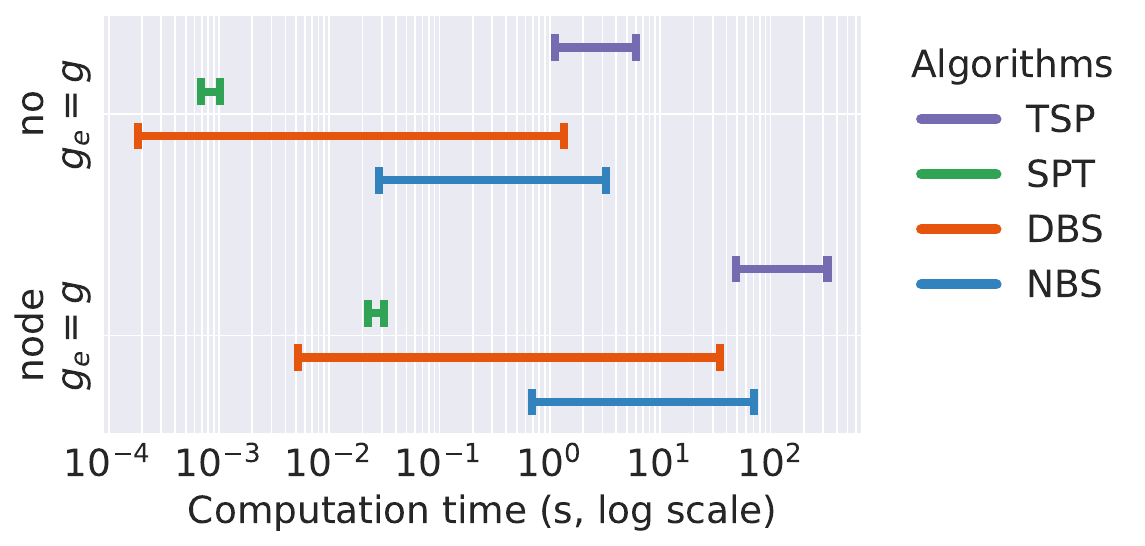}
    \caption{Average computation time in large scattered-gain graph}
    \label{fig:scattered_gain_computation_time_large_graph}
\end{subfigure}
\caption{\textbf{Average computation time in the graph a priori experiments.}}
\label{fig:computation_time}
\end{figure*}

\begin{figure*}[!t]
  \centering
  \begin{subfigure}{0.49\textwidth}
      \centering
      \includegraphics[width=0.8\textwidth]{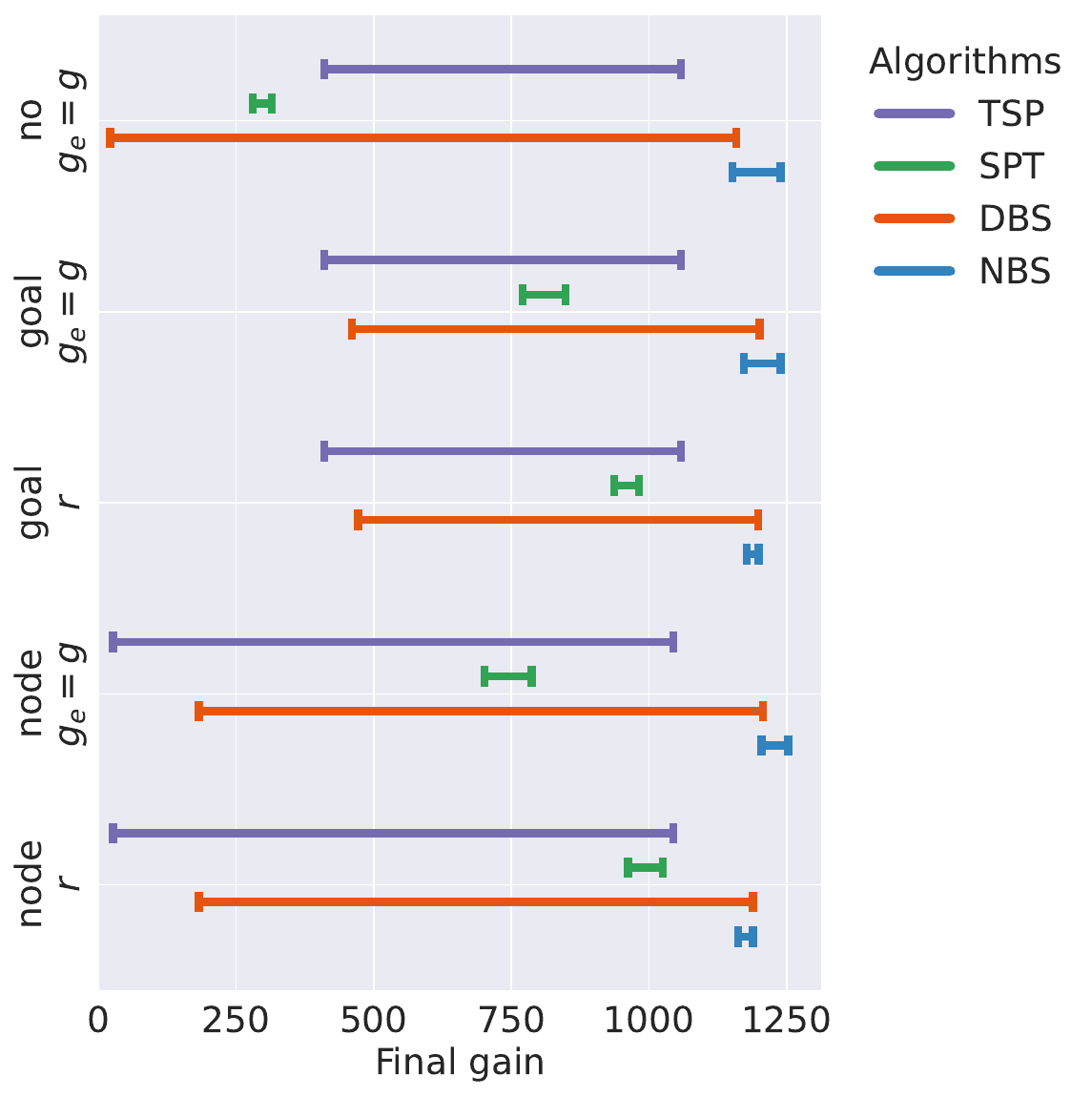}
      \caption{Clustered gains a priori (cost = 50, size $25 \times 25$)}
      \label{fig:clustered_gain_cost50_aprior}
  \end{subfigure}
  \hfill
  \begin{subfigure}{0.49\textwidth}
      \centering
      \includegraphics[width=0.8\textwidth]{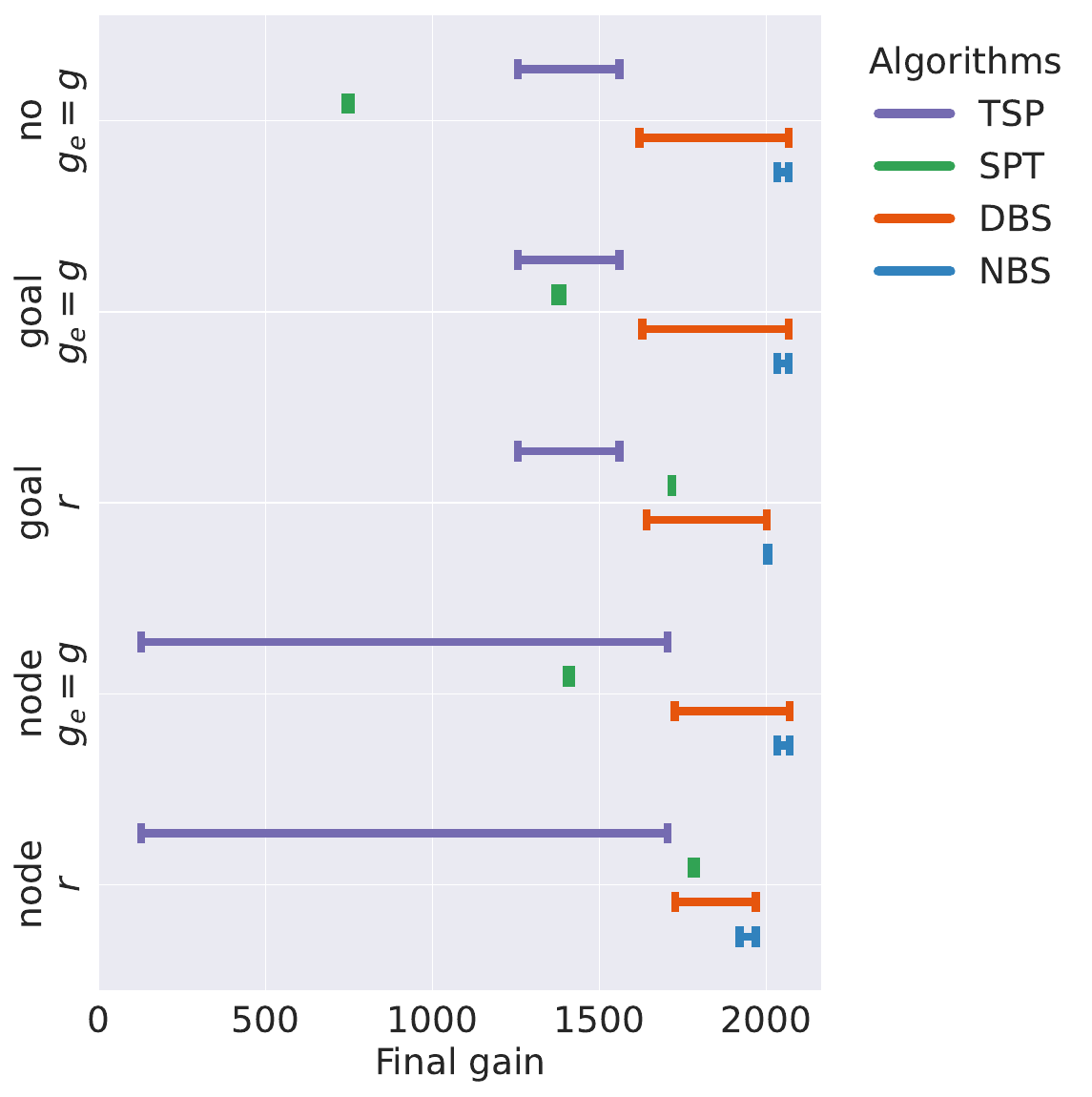}
      \caption{Scattered gains a priori (cost = 50, size $25 \times 25$)}
      \label{fig:scattered_gain_cost50_aprior}
  \end{subfigure}
  \begin{subfigure}{0.49\textwidth}
    \centering
    \includegraphics[width=0.8\textwidth]{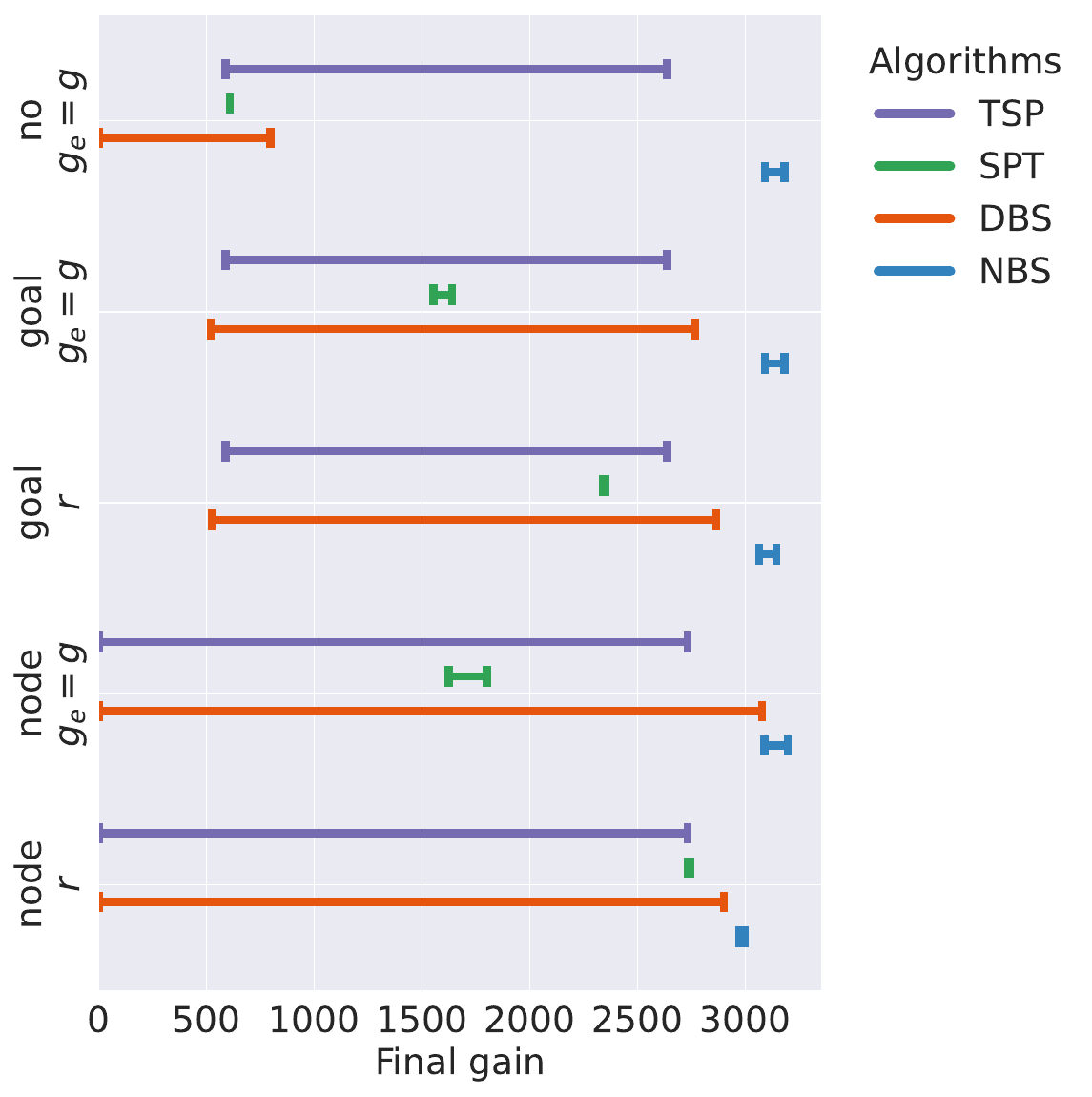}
    \caption{Clustered gains a priori (cost = 100, size $50 \times 50$)}
    \label{fig:clustered_gain_cost100_aprior}
  \end{subfigure}
  \hfill
  \begin{subfigure}{0.49\textwidth}
      \centering
      \includegraphics[width=0.8\textwidth]{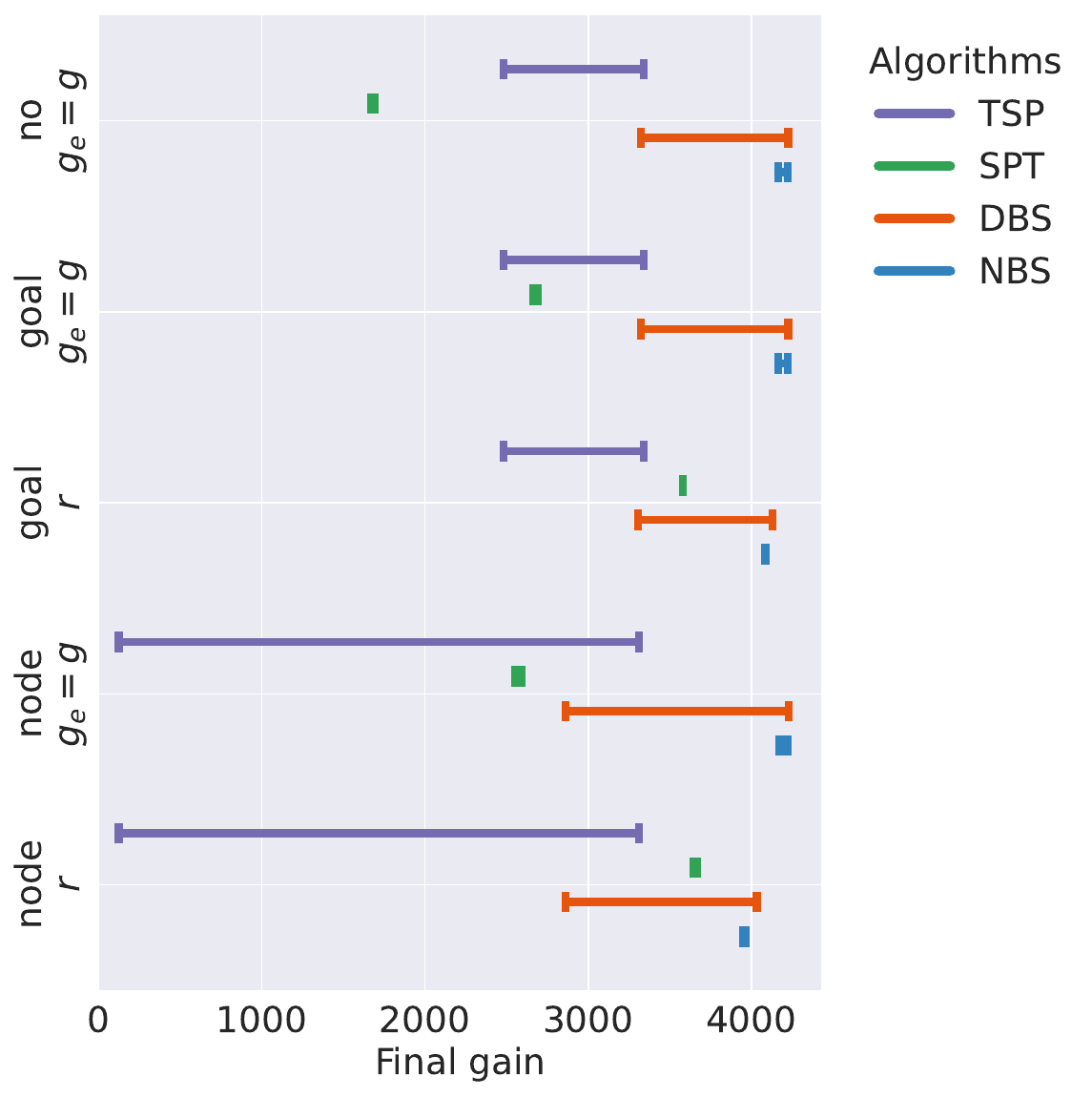}
      \caption{Scattered gains a priori (cost = 100, size $50 \times 50$)}
      \label{fig:scattered_gain_cost100_aprior}
  \end{subfigure}
  \caption{\textbf{The range of final collected gain in the graph a priori scenario.} NBS achieves the highest and most consistent gains, as shown by short horizontal bars positioned to the far right.}
  \label{fig:gains_cost_aprior}
\end{figure*}

\Cref{fig:computation_time} shows the average computation time per algorithm, in seconds, plotted on a logarithmic scale. We consider two replanning strategies: no replanning and replanning at every node; the ``replan-at-goal'' case is excluded due to its inconsistent replanning frequency. Results are shown only for scattered gain graphs, as clustered graphs exhibit qualitatively similar trends. We observe that computation time increases with both graph size and replanning frequency. TSP scales the worst, exceeding 300 seconds per episode on large graphs, whereas SPT remains consistently efficient. DBS achieves very low planning times at small beam widths. NBS is slightly slower than DBS at a beam width of 1, but its absolute computation time remains small: approximately 0.1 seconds for small graphs and about 1.0 second for large graphs under the every-node replanning strategy.

We visualize algorithm performance in the a priori formulation in \Cref{fig:gains_cost_aprior}. NBS consistently achieves the best performance across all conditions. It exhibits low sensitivity to both parameter settings and path selection, achieving high gains---as evidenced by short horizontal bars positioned to the far right. While DBS can potentially achieve comparable gains, its performance is more sensitive to both parameter and path selection setups. For example, it performs much worse in the clustered gain graph without replanning. Across most settings, the long horizontal bars indicate that DBS's performance is highly dependent on careful parameter tuning. SPT is insensitive to its parameters but shows considerable sensitivity to path selection. Without replanning, SPT fails to function effectively and, overall, achieves much lower gains than DBS and NBS. TSP lacks an explicit selection criterion, and its performance varies only with the replanning strategy. It is the most sensitive algorithm: while it can occasionally match SPT’s performance, it often achieves the lowest gains due to repeatedly traversing the same edge back and forth.

\begin{figure*}[!t]
\centering
\begin{subfigure}{0.49\textwidth}
    \centering
    \includegraphics[width=0.8\textwidth]{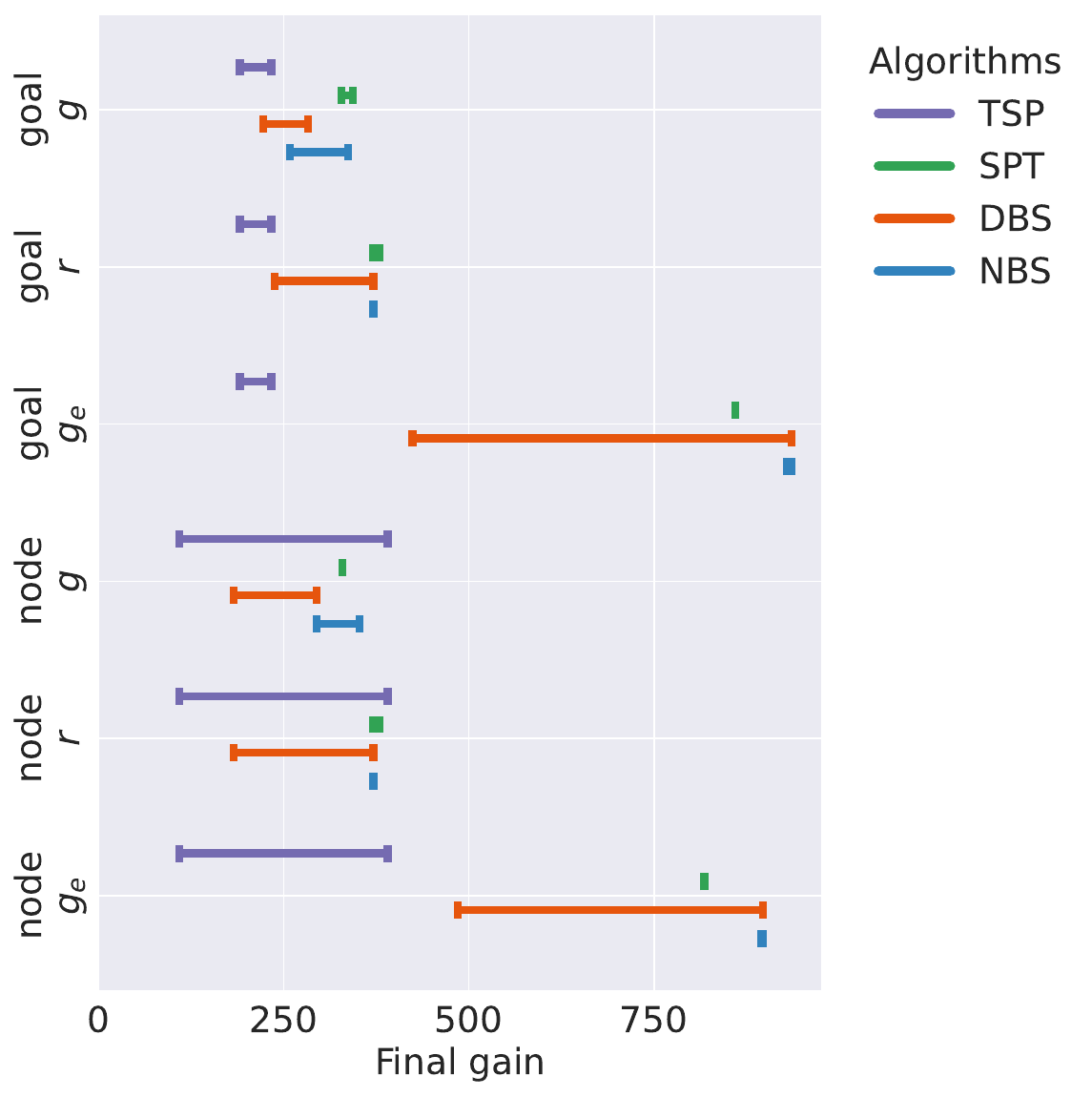}
    \caption{Clustered gains online perception (cost = 50, size $25 \times 25$)}
    \label{fig:clustered_gain_cost50_onlineperception}
\end{subfigure}
\begin{subfigure}{0.49\textwidth}
    \centering
    \includegraphics[width=0.8\textwidth]{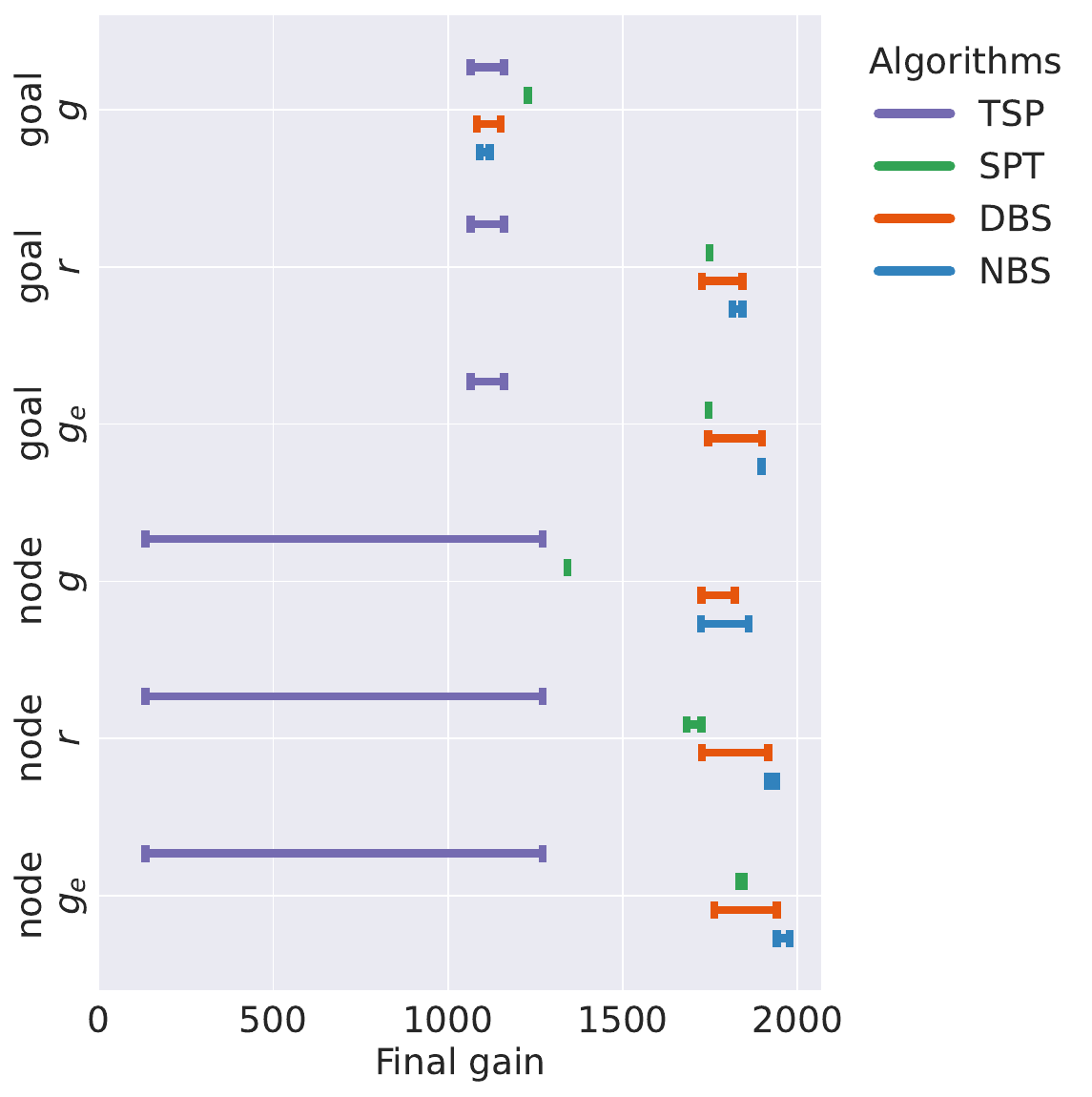}
    \caption{Scattered gains online perception (cost = 50, size $25 \times 25$)}
    \label{fig:scattered_gain_cost50_onlineperception}
\end{subfigure}
\begin{subfigure}{0.49\textwidth}
  \centering
  \includegraphics[width=0.8\textwidth]{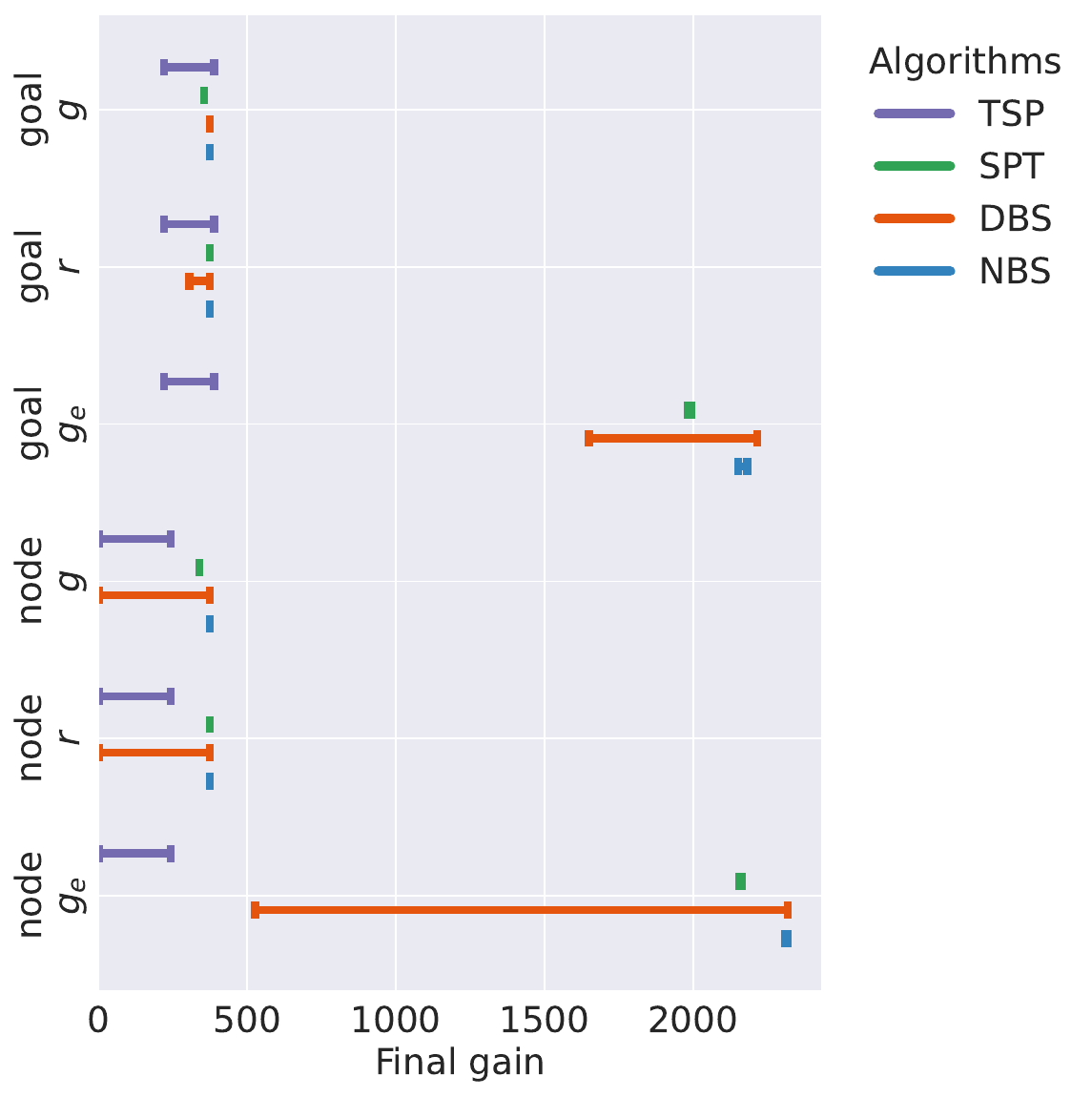}
  \caption{Clustered gains online perception (cost = 100, size $50 \times 50$)}
  \label{fig:clustered_gain_cost100_onlineperception}
\end{subfigure}
\begin{subfigure}{0.49\textwidth}
    \centering
    \includegraphics[width=0.8\textwidth]{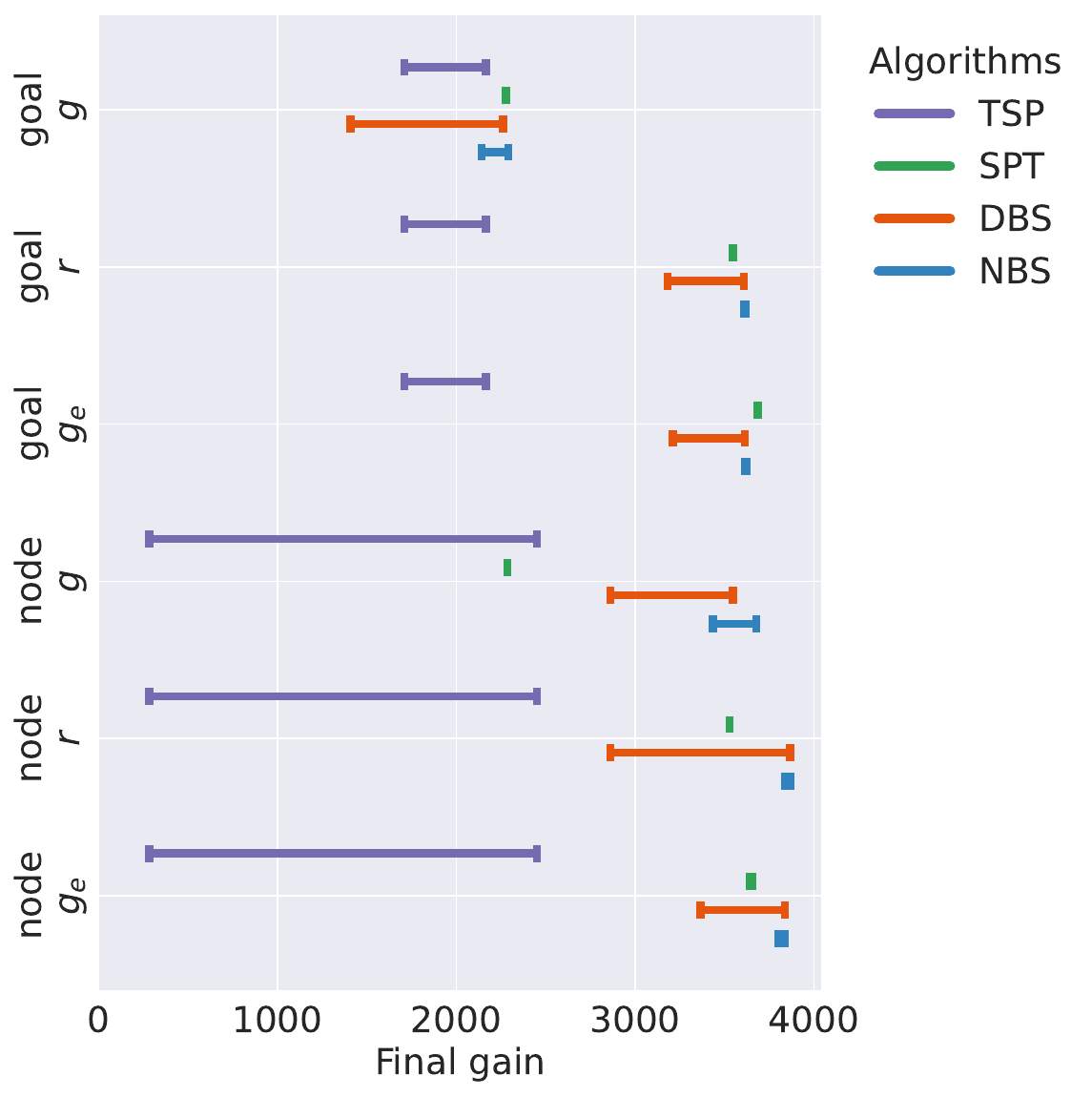}
    \caption{Scattered gains online perception (cost = 100, size $50 \times 50$)}
    \label{fig:scattered_gain_cost100_onlineperception}
\end{subfigure}
    \caption{\textbf{The range of final collected gain in the graph online perception scenario.} The expected gain selection criterion delivers the strongest results overall, and replanning at every node further improves performance compared to other replanning strategies. With this setup, NBS offers the most reliable performance with little sensitivity to parameter choices.}
\label{fig:gains_cost_onlineperception}
\end{figure*}

Regarding path selection, the top-performing algorithm, NBS, exhibits minimal sensitivity to this choice. This suggests that the specific path selection strategy is less critical in the fully known graph scenario when a robust algorithm is employed. Consequently, we omit further discussion of path selection for this scenario.

\subsection{Graph Online Perception Benchmarking}

In the online perception scenario, the robot can only perceive nodes within a fixed radius (\SI{5}{\meter} in this case). \Cref{fig:gains_cost_onlineperception} illustrates the algorithms' performance with this setup. The analysis parallels that of the a priori case, focusing on maximum gain, parameter sensitivity, and the path selection combination that yields the best performance.

We observe that the expected gain objective consistently yields superior results compared to the path ratio or path gain. This advantage arises because the expected gain explicitly encourages exploration by incorporating frontier information. In contrast, the other selection criteria lack such exploratory incentives, leading planners to over-exploit known areas and become trapped in local optima. Under the expected gain criterion, replanning at every step generally outperforms replanning only at the goal. Replanning exclusively at the goal ignores new information acquired en route, resulting in suboptimal behavior in three of the four evaluated graph types.

With this path selection setup, NBS consistently achieves the best performance and exhibits low sensitivity to parameter choices. DBS can attain similar performance but is more sensitive to parameter tuning. SPT performs worse overall, limited by its tree-based structure. TSP shows the weakest performance and exhibits high sensitivity to parameter settings.

\subsection{Performance Analysis on Graphs}

The effectiveness of selection criteria depends on the underlying graph formulation. When the graph is fully known a priori, the top-performing algorithms achieve strong results across all criteria. In contrast, in the online perception setting, the expected gain consistently outperforms other criteria across all planners. Regarding replanning frequency, in the a priori case, DBS and SPT benefit significantly from frequent replanning to improve performance. While TSP can also benefit from higher replanning frequency, it often exhibits oscillatory behavior between competing paths. Notably, NBS maintains strong performance across all replanning frequencies (even without any replanning), highlighting its robustness and computational efficiency. In the online perception setting, when paired with the expected gain criterion, high-frequency replanning mostly improves gain accumulation. This combination enables planners to effectively balance exploration and exploitation, adapting in real time to newly acquired information.

NBS demonstrates exceptional robustness to parameter variations, enabling reliable deployment with minimal tuning effort. Even at the smallest beam width ($B=1$), it delivers strong performance. In contrast, DBS is more sensitive to beam width selection---poor parameter choices can significantly degrade its performance. SPT remains the most computationally efficient but sacrifices solution quality due to its restricted search space. TSP performs poorly overall: it is highly sensitive to its parameters, prone to oscillations between suboptimal paths, and exhibits poor scalability, with prohibitively high computation times on larger graphs.

Overall, our results show that NBS is the top-performing \ac{ipp} algorithm. The combination of the expected gain metric with every-node replanning demonstrates the most effective selection strategy. Together, they form the best-performing planning approach, consistently achieving superior results in both a priori and online perception scenarios. In the following section, we shift our focus to active perception. Based on our findings, we adopt the expected gain criterion and apply a high replanning frequency.  NBS is selected with a beam width of $1$, as it is the most robust and reliable algorithm across all settings. DBS is configured with a larger beam width ($B=100$). For SPT, given its computational efficiency, we evaluate the entire tree search space by setting the threshold $\alpha$ to $1.0$. The TSP algorithm is configured with a threshold of $0.5$.

\section{Graph Construction for Active Perception}
\label{sec:graph-construction}
As introduced in \Cref{sec:active-perception-formulation} and illustrated in \Cref{fig:paper-organization}, active perception does not assume a predefined underlying graph environment, unlike prior \ac{ipp} problems defined on graphs. Therefore, to apply \ac{ipp} algorithms, we incrementally build a graph as the robot explores the environment.

Following prior work~\citep{schmid2020efficient, xu2021autonomous}, we enforce a minimum edge length constraint $l_{\min}$: a newly sampled, collision-free configuration is accepted only if its distance to all existing graph nodes exceeds $l_{\min}$. This helps avoid oversampling that would yield highly redundant information and improves computational efficiency during information gain evaluation and path planning. Once accepted, bidirectional edges are attempted between the new node and its neighbors, subject to a maximum edge length $l_{\max}$ and the existence of a collision-free path. The maximum edge length prevents excessive edge density, such as $|\mathset{V}|(|\mathset{V}| - 1)$ edges in a complete graph. When a sampled configuration lies beyond $l_{\max}$ from any node, it is relocated to a position at a distance $l_{\min}$ along the direction toward the nearest neighbor, increasing the likelihood of establishing connections with nearby nodes. This sampling procedure is detailed in \Cref{alg:sample_free}, where $n_{\mathrm{sample}}$ denotes the maximum number of sampling attempts and $v_{\mathrm{null}}$ represents an invalid sampled node.

\begin{algorithm}
    \caption{Node Sampling (\method{SampleFree})}
    \label{alg:sample_free}
    \For {$i = 1 : n_{\mathrm{sample}}$} {
        $v_{\mathrm{rand}} \gets \method{Sample}()$ \\
        $v_{\mathrm{nearest}} \gets \method{Nearest}(\mathset{V}, v_{\mathrm{rand}})$ \\
        \If {$\| v_{\mathrm{rand}} - v_{\mathrm{nearest}} \| < l_{\min}$} {
            \Continue
        }
        \If{$\| v_{\mathrm{rand}} - v_{\mathrm{nearest}} \| > l_{\max}$}{
            $v_{\mathrm{rand}} \gets
              v_{\mathrm{nearest}} +
              l_{\min} \dfrac{v_{\mathrm{rand}} - v_{\mathrm{nearest}}}
                    {\| v_{\mathrm{rand}} - v_{\mathrm{nearest}} \|}$
        }
        \If {$\method{CollisionFree}(v_{\mathrm{{rand}}})$} {
            \Return $v_{\mathrm{rand}}$
        }
    }
    \Return $v_{\mathrm{null}}$
\end{algorithm}

The resulting structure, in which edges exist only between nodes separated by distances within the interval $[l_{\min}, l_{\max}]$, is known as an annulus graph~\citep{lichev2024annulus}. In such graphs, the number of outgoing edges per node---referred to as the out-degree---is inherently bounded. This bound follows from a geometric sphere-packing argument: specifically, the maximum number of non-overlapping spheres of radius $l_{\min}/2$ that can be packed into a larger sphere of radius $l_{\max}$\footnote{Two non-overlapping spheres of radius $l_{\min}/2$ ensure that the distance between their centers is at least $l_{\min}$.}.
This yields the following conservative upper bound on the out-degree:
\begin{equation}
    | \delta^+(v) | \leq 8\left(\frac{l_{\max}}{l_{\min}}\right)^3, \quad \forall v \in \mathset{V},
\end{equation}
where $\delta^+(v)$ is the set of outgoing edges of $v$. While tighter bounds exist (see~\cite{conway2013sphere}), this estimate provides a sufficiently conservative, closed-form upper bound.

We introduce the \acl{rrag} (\ac{rrag}), a variant of the RRG that incorporates both minimum and maximum edge length constraints. The graph construction procedure is described in \Cref{alg:rrag}. The auxiliary function \method{Near}$(\mathset{V}, v, l)$ returns the subset of vertices in $\mathset{V}$ within distance $l$ of vertex $v$, computed efficiently using the \texttt{nanoflann} implementation of the $k$-d tree~\citep{bentley1975multidimensional, blanco2014nanoflann}. The sampling routine is detailed in \Cref{alg:sample_free} and the collision-checking procedure between two nodes is discussed in \Cref{subsec:collision-check}. For completeness, we also implement two related planners: RRAT (\Cref{alg:rrat}), analogous to RRT in~\citep{bircher2016receding}, and RRAT* (\Cref{alg:rrat-star}), which parallels RRT* in~\citep{schmid2020efficient}. We denote $\method{Cost}(v)$ as the lowest cost from the start node to node $v$.

\begin{algorithm}
    \caption{RRAG}
    \label{alg:rrag}
    \KwIn{graph $\mathset{G} = (\mathset{V}, \mathset{E})$}
    \For {$i = 1 : n_{\mathrm{new}}$} {
        $v_{\mathrm{new}} \gets \method{SampleFree}()$ \\
        \If{$v_{\mathrm{new}} = v_{\mathrm{null}}$} {
            \Continue
        }
        $v_{\mathrm{nearest}} \gets \method{Nearest}(\mathset{V}, v_{\mathrm{new}})$ \\
        \If{$\method{CollisionFree}(v_{\mathrm{nearest}}, v_{\mathrm{new}})$} {
            $\mathset{V}_{\mathrm{near}} \gets \method{Near}(\mathset{V}, v_{\mathrm{new}}, l_{\max})$ \\
            \For{$v_{\mathrm{near}} \in \mathset{V}_{\mathrm{near}}$} {
                \If{$\method{CollisionFree}(v_{\mathrm{near}}, v_{\mathrm{new}})$} {
                    $\mathset{E} \gets \mathset{E} \cup \{(v_{\mathrm{near}}, v_{\mathrm{new}})\}$ \\
                }
                \If{$\method{CollisionFree}(v_{\mathrm{new}}, v_{\mathrm{near}})$} {
                    $\mathset{E} \gets \mathset{E} \cup \{(v_{\mathrm{new}}, v_{\mathrm{near}})\}$ \\
                }
            }
            $\mathset{V} \gets \mathset{V} \cup \{v_{\mathrm{new}}\}$ \\
        }
    }
    \Return $\mathset{G} = (\mathset{V}, \mathset{E})$
\end{algorithm}

\begin{algorithm}
    \caption{RRAT}
    \label{alg:rrat}
    \KwIn{graph $\mathset{G} = (\mathset{V}, \mathset{E})$}
    \For {$i = 1 : n_{\mathrm{new}}$} {
        $v_{\mathrm{new}} \gets \method{SampleFree}()$ \\
        \If{$v_{\mathrm{new}} = v_{\mathrm{null}}$} {
            \Continue
        }
        $v_{\mathrm{nearest}} \gets \method{Nearest}(\mathset{V}, v_{\mathrm{new}})$ \\
        \If{$\method{CollisionFree}(v_{\mathrm{nearest}}, v_{\mathrm{new}})$} {
            $\mathset{E} \gets \mathset{E} \cup \{(v_{\mathrm{nearest}}, v_{\mathrm{new}})\}$ \\
            $\mathset{V} \gets \mathset{V} \cup \{v_{\mathrm{new}}\}$ \\
        }
    }
    \Return $\mathset{G} = (\mathset{V}, \mathset{E})$
\end{algorithm}

\begin{algorithm}
    \caption{RRAT*}
    \label{alg:rrat-star}
    \KwIn{graph $\mathset{G} = (\mathset{V}, \mathset{E})$}
    \For {$i = 1 : n_{\mathrm{new}}$} {
        $v_{\mathrm{new}} \gets \method{SampleFree}()$ \\
        \If{$v_{\mathrm{new}} = v_{\mathrm{null}}$} {
            \Continue
        }
        $v_{\mathrm{nearest}} \gets \method{Nearest}(\mathset{V}, v_{\mathrm{new}})$ \\
        \If{$\method{CollisionFree}(v_{\mathrm{nearest}}, v_{\mathrm{new}})$} {
            $v^* \gets v_{\mathrm{nearest}}$\\
            $c^* \gets \method{Cost}(v^*) + c(v^*, v_{\mathrm{new}})$ \\
            $\mathset{V}_{\mathrm{near}} \gets \method{Near}(\mathset{V}, v_{\mathrm{new}}, l_{\max})$ \\
            \For{$v_{\mathrm{near}} \in \mathset{V}_{\mathrm{near}}$} {
                \If{$\method{CollisionFree}(v_{\mathrm{near}}, v_{\mathrm{new}})$} {
                    $c \gets \method{Cost}(v_{\mathrm{near}}) + c(v_{\mathrm{near}}, v_{\mathrm{new}})$ \\
                    \If{$c < c^*$} {
                        $v^* \gets v_{\mathrm{near}}$, $c^* \gets c$ \\
                    }
                }
            }
            $\mathset{E} \gets \mathset{E} \cup \{(v^*, v_{\mathrm{new}})\}$ \\
            $\mathset{V} \gets \mathset{V} \cup \{v_{\mathrm{new}}\}$ \\
            \For{$v_{\mathrm{near}} \in \mathset{V}_{\mathrm{near}}$} {
                \If{$\method{CollisionFree}(v_{\mathrm{new}}, v_{\mathrm{near}})$} {
                    \If{$\method{Cost}(v_{\mathrm{new}}) + c(v_{\mathrm{new}}, v_{\mathrm{near}}) < \method{Cost}(v_{\mathrm{near}})$} {
                        $\mathset{E} \gets \mathset{E} \setminus \{(\cdot, v_{\mathrm{near}})\} \ \cup \{(v_{\mathrm{new}}, v_{\mathrm{near}})\}$ \\
                    }
                }
            }
        }
    }
    \Return $\mathset{G} = (\mathset{V}, \mathset{E})$
\end{algorithm}

\subsection{Collision Checking Between Nodes}
\label{subsec:collision-check}

To determine whether a trajectory between two nodes is feasible, we must verify that it lies entirely within the free space $\mathset{X}_{\mathrm{free}}$, i.e., the robot avoids collisions with obstacles and (if applicable) self-collisions along the path. Let $\clearance \colon \mathset{X} \to \mathbb{R}$ denote the clearance function, which assigns to each configuration $x \in \mathset{X}$ the signed Euclidean distance between the robot at $x$ and the nearest obstacle or self-collision, whichever is closer. A configuration $x$ is collision-free if and only if $\clearance(x) > 0$.

When connecting two configurations via straight-line interpolation in $\mathset{X}$, the actual path traced by points on the robot depends on the robot's geometry and kinematics. The following lemma bounds the length of such paths:

\begin{lemma}[Path length bound \citep{barraquand1997random}]
\label{lem:rho-bound}
For any bounded robot, there exists a constant $\eta > 0$, such that when the robot moves along the straight-line path between any two configurations $x_1$ and $x_2$, every point on the robot traces a curve of length at most $\eta \| x_1 - x_2 \|_2$ in the workspace.
\end{lemma}

This bound allows us to derive a sufficient condition under which linear interpolation between two collision-free configurations remains entirely within $\mathset{X}_{\mathrm{free}}$. This concept is formalized as adjacency.

\begin{definition}[Adjacency \citep{barraquand1997random}]
\label{def:adjacency}
For a self-collision-free robot, two configurations $x_1, x_2 \in \mathset{X}_{\mathrm{free}}$ with $\clearance(x_1), \clearance(x_2) > 0$ are adjacent if
\begin{equation}
   \eta \, \| x_1 - x_2 \|_2 < \max\{\clearance(x_1), \clearance(x_2)\}.
   \label{eq:adjacency}
\end{equation}
If self-collisions are possible, the right-hand side should be scaled by $1/2$.
\end{definition}

This sufficient condition leads to the following guarantee for collision-free edges in annulus graphs.
\begin{theorem}
\label{thm:safe-connect}
For robots without self-collision risks, any attempted trajectory verification in \ac{rrag} between $x_1, x_2 \in \mathset{X}_{\mathrm{free}}$ is collision-free if
$\clearance(x_1) > \eta \, l_{\max}$ or $\clearance(x_2) > \eta \, l_{\max}$. 
When the distance between $x_1$ and $x_2$ reaches the lower bound $l_{\min}$, a collision-free path is then guaranteed if $\clearance(x_1) > \eta \, l_{\min}$ or
$\clearance(x_2) > \eta \, l_{\min}$.
\end{theorem}

\begin{proof}
In \ac{rrag}, edges are only formed between nodes satisfying $\|x_1 - x_2\|_2 \leq l_{\max}$. To ensure that the linear path lies entirely in $\mathset{X}_{\mathrm{free}}$, it suffices to satisfy the adjacency condition, which holds whenever
\begin{equation}
\eta\,\|x_1 - x_2\|_2 \leq \eta\, l_{\max} < \max\{\clearance(x_1), \clearance(x_2)\}.
\end{equation}
Requiring $\clearance(x_1) > \eta\, l_{\max}$ or $\clearance(x_2) > \eta\, l_{\max}$ ensures this inequality, yielding a conservative sufficient condition.

Similarly, when $\|x_1 - x_2\|_2 = l_{\min}$, we require
\begin{equation}
\eta\,\|x_1 - x_2\|_2 = \eta\, l_{\min} < \max\{\clearance(x_1), \clearance(x_2)\}.
\end{equation}
Imposing $\clearance(x_1) > \eta\, l_{\min}$ or $\clearance(x_2) > \eta\, l_{\min}$ guarantees this condition. \hfill \qedsymbol
\end{proof}

\begin{figure}[!t]
    \centering
    \includegraphics[width=0.95\columnwidth]{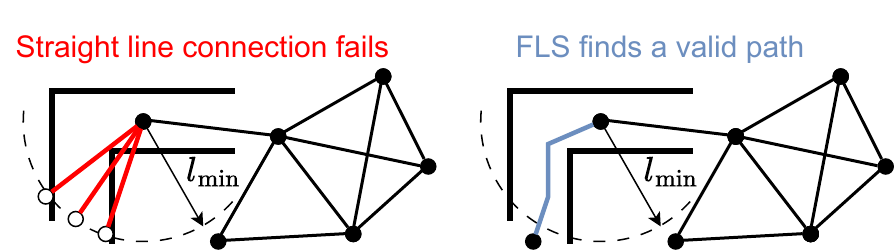}
    \caption{\textbf{The necessity of the FLS in narrow passages.} The FLS can find a path through a narrow ``L''-shaped corridor when a straight-line connection fails.}
    \label{fig:local-planner}
\end{figure}

Therefore, in narrow or heavily cluttered workspaces, where clearances are typically small, $l_{\min}$ and $l_{\max}$ must be reduced accordingly to guarantee sufficient graph coverage. However, this adjustment results in a denser graph, leading to highly overlapping information gains and significantly increased computational burden. If not properly tuned, the robot may fail to establish new edges, stalling graph expansion. To address this, we introduce a fallback local sampling-based planner (FLS) to recover valid transitions in cluttered or geometrically complex regions where straight-line interpolation would otherwise fail, as illustrated in \Cref{fig:local-planner}. The FLS is integrated into all three graph construction methods (RRAG, RRAT, RRAT*). Specifically, we employ the standard RRT* algorithm, implemented via the OMPL library~\citep{sucan2012open}, to explore alternative collision-free paths locally. The FLS operates within an axis-aligned bounding box that tightly encloses the start and goal configurations, expanded by a margin of $l_{\max}$ in each dimension. It is invoked with a planning time budget to minimize path length. Thanks to the probabilistic completeness of RRT*, the planner will eventually find a valid path if one exists, given sufficient time.

\subsection{Connectivity of Annulus Graph}

While \citet{galhotra2019connectivity} provided initial theoretical insights into the connectivity of random annulus graphs, their analysis primarily focuses on point sets sampled on a unit sphere (e.g., $\{x \in \mathset{X} \mid \|x\|_2 = 1\}$) and applies only to specific connectivity regimes. More recent work by \citet{lichev2024annulus} highlights that the structural properties of annulus graphs are sensitive to the choice of $l_{\min}$ and $l_{\max}$. To ensure connectivity among nodes, the annular region used for establishing edges must be sufficiently thick. The following theorem formalizes this requirement under natural geometric conditions: the maximum connection radius should be at least twice the minimum separation distance, i.e., $l_{\max} \geq 2 l_{\min}$.

\begin{theorem}
\label{thm:graph-connectivity}
Let $\mathset{X}_{\mathrm{free}}$ be a connected subset of the configuration space, and let $\mathset{V} \subset \mathset{X}_{\mathrm{free}}$ be a finite set of configurations satisfying:
\begin{enumerate}
    \item $l_{\min}$-packing: Nodes in $\mathset{V}$ are mutually separated by at least $l_{\min}$, i.e., $        \forall v_i, v_j \in \mathset{V},\ v_i \ne v_j \implies \|v_i - v_j\| \geq l_{\min}$;
    \item $l_{\min}$-covering: Every point in $\mathset{X}_{\mathrm{free}}$ lies within distance $l_{\min}$ of some node in $\mathset{V}$, i.e., $
    \forall x \in \mathset{X}_{\mathrm{free}},\ \exists v \in \mathset{V} \text{ such that } \|x - v\| \leq l_{\min}$.
\end{enumerate}
Let $\mathset{G} = (\mathset{V}, \mathset{E})$ be the undirected annulus graph where two nodes with a distance between $l_{\min}$ and $l_{\max}$ are connected if there exists a collision-free path between them.
Then, $\mathset{G}$ is guaranteed to be connected for any such $\mathset{V}$ if and only if
\begin{equation}
    l_{\max} \geq 2 l_{\min}.
\end{equation}
\end{theorem}

\begin{proof}
We first prove the sufficiency condition by contradiction. Suppose $l_{\max} \geq 2 l_{\min}$, but $\mathset{G}$ is disconnected. Then the vertex set $\mathset{V}$ can be partitioned into two non-empty, disjoint subsets $\mathset{V}_a$ and $\mathset{V}_b$.

Since $\mathset{X}_{\mathrm{free}}$ is connected, there exists a continuous path $\tau: [0,1] \to \mathset{X}_{\mathrm{free}}$ connecting any $v_a \in \mathset{V}_a$ to $v_b \in \mathset{V}_b$. The absence of edges implies that $\|v_a - v_b\| > l_{\max}$ for all such pairs (the lower bound $\|v_a - v_b\| \geq l_{\min}$ is satisfied due to the packing condition).

By the $l_{\min}$-covering property, $\mathset{X}_{\mathrm{free}}$ is contained in the union of closed balls of radius $l_{\min}$ centered at nodes in $\mathset{V}$:
\begin{equation}
    \mathset{X}_{\mathrm{free}} \subseteq \bigcup_{v \in \mathset{V}} \mathset{B}(v, l_{\min}).
\end{equation}
Define the covered regions:
\begin{equation}
    \mathset{X}_a = \bigcup_{v_a \in \mathset{V}_a} \mathset{B}(v_a, l_{\min}), \quad
    \mathset{X}_b = \bigcup_{v_b \in \mathset{V}_b} \mathset{B}(v_b, l_{\min}).
\end{equation}
Then $\mathset{X}_{\mathrm{free}} \subseteq \mathset{X}_a \cup \mathset{X}_b$, and since $\tau(0) \in \mathset{X}_a$, $\tau(1) \in \mathset{X}_b$, continuity ensures the existence of a point $\tilde{x} = \tau(t)$ for some $t \in (0,1)$ lying on the boundary between $\mathset{X}_a$ and $\mathset{X}_b$.

At this boundary point, there exist nodes $\tilde{v}_a \in \mathset{V}_a$ and $\tilde{v}_b \in \mathset{V}_b$ such that $\|\tilde{x} - \tilde{v}_a\| \leq l_{\min}$ and $\|\tilde{x} - \tilde{v}_b\| \leq l_{\min}$. By the triangle inequality,
\begin{equation}
    \|\tilde{v}_a - \tilde{v}_b\| \leq \|\tilde{x} - \tilde{v}_a\| + \|\tilde{x} - \tilde{v}_b\| \leq 2l_{\min}.
\end{equation}
Given $l_{\max} \geq 2l_{\min}$, it follows that $\|\tilde{v}_a - \tilde{v}_b\| \leq l_{\max}$, which contradicts the assumption of disconnection. Hence, $\mathset{G}$ is connected.

For necessity, consider a counterexample in one dimension. Let $\mathset{X}_{\mathrm{free}} = [0, d]$ with $l_{\max} < d < 2l_{\min}$, and define $\mathset{V} = \{0, d\}$. The packing condition holds because $\|0 - d\| = d > l_{\max} > l_{\min}$. The covering condition also holds: every point in $[0,d]$ is within $l_{\min}$ of either 0 or $d$, because $d < 2l_{\min}$ implies the balls $\mathset{B}(0, l_{\min})$ and $\mathset{B}(d, l_{\min})$ cover the interval.

However, $\|0 - d\| = d > l_{\max}$, so no edge exists between the two nodes. Thus, $\mathset{G}$ is disconnected despite $\mathset{X}_{\mathrm{free}}$ being connected. This completes the proof. \hfill \qedsymbol
\end{proof}

\Cref{thm:graph-connectivity} requires two conditions. The packing condition is inherently enforced by our sampling procedure (\Cref{alg:sample_free}), which rejects any new sample within distance $l_{\min}$ of existing nodes. The covering condition means that the vertices must be sufficiently densely sampled to ensure effective coverage of the free space, which encourages using a sufficiently large $n_{\mathrm{new}}$ in \Cref{alg:rrag}.

\subsection{Preserving All Orientations}

\begin{figure*}[t]
    \centering
    \includegraphics[width=0.9\textwidth]{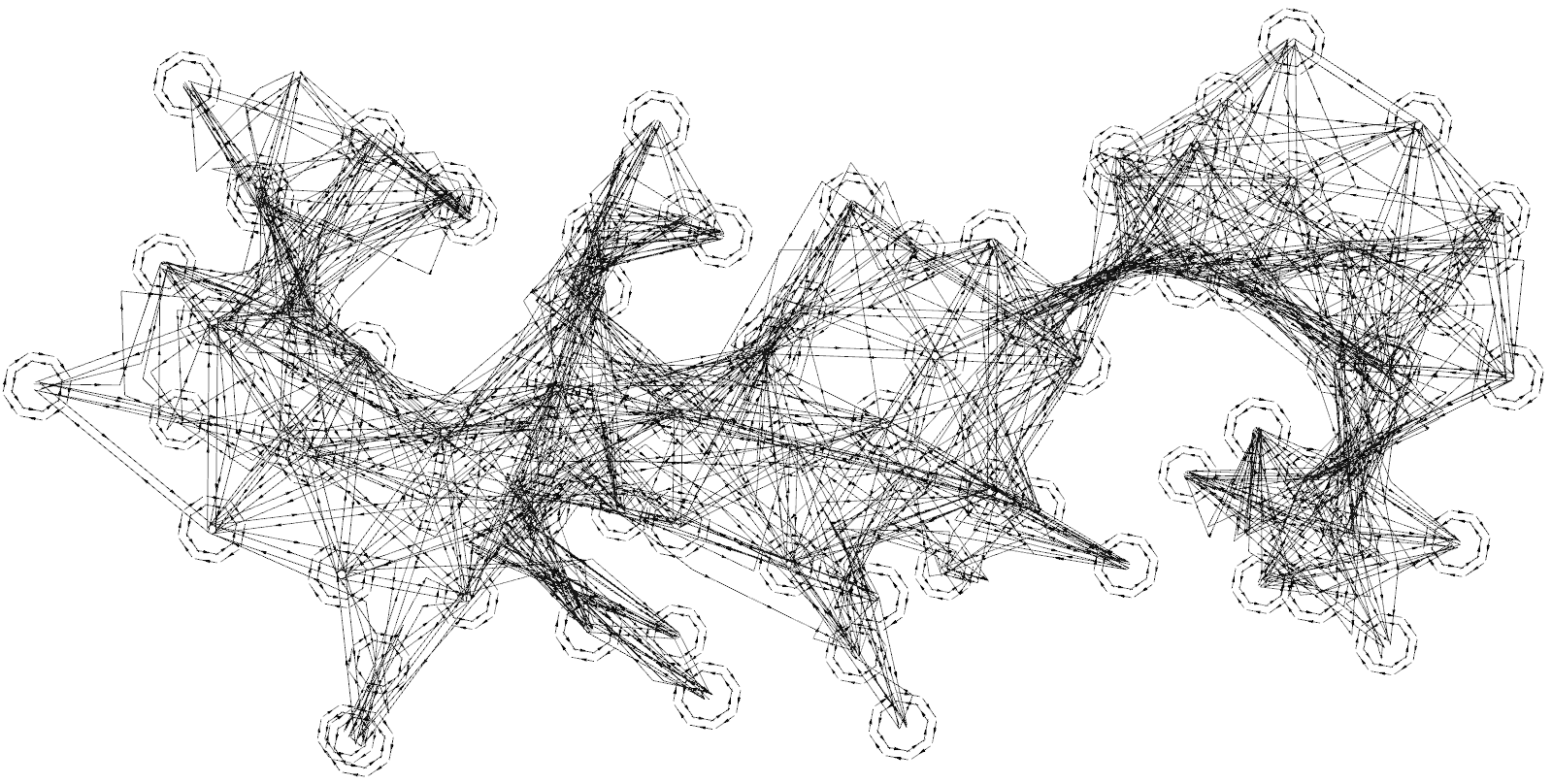}
    \caption{\textbf{Example of the RRAG with all orientations preserved.} Only the edges are visualized. An offset perpendicular to the edge direction is applied to inter-cluster edges to distinguish overlapping connections. For intra-cluster edges, an offset is applied along the node’s yaw direction, forming a polygonal pattern around the actual node. Since FLS can be used to connect nodes, the edges are not strictly straight lines.}
    \label{fig:all-orientation-preserved}
\end{figure*}

In prior work, orientation is typically managed by discretizing the yaw space and selecting a single orientation per sampled position based on the highest gain, while discarding all other candidates~\citep{witting2018history, schmid2020efficient, xu2021autonomous}. Consequently, once a node is selected by the planner, its assigned orientation must be reached through additional turning---even if such rotation is less advantageous than proceeding directly to other nodes. Both RRAT and RRAT* adhere to this strategy.

In contrast, our graph representation provides greater flexibility. For each sampled position, all discrete orientations are retained and treated as distinct nodes. The set of nodes corresponding to different orientations at the same position is referred to as a node cluster. Within each cluster, adjacent orientation nodes are connected by bi-directional edges, provided the transitions are feasible. For connections between different clusters (i.e., inter-cluster edges), connectivity is simplified by retaining only two directed edges (one in each direction) rather than creating a fully connected graph. These edges are constructed by selecting, in each cluster, the node whose orientation is most aligned with the direction of traversal, thereby minimizing unnecessary reorientation during inter-cluster motion. An example of RRAG with all orientations preserved is shown in \Cref{fig:all-orientation-preserved}.

\subsection{Graph Updating}

Prior to each replanning step, the planner updates both node gains and the graph structure. Node gains are recomputed within a local neighborhood of radius $l_{\mathrm{gain}}$ centered at the robot's current position. For RRAT and RRAT*, the best yaw is reselected at each update. Graph structure updates are handled differently depending on the planner. In all cases, edges are revalidated within a local region of radius $l_{\mathrm{edge}}$; any edge that no longer satisfies feasibility constraints is removed. For RRAT, graph updating further includes pruning all child edges of the current root node except the one selected for traversal. For RRAT*, there is a rewiring mechanism that rewires all nodes in the tree relative to the new root. While effective, this step is computationally expensive due to the need for collision checks across numerous node pairs. In contrast, RRAG benefits from its graph structure: since edges are not constrained to form a tree, no additional rewiring is required during replanning. After edge revalidation, the graph expansion process continues as in \Cref{alg:rrag}, \Cref{alg:rrat}, and \Cref{alg:rrat-star}.

Different planners handle replanning slightly differently. For RRAT and RRAT*, which maintain a strict tree structure, this is done by introducing an intermediate node at which replanning is triggered, and then splitting the current edge at this intermediate node. Once the robot reaches this intermediate node, replanning proceeds as discussed, including gain updates, edge revalidation, graph expansion, and path selection. In the RRAG setting, the original edge is retained, and two new edges are added---connecting the start of the current edge to the intermediate node, and the intermediate node to the end of the current edge. Additionally, RRAG attempts to connect the intermediate node to nearby nodes within distance $l_{\max}$ via outgoing edges, enabling the robot to consider alternative paths beyond the original trajectory. The intermediate node is discarded once the robot reaches the next node.

\section{Experiments on Active Perception}
\label{sec:active-perception}
\begin{figure*}[!t]
    \centering
    \includegraphics[width=0.95\textwidth]{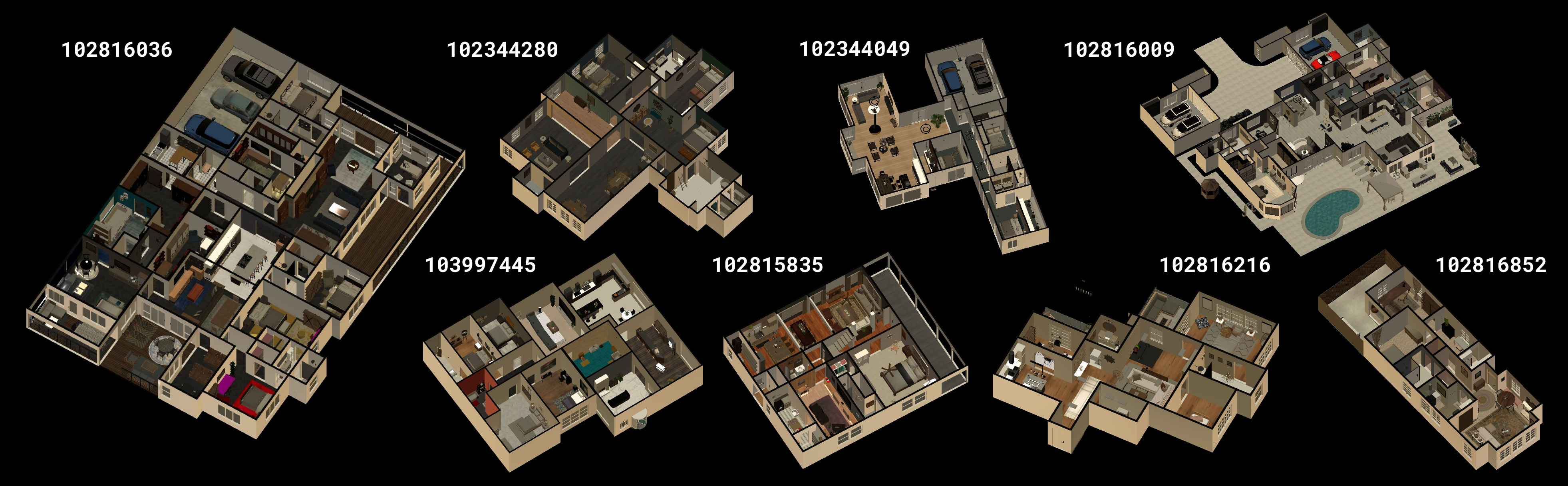}
    \caption{\textbf{Eight \ac{hssd} scenarios used in the active perception simulation environment.} Note that the scenarios are displayed at different scales for visualization purposes.}
    \label{fig:hssd-scenarios}
\end{figure*}

To evaluate the proposed planner, we conduct experiments in a simulated environment using Habitat~\citep{savva2019habitat} with the \ac{hssd} dataset~\citep{khanna2024habitat}, a synthetic collection of diverse 3D environments. We consider eight distinct scenarios, which are visualized along with their IDs in \Cref{fig:hssd-scenarios}.

The robot parameters, including maximum linear and angular velocities, are listed in \Cref{tab:system-constraints}. It moves and receives image data at a fixed timestep. The robot follows a walk-forward-only policy, consistent with standard practices in the Habitat Simulator, where each straight-line motion segment is executed in three phases: (1) rotate to face the target position, (2) translate to the target, and (3) rotate to the desired final orientation\footnote{Other navigation modes are also compatible, as the system only requires a feasible trajectory connecting two viewpoints. For instance, Reeds-Shepp paths can be employed for platforms with a minimum turning radius (e.g., car-like vehicles).}. This strategy enforces forward-facing perception, enhancing environmental awareness during navigation. The onboard camera is configured with a resolution of $640 \times 480$ pixels, a field of view of $90^\circ \times 74^\circ$, and an effective range of 3 meters. Upon receiving sensor input, we use Voxblox~\citep{oleynikova2017voxblox} with a voxel resolution of \SI{0.1}{\meter} to construct an ESDF-based map for gain evaluation and collision checking. The robot's collision body is modeled as a sphere of radius \SI{0.3}{\meter}, centered at the base frame.

\begin{table}[ht] 
  \centering
  \caption{\textbf{Robot parameters.}}
  \label{tab:system-constraints}
  \begin{tabular}{l|l}
    \toprule
    \textbf{Parameter} & \textbf{Value} \\
    \midrule
    max linear velocity $v_{\max}$ & \SI{0.5}{\meter\per\second} \\
    max angular velocity $\omega_{\max}$ & \SI{1.6}{\radian\per\second} \\ 
    timestep  $\Delta t$ & \SI{0.2}{\second} \\
    move mode & walk forward only \\
    collision body & Ball of radius \SI{0.3}{\meter} \\
    \bottomrule
  \end{tabular}
\end{table}

We consider three representative application scenarios: (i) point collection in an observed area, (ii) surface mapping for 3D reconstruction, and (iii) volumetric exploration. These tasks differ primarily in their respective gain definitions. In the point collection scenario, gain is computed by identifying all observed points within a local neighborhood of each node. The total path gain is defined as the sum of all unique points near the nodes visited along the path. For surface reconstruction, we adopt the surface frontier criterion introduced by~\citet{yoder2016autonomous}. In the volumetric exploration setting, gain corresponds to the number of unknown voxels within the sensor's field of view. For both surface reconstruction and volumetric exploration, inter-node gain correlations are neglected; node-level gains are simply aggregated along the trajectory, following established practice~\citep{witting2018history, schmid2020efficient, xu2021autonomous}. 
	
We further analyze the relationship between the modeled gain $g$ and the realized gain $f$---i.e., the principal objective. In the point collection scenario, the gain estimate is exact, as it accounts for all visible points and captures spatial correlations. For surface reconstruction, the estimate is conservative and constitutes a lower bound, as the surface frontier only considers unknown voxels adjacent to occupied regions as potential surface continuations. In contrast, the volumetric exploration model provides an upper bound, assuming visibility through free space while ignoring occlusions. By evaluating different algorithms under these differing modeling assumptions, we aim to demonstrate the robustness and effectiveness of our proposed algorithm.

We compare our proposed method against several baselines. These include the approaches of \cite{bircher2016receding} and \cite{schmid2020efficient}, which we adapt to use RRAT (\Cref{alg:rrat}) and RRAT* (\Cref{alg:rrat-star}) as their underlying tree representations, respectively. We also include SPT / DEP~\citep{xu2021autonomous}, which is used in conjunction with RRAG (\Cref{alg:rrag}).
For TSP-based planning, we consider both the TSP algorithm described in \Cref{sec:tsp} and the FUEL planner~\citep{zhou2021fuel}, which performs active mapping using only frontier information. For clarity, each planner is named according to its core conceptual contribution. \Cref{tab:planner-description} summarizes the graph structure, planning algorithm, and reference for each method.

\begin{table*}[th] 
\centering
\caption{\textbf{Planner description.}}
\label{tab:planner-description}
\begin{tabular}{c|c|c|c}
\toprule
\textbf{Planner name} & {\textbf{Graph structure}} & \textbf{Planning algorithm} & \textbf{References \& Other names} \\
\midrule
NBS & RRAG & NBS & - \\ 
DBS & RRAG & DBS & - \\ 
SPT & RRAG & SPT & \cite{xu2021autonomous}: DEP \\ 
TSP & RRAG & TSP & \cite{arora2017randomized} \\ 
RRT* & RRAT* & - & \cite{schmid2020efficient} \\ 
RRT & RRAT & - & \cite{bircher2016receding}: NBVP \\
FUEL & - & TSP & \cite{zhou2021fuel} \\
\bottomrule
\end{tabular}
\end{table*}

Parameter settings for RRT, RRT*, SPT, and TSP are selected based on the results in \Cref{sec:benchmark}. 
The parameters of the FUEL planner are tuned to better align with our robot’s sensing capabilities and the characteristics of our experimental environments. In all experiments, we define {frontier nodes} $\mathset{F}$ as those with a substantial number of unobserved voxels within their field of view. This indicates that navigating the robot to such a node is likely to enable further expansion of the graph. This behavior is governed by two parameters: the minimum ratio of unknown voxels within the field of view, and the minimum ratio of currently visible voxels relative to the historical maximum. \Cref{tab:planner-parameters} summarizes the planner parameters.

\begin{table}[ht] 
\centering
\caption{\textbf{Planner parameters.}}
\label{tab:planner-parameters}
    \begin{tabular}{l|l}
    \toprule
    \textbf{Parameter} & \textbf{Value} \\
    \midrule
    num yaw candidates & 8 \\ 
    frontier min unknown ratio & 0.8 \\
    frontier min visible ratio & 0.6 \\
    selection criterion & $g_e$ \\ 
    replan frequency & \SI{1.0}{\hertz} \\
    FLS planning time & \SI{2}{\milli\second} \\
    max new nodes tries $n_{\mathrm{new}}$ & 50 \\
    update radius ($l_{\mathrm{gain}}, l_{\mathrm{edge}}$)  & \SI{5.0}{\meter} \\ 
    \midrule
    NBS beam width $B$ & 1 \\
    NBS search depth $D$  & 100 \\
    \midrule
    DBS beam width $B$ & 100 \\
    DBS search depth $D$ & 100 \\
    \midrule
    TSP top node percentage $\alpha$ & 0.5 \\
    \midrule
    SPT top node percentage $\alpha$ & 1.0 \\
    \bottomrule
    \end{tabular}
\end{table}

The remainder of this section details each application scenario and presents corresponding benchmarking results. We then discuss validation of the pipeline on a physical robot platform using onboard computation.

\begin{figure*}[!t]
    \centering
    \includegraphics[width=0.98\textwidth]{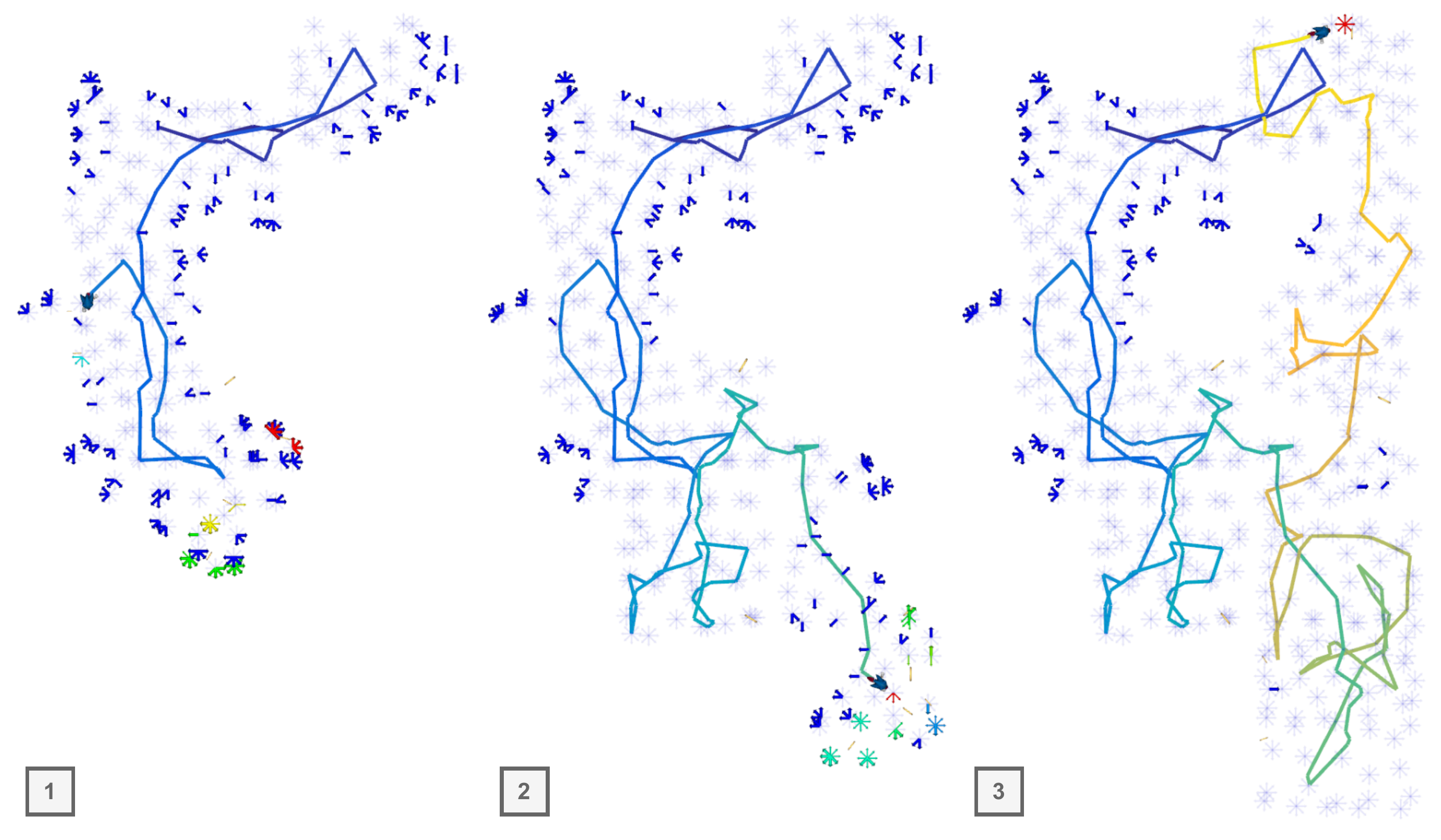}
    \caption{\textbf{Visualization of an agent navigating to collect points represented as golden rings.} These points are initially unknown and become visible only when the corresponding regions are perceived by the agent. The trajectory is visualized using a temporal gradient, transitioning from dark blue (start) to bright yellow (end). Each node is visualized as an arrow, with clusters of nodes at identical positions forming node clusters. Node gains are color-coded, with blue indicating the lowest gain and red the highest. Frontier nodes are emphasized with larger markers, while transparent nodes indicate non-frontier nodes with zero gain.}
    \label{fig:point-collection-example}
\end{figure*}

\subsection{Point Collection}

In the point collection scenario, the robot observes a set of collectible points and must navigate to within a Euclidean distance $l_{\text{col}}$ of each point (orientation-independent) to collect it. Let $\mathset{Z}_{\text{obs}}$ denote the set of observed points. For any node $v$, the collection neighborhood is defined as  
\begin{equation}
\mathset{Z}_{\text{col}}(v) = \left\{ z \in \mathset{Z}_{\text{obs}} \;\middle|\; \|v - z\|_2 \leq l_{\text{col}} \right\}.
\end{equation}
When the robot visits node $v$, all points $z \in \mathset{Z}_{\text{col}}(v)$ are considered collected, and their associated gains $g(z)$ contribute to the cumulative reward.

Because the collection neighborhoods of different nodes may overlap, we impose a non-redundancy constraint: once a point $z$ is collected (i.e., its gain has been accounted for), any future visit to nodes within distance $l_{\text{col}}$ of $z$ yields no additional reward. Formally, for a path $p \in \mathset{P}$, the set of collected points is
\begin{equation}
\mathset{Z}_{\text{col}}(p) = \bigcup_{v \in p} \mathset{Z}_{\text{col}}(v),
\end{equation}
where each point is included only once. The total gain along path $p$ is then given by
\begin{equation}
g(p) = \sum_{z \in \mathset{Z}_{\text{col}}(p)} g(z),
\end{equation}
which complies with the general path gain definition in \Cref{eq:path-gain-general}.

This setup parallels the {object-goal navigation}~\citep{batra2020objectnav} task, albeit without relying on semantic mapping. The core mechanism is similar: after observing a target, the robot must navigate to its vicinity to realize the reward. Consequently, the robot must balance exploration of unobserved regions against exploitation of known collectibles. An example of an agent collecting points is visualized in \Cref{fig:point-collection-example}. For evaluation, points are uniformly sampled within the bounding volume and assigned integer gains $g(z) \in [0, 10]$, with both positions and gains held constant for all algorithms to ensure fair comparison. The collection radius $ l_{\text{col}} $ is set to \SI{1}{\meter}, as in the Habitat ObjectNav Challenge~\citep{habitatchallenge2023}. Accordingly, we set the minimum and maximum distance thresholds to $ l_{\min} = \SI{1}{\meter} $ and $ l_{\max} = \SI{2}{\meter} $, respectively.

\begin{figure*}[!t]
    \centering
    \includegraphics[width=0.96\textwidth]{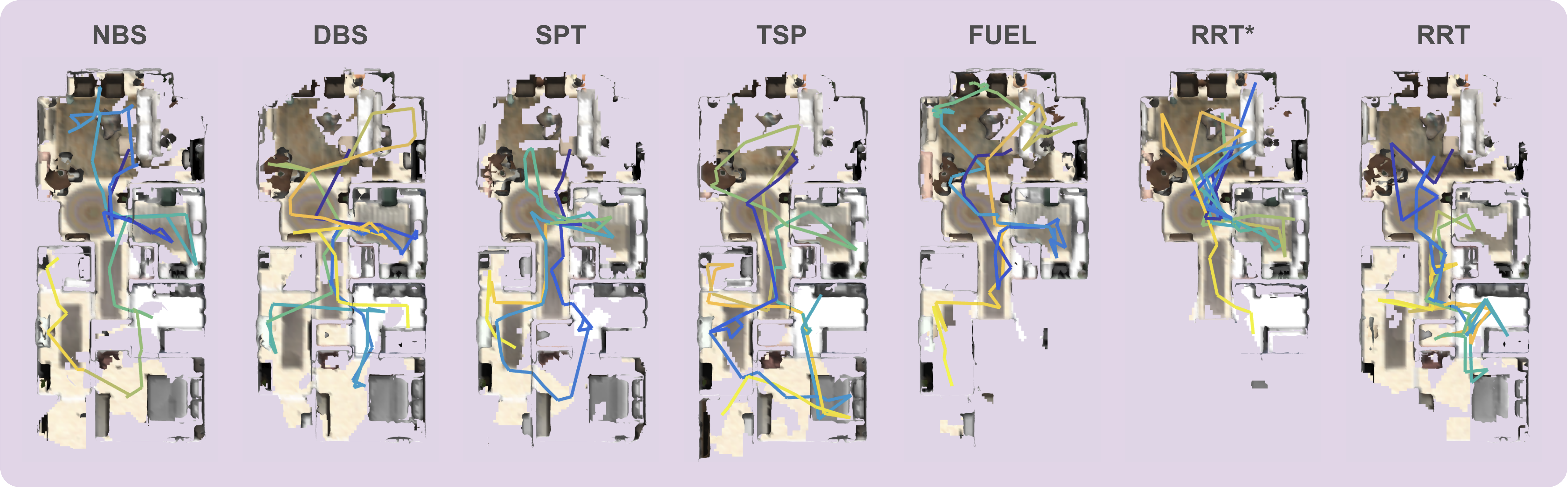}
    \caption{\textbf{Qualitative comparison of surface reconstruction for scene 102816852.} Each planner is executed 10 times, and the visualization shows the worst-case result among the trials, representing a lower-bound performance. The total reconstructed surface area (in square meters) for each planner is as follows: NBS: 445.73, DBS: 419.10, SPT: 349.96, TSP: 432.21, FUEL: 341.03, RRT*: 301.75, and RRT: 348.43. The trajectory is visualized using a temporal gradient, from dark blue to bright yellow.}
    \label{fig:surface-reconstruction-example}
\end{figure*}

\subsection{Surface Reconstruction}

In surface reconstruction, the objective is to scan and recover the geometry of a scene as completely as possible, quantified by the total number of voxels ultimately classified as occupied. However, this true objective $ f $, as used in \Cref{eq:active-perception-objectives}, cannot be directly computed because it requires complete knowledge of the environment---information unavailable to the robot during exploration.

Therefore, we define our gain function $ g $ in \Cref{eq:active-perception-approximation} using the concept of surface frontiers~\citep{yoder2016autonomous}. Surface frontiers are unknown voxels that are adjacent to currently known occupied voxels (i.e., the reconstructed surface). For each node $ v $, the gain is calculated as the number of such frontier voxels visible from $ v $:
\begin{equation}
    g(v) = \left| \left\{ o \in \mathset{O}(v) \cap \mathset{U} \;\middle|\; \mathset{A}(o) \cap \mathset{S} \neq \emptyset \right\} \right|,
    \label{eq:surface-gain}
\end{equation}
where $ \mathset{O}(v) $ is the set of voxels observable from $ v $, $ \mathset{U} $ denotes unknown voxels, $ \mathset{S} $ represents the current set of occupied (surface) voxels, and $ \mathset{A}(o) $ is the set of voxels adjacent to $ o $.  Since nodes do not need to be densely distributed for this task, we set $ l_{\min} = \SI{1.5}{\meter} $ and $ l_{\max} = \SI{3}{\meter} $. The total gain of a trajectory is the sum of the gains of its constituent nodes, as in \Cref{eq:active-perception-gain-sum}.  A qualitative comparison of various approaches on this task is shown in \Cref{fig:surface-reconstruction-example}.

\begin{figure*}[t]  
  \centering
  \setlength{\tabcolsep}{2pt}
  \renewcommand{\arraystretch}{1.0}
  \newcommand{\imgcell}[1]{\adjustbox{valign=t}{\includegraphics[width=0.23\textwidth]{#1}}}

  \begin{tabular}{cccc}
     &
        \hspace{0.5cm} \textbf{Point collection} &
        \hspace{0.5cm} \textbf{Surface reconstruction} &
        \hspace{0.5cm}\textbf{Volumetric exploration} \\
      \midrule
      \rotatebox{90}{\parbox{3.3cm}{\centering \textbf{\qquad 102344280}}} &
      \includegraphics[width=0.28\textwidth]{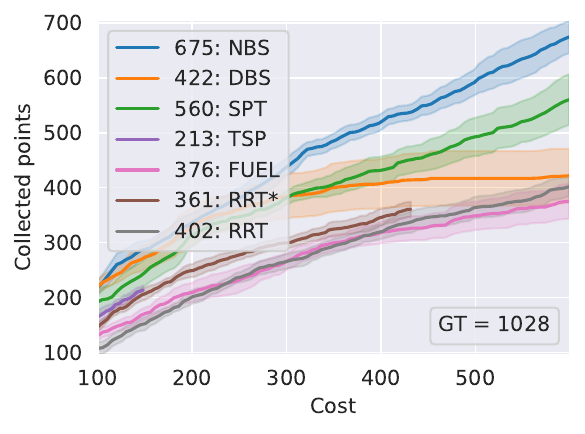} &
      \includegraphics[width=0.28\textwidth]{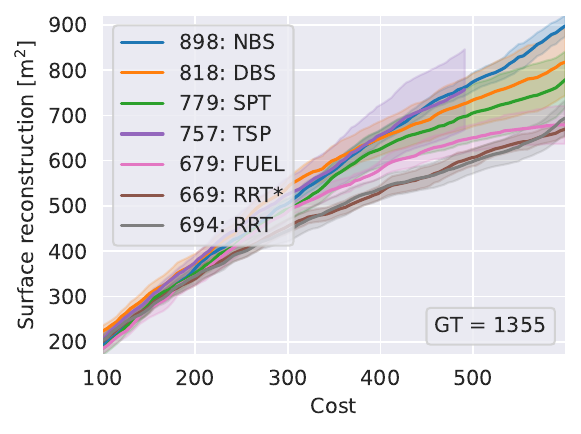} &
      \includegraphics[width=0.28\textwidth]{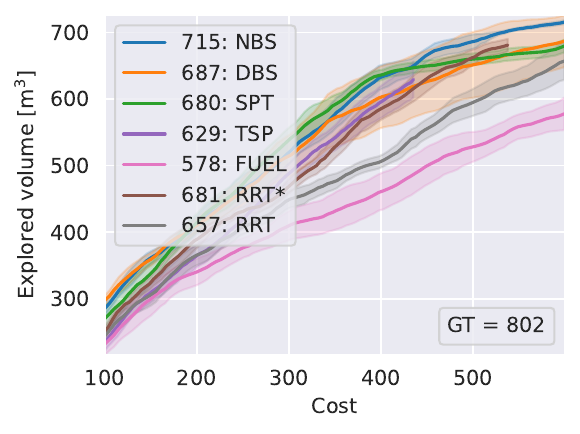} \\
      \rotatebox{90}{\parbox{3.3cm}{\centering \textbf{\qquad 102344049}}} &
      \includegraphics[width=0.28\textwidth]{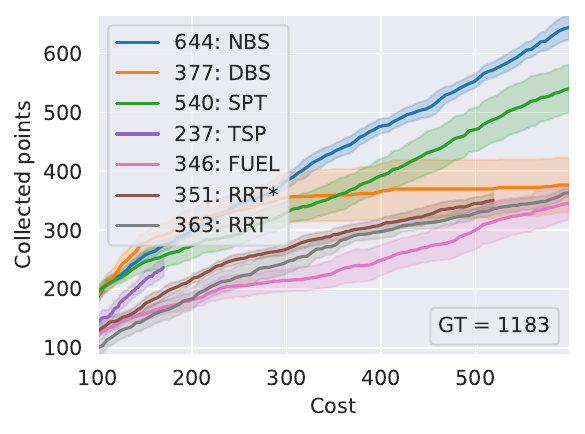} &
      \includegraphics[width=0.28\textwidth]{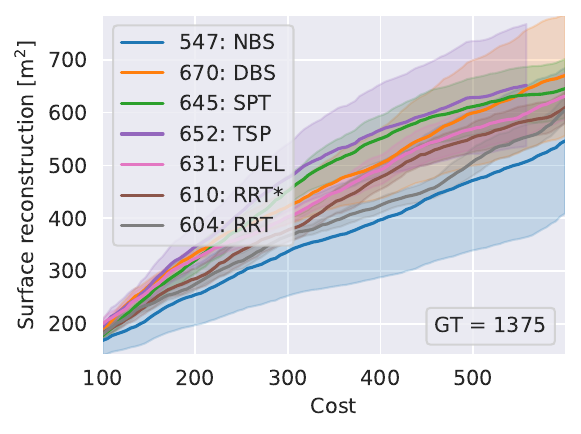} &
      \includegraphics[width=0.28\textwidth]{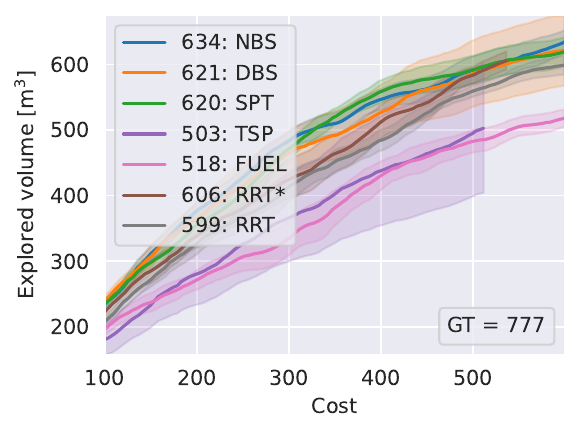} \\
      \rotatebox{90}{\parbox{3.3cm}{\centering \textbf{\qquad 102816009}}} &
      \includegraphics[width=0.28\textwidth]{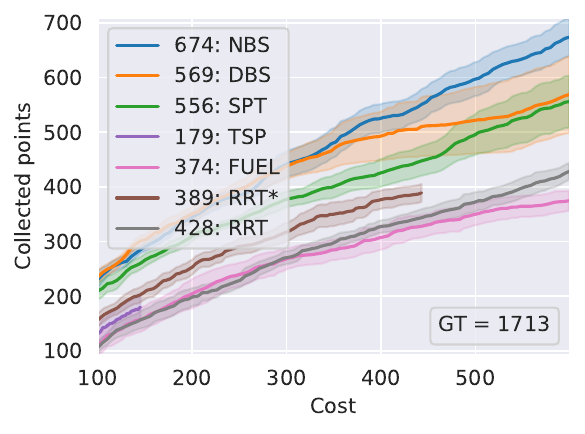} &
      \includegraphics[width=0.28\textwidth]{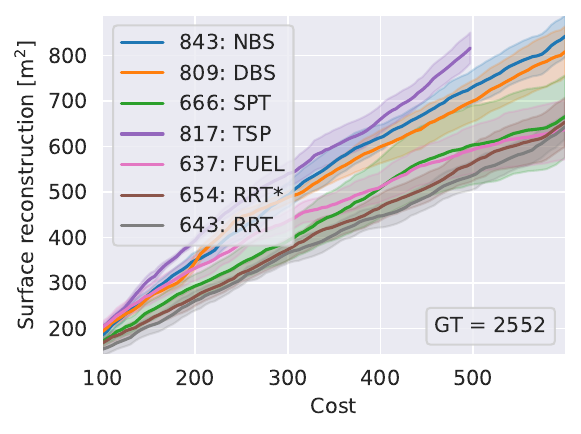} &
      \includegraphics[width=0.28\textwidth]{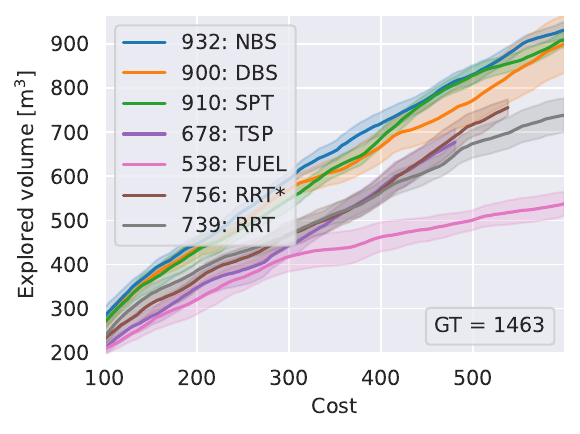} \\
      \rotatebox{90}{\parbox{3.3cm}{\centering \textbf{\qquad 103997445}}} &
      \includegraphics[width=0.28\textwidth]{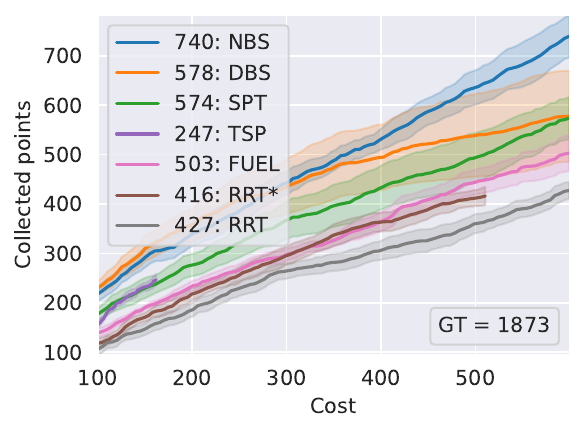} &
      \includegraphics[width=0.28\textwidth]{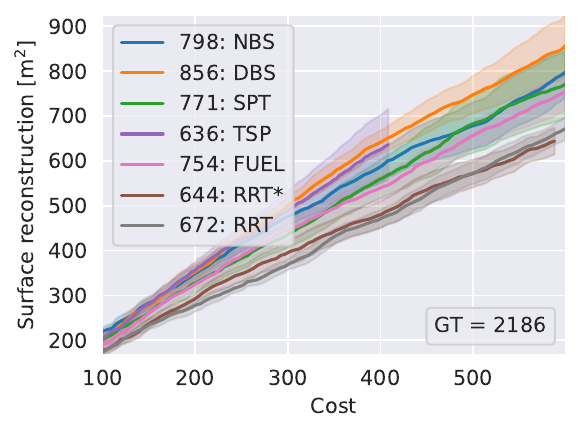} &
      \includegraphics[width=0.28\textwidth]{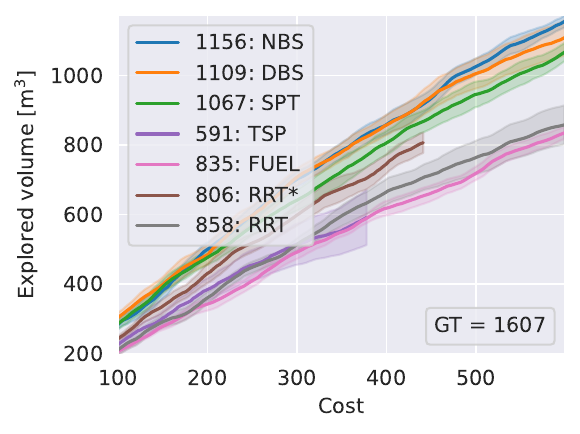} \\
      \rotatebox{90}{\parbox{3.3cm}{\centering \textbf{\qquad 102816036}}} &
      \includegraphics[width=0.28\textwidth]{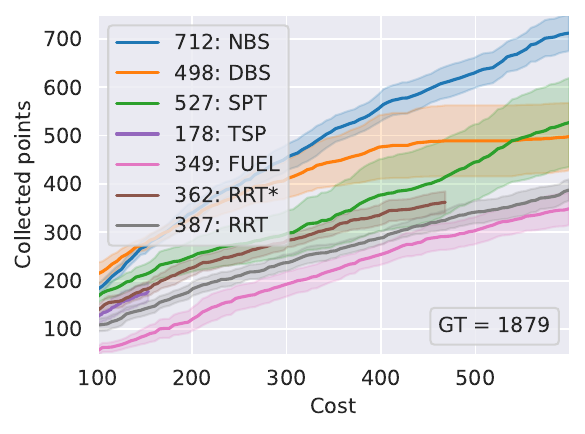} &
      \includegraphics[width=0.28\textwidth]{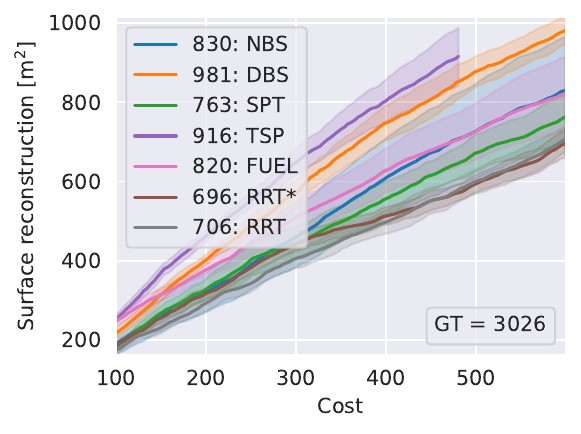} &
      \includegraphics[width=0.28\textwidth]{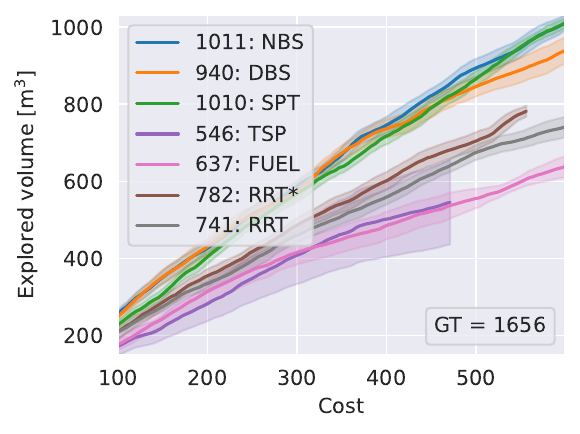} \\
      \bottomrule
  \end{tabular}
  \caption{\textbf{The mean and half the standard deviation of accumulated gain across large \ac{hssd} scenarios.} Trajectory time serves as the cost function with a 10-minute budget. Planning terminates at the 15-minute time limit.}
  \label{fig:active-perception-large-result}
\end{figure*}

\begin{figure*}[t]  
  \centering
  \setlength{\tabcolsep}{2pt}
  \renewcommand{\arraystretch}{1.0}
  \newcommand{\imgcell}[1]{\adjustbox{valign=t}{\includegraphics[width=0.23\textwidth]{#1}}}

  \begin{tabular}{cccc}
         &
        \hspace{0.5cm} \textbf{Point collection} &
        \hspace{0.5cm} \textbf{Surface reconstruction} &
        \hspace{0.5cm}\textbf{Volumetric exploration} \\
      \midrule
      \rotatebox{90}{\parbox{3.3cm}{\centering \textbf{\qquad 102815835}}} &
      \includegraphics[width=0.28\textwidth]{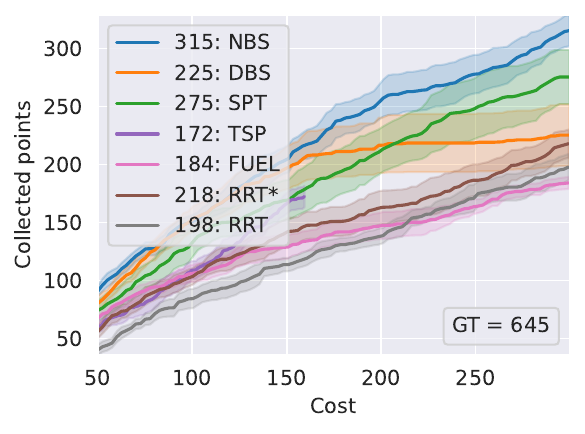} &
      \includegraphics[width=0.28\textwidth]{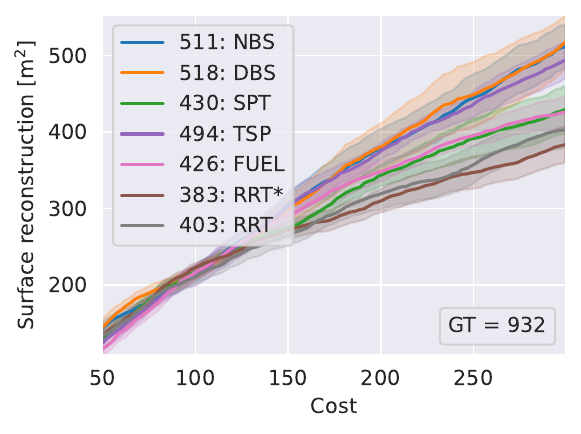} &
      \includegraphics[width=0.28\textwidth]{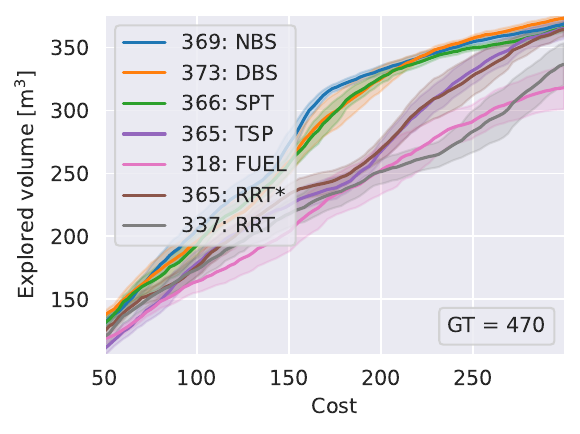} \\
      \rotatebox{90}{\parbox{3.3cm}{\centering \textbf{\qquad 102816852}}} &
      \includegraphics[width=0.28\textwidth]{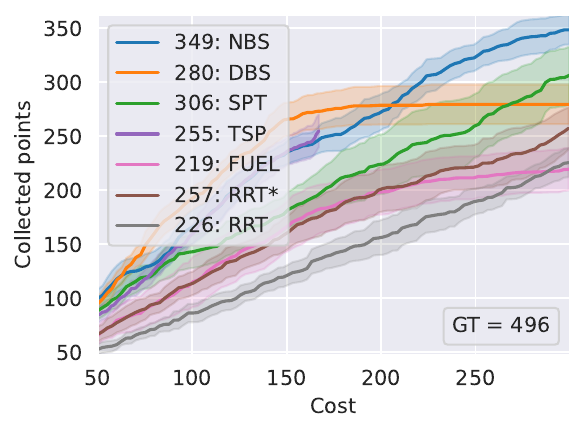} &
      \includegraphics[width=0.28\textwidth]{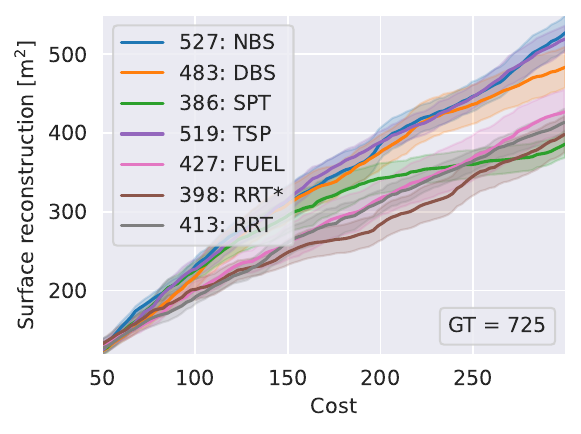} &
      \includegraphics[width=0.28\textwidth]{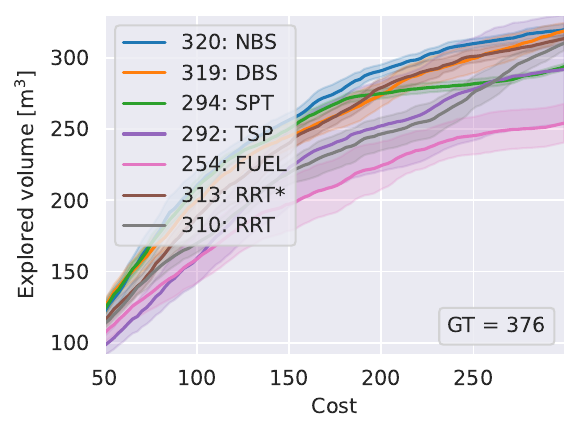} \\
      \rotatebox{90}{\parbox{3.3cm}{\centering \textbf{\qquad 102816216}}} &
      \includegraphics[width=0.28\textwidth]{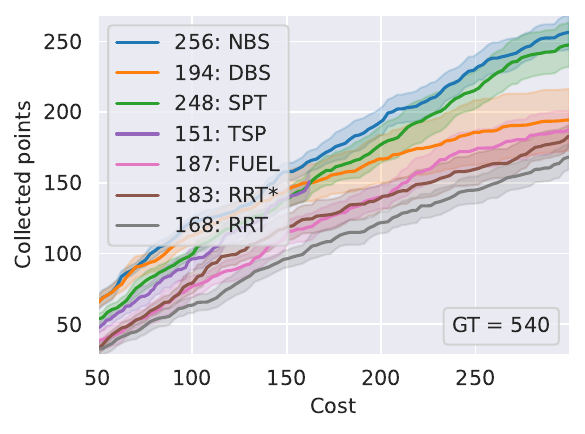} &
      \includegraphics[width=0.28\textwidth]{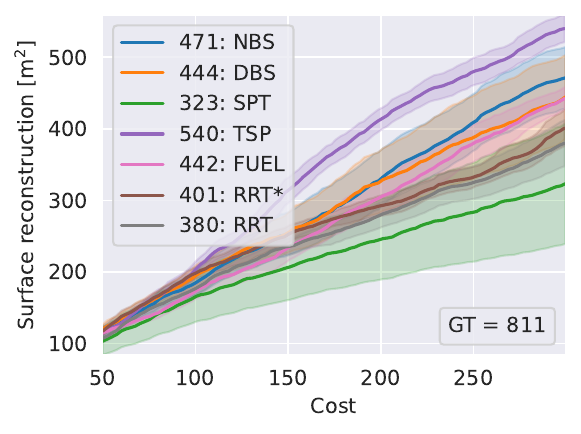} &
      \includegraphics[width=0.28\textwidth]{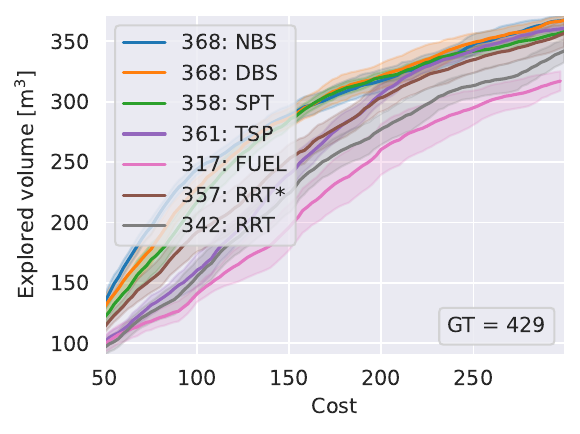} \\
      \bottomrule
  \end{tabular}
  \caption{\textbf{The mean and half the standard deviation of accumulated gain across small \ac{hssd} scenarios.} Trajectory time serves as the cost function with a 5-minute budget. Planning terminates at the 7-minute time limit.}
  \label{fig:active-perception-small-result}
\end{figure*}

\subsection{Volumetric Exploration}

In the task of volumetric exploration, the objective is to observe as much of the unknown environment as possible, which can be measured by the total number of voxels whose occupancy state is known after the exploration. This includes both free and occupied voxels that have been successfully classified. Unlike other tasks, volumetric exploration directly equates gain collection with exploration itself. Thus, there is no intrinsic need to trade off between exploration and exploitation. 

As with surface reconstruction, the true objective $ f $ is not accessible during execution. Therefore, the gain $ g $ at each node is defined as the number of unknown voxels visible from that node, a commonly used metric due to its simplicity and practical effectiveness. Formally, for any node $ v $, the gain is given by:
\begin{equation}
g(v) = | \mathset{O}(v) \cap \mathset{U}|.
\end{equation}
Following the setup for surface reconstruction, we set $ l_{\min} = \SI{1.5}{\meter} $ and $ l_{\max} = \SI{3}{\meter} $, and compute the total gain along a trajectory as the sum of node gains.

\subsection{Performance Analysis in Active Perception}

Each algorithm is evaluated across the three application scenarios over 10 independent trials, with execution terminated upon reaching a predefined maximum planning time. The resulting performance curves, showing average gain as a function of cost, are presented in \Cref{fig:active-perception-large-result} for five large environments and \Cref{fig:active-perception-small-result} for three small environments. For each curve, the mean and standard deviation of accumulated gains are calculated from linearly interpolated gain values over a cost range from 0.0 to the smallest final cost achieved by the corresponding algorithm.

Since gain varies across environments and tasks, we use normalized gain to enable comparisons across different scenarios (\Cref{fig:normalized-gain}). For each task and environment, we run both NBS and RRT twice for a sufficient duration until the gain plateaus. We take the maximum observed gain from these four runs as an estimate of the ground truth (see \Cref{fig:active-perception-large-result} and \Cref{fig:active-perception-small-result}) and normalize the gain with respect to this value. 

\begin{figure*}[!t]
    \centering
    \includegraphics[width=\textwidth]{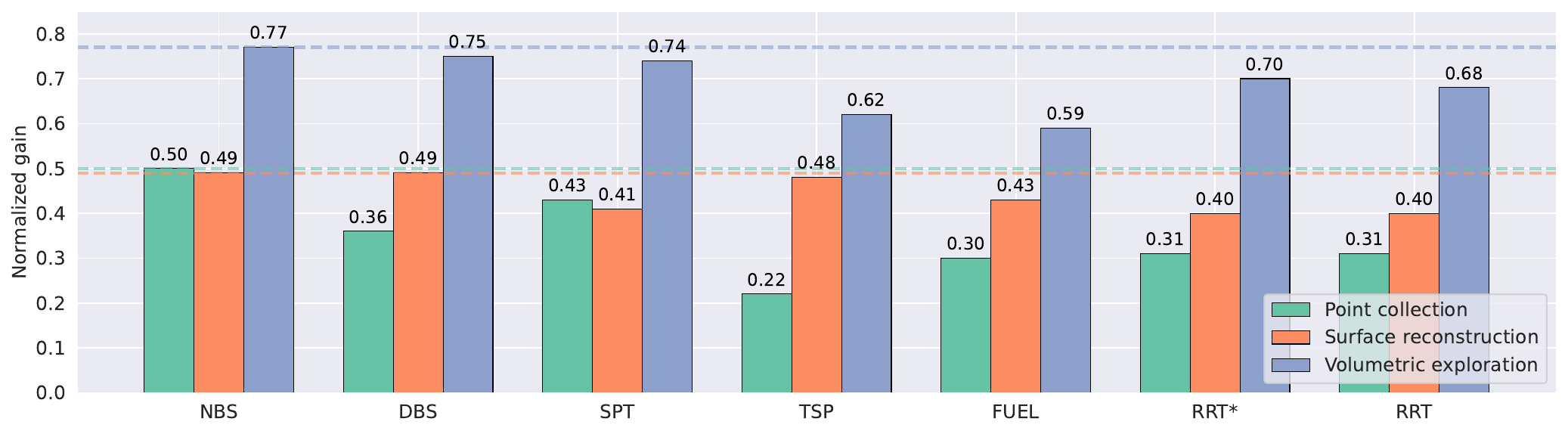}
    \caption{\textbf{Comparison of normalized gain across methods and tasks.} }
    \label{fig:normalized-gain}
\end{figure*}

The three tasks present different levels of difficulty. Exploration is the easiest, as all algorithms achieve more than 0.50 normalized gain. This is intuitive, since this task does not require the robot to exploit the known environment. The other two tasks are more challenging, as they require a balance between exploration and exploitation. Surface reconstruction has a medium level of difficulty, with the worst performance around 0.40. Point collection is the hardest, due to the smaller collection radius and consequently denser graph. The worst-performing algorithm achieves only 0.22 normalized gain.

NBS achieves the highest performance across all three tasks: 0.50 in point collection, 0.49 in surface reconstruction, and 0.77 in exploration. Moreover, NBS outperforms other algorithms by at least 20\% in one or more tasks, highlighting its robustness and effectiveness. In point collection, it demonstrates a significant improvement, surpassing the second-best algorithm by 16\% (increasing normalized gain from 0.43 to 0.50). 

DBS performs best in surface reconstruction, primarily because this task often requires the robot to thoroughly map its immediate surroundings before exploring farther. In such cases, DBS's preference for locally optimal paths becomes an advantage. However, it performs notably worse in point collection. This degradation arises from the sparser gain distribution in this setting, where DBS is prone to getting trapped in local regions with low or zero gains. This behavior reveals a fundamental limitation of DBS: its susceptibility to local optima.

SPT underperforms relative to NBS, primarily due to its constrained search space. The performance gap is particularly pronounced in point collection and surface reconstruction, where the limited candidate set often yields trajectories biased toward either exploration or exploitation, without effectively balancing the two. 

RRT* exhibits similar limitations, further exacerbated by the need to rewire the tree when the root node changes, which incurs significant computational overhead. Moreover, it lacks full orientation coverage, often requiring inefficient back-and-forth maneuvers to execute simple turns. RRT inherits some of these drawbacks: although computationally efficient, its continuous node pruning and sparse graph structure can lead to suboptimal decisions.

The TSP algorithm should, in principle, be well-suited for surface reconstruction, as TSP tends to prioritize high-gain nodes in the vicinity before expanding to other regions. However, we find that it suffers from excessive computational demands. In practice, it often fails to complete within the allotted time budget, resulting in the lowest average performance among the evaluated methods. Additionally, its performance is even worse in the other two tasks, as the decoupling of node selection and path optimization can lead to poor node choices that the TSP is then forced to traverse.

In contrast, the FUEL planner is driven by frontier information and generates a relatively sparse set of viewpoints to traverse, contributing to its time efficiency. Its overall performance trend closely resembles that of the other TSP-based planner. However, FUEL performs poorly on the point collection task, as this requires specific point-level information rather than just frontier awareness. Because FUEL actively maps all frontiers, it tends to exploit nearby viewpoints before moving to distant regions, leading to comparatively better performance in surface reconstruction than in pure exploration. Since the robot assigns only one viewpoint per location, it would need to travel back and forth to perform simple turns, which could consume the cost budget compared to our approach. 

\subsection{Hardware Experiments}

\begin{figure*}[!t]
    \begin{subfigure}{0.498\textwidth}
      \centering
      \includegraphics[width=\textwidth]{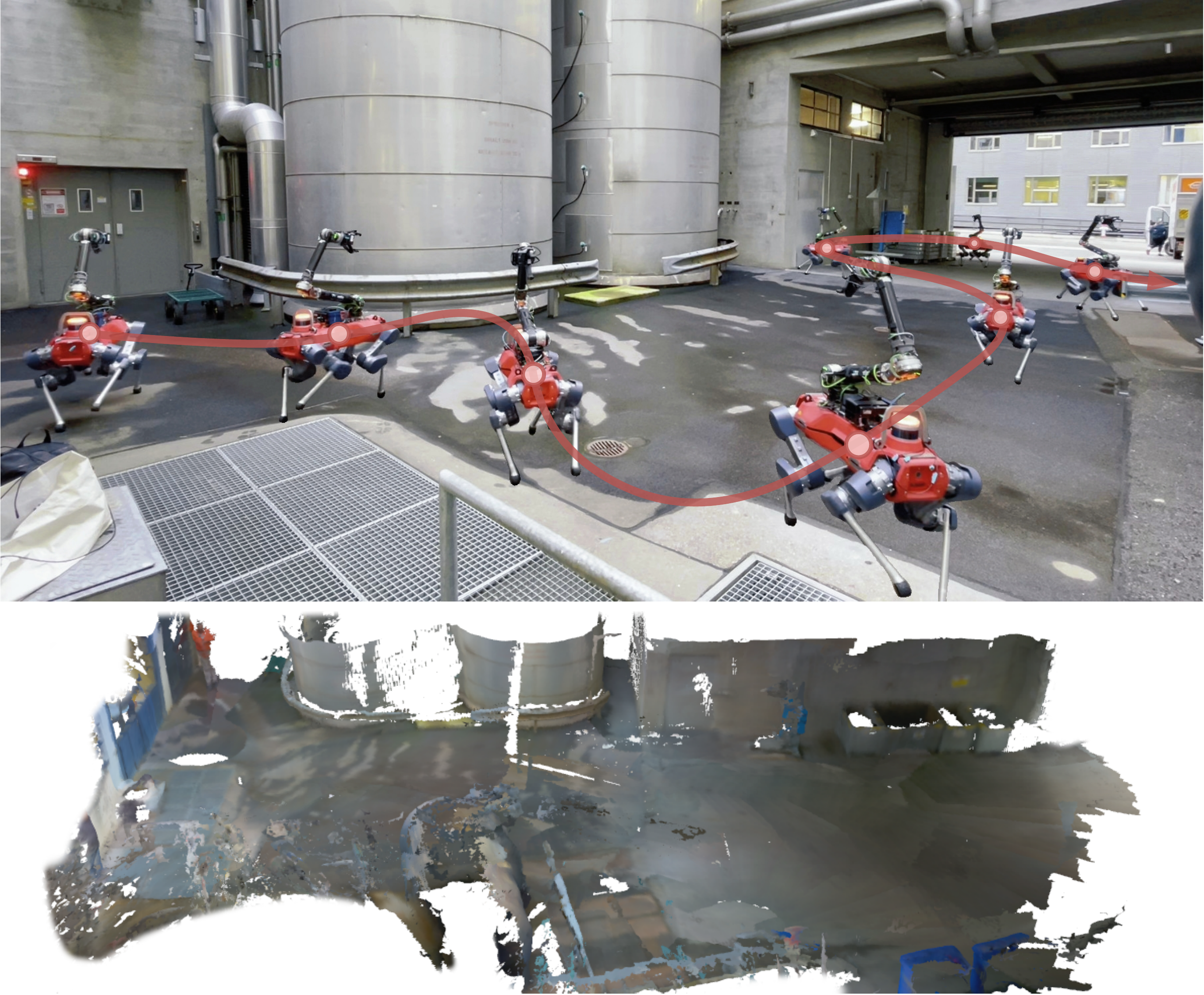}
    \caption{Surface reconstruction in an outdoor space}
      \label{fig:hardware-outdoor-fixed}
    \end{subfigure}
    \begin{subfigure}{0.498\textwidth}
        \centering
        \includegraphics[width=\textwidth]{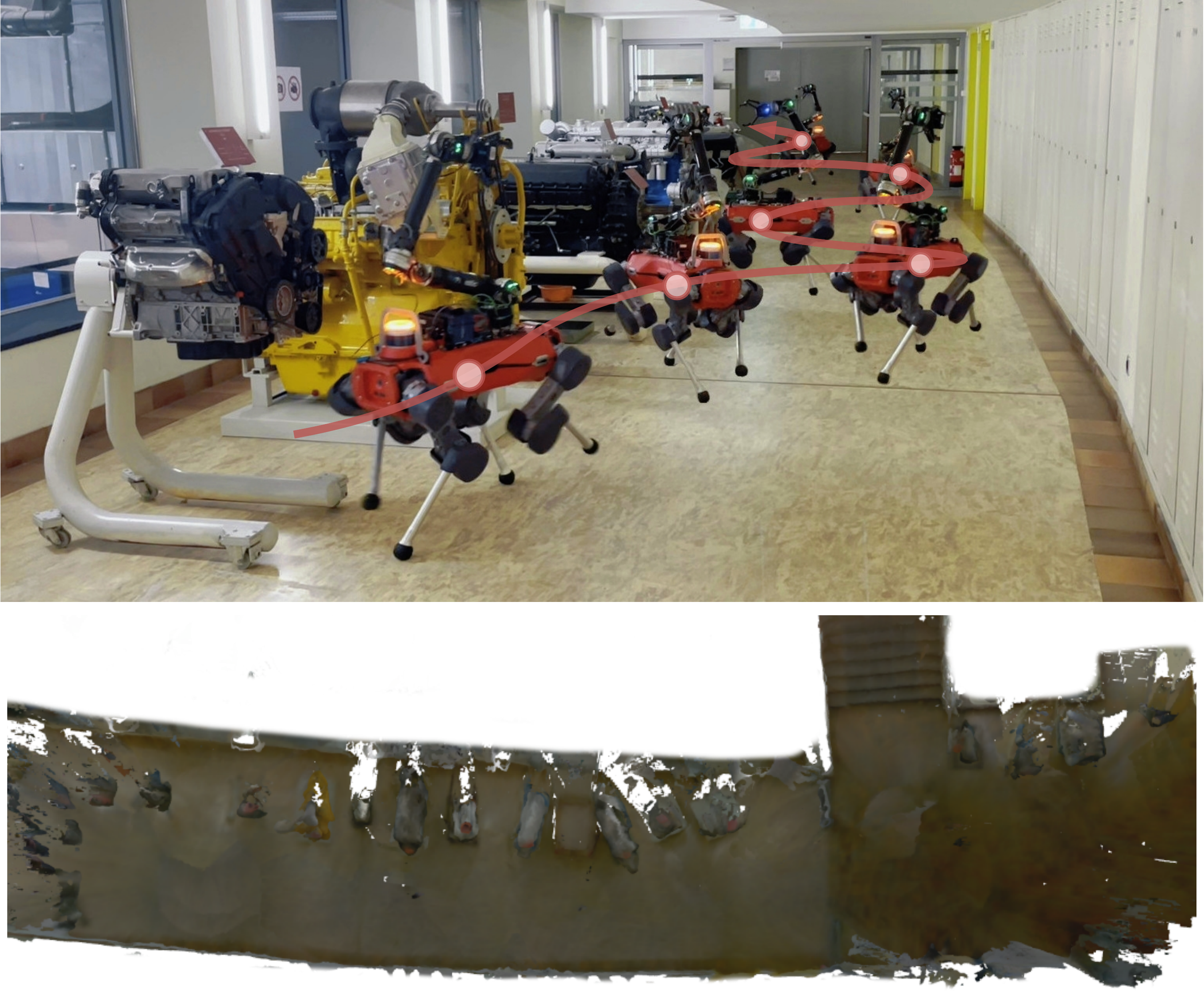}
    \caption{Surface reconstruction in an engine museum}
    \label{fig:hardware-engine-fixed}
    \end{subfigure}
    \caption{\textbf{Autonomous surface reconstruction using surface frontiers as the gain.} The robot’s trajectory is visualized with a red arrow indicating its direction of movement. The mesh reconstructed from the captured RGB-D images is also shown, with boundary artifacts removed for clarity.}
    \label{fig:hardware-fixed}
\end{figure*}

\begin{figure*}[!t]
    \begin{subfigure}{\textwidth}
      \centering
      \includegraphics[width=\textwidth]{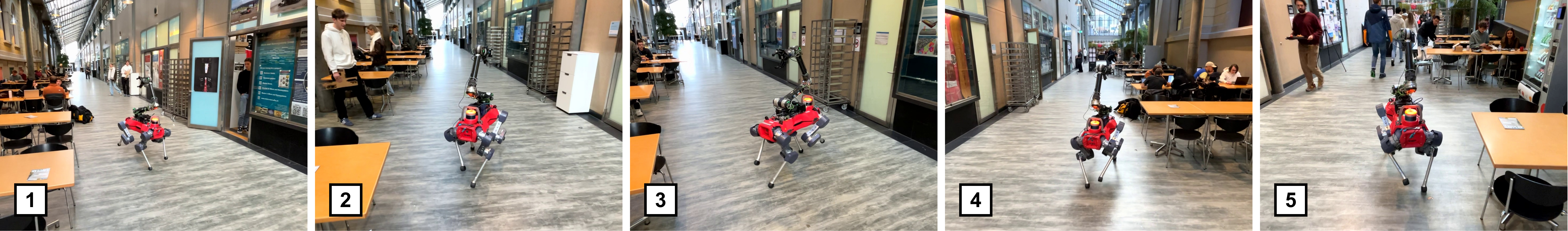}
    \caption{Volumetric exploration in a cafeteria}
    \label{fig:hardware-cafeteria-following}
    \end{subfigure}
    
    \smallskip
    
    \begin{subfigure}{\textwidth}
      \centering
    \includegraphics[width=\textwidth]{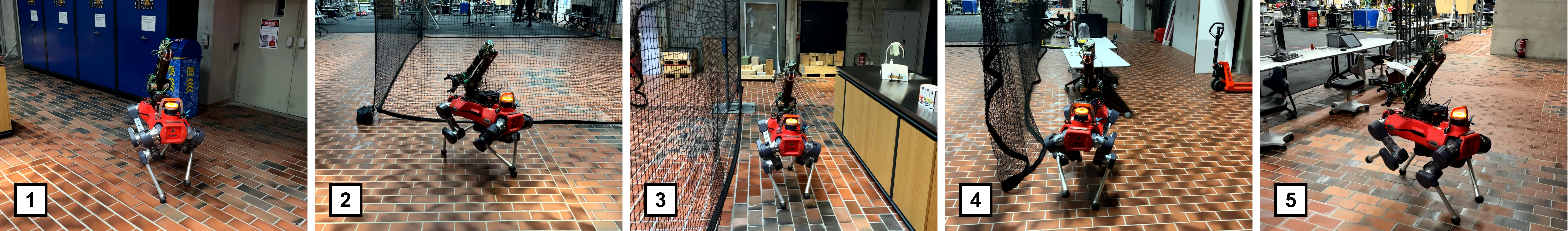}
    \caption{Volumetric exploration in a lab}
    \label{fig:hardware-lab-following}
    \end{subfigure}
    \caption{\textbf{Autonomous volumetric exploration using unknown voxels as the gain.} Experiments were conducted in both a cafeteria and a lab environment.}
    \label{fig:hardware-following}
\end{figure*}

We deploy our planner on the ANYmal quadruped robot~\citep{hutter2016anymal}, equipped with a ZED X Mini camera mounted on the front-facing end-effector, offering a field of view of $[74^\circ, 45^\circ]$. On the physical platform, depth estimation from a ZED X camera is processed onboard by an NVIDIA Jetson Orin, while the mapping and planning modules execute on an Intel Core i7 CPU.
Captured RGB-D images are converted into colorized, downsampled point clouds, which are subsequently used by Voxblox for volumetric map generation. Unlike the simulation setup, the real-world deployment constrains the planner to a two-dimensional search space at a fixed height. The collision model is adapted to better encapsulate the legged robot’s body, resulting in larger and additional collision spheres. Consequently, collision checking needs to be more optimistic: if at least one collision sphere is verified to be collision-free and none is known to be in collision, the pose is deemed traversable. Because yaw influences the collision geometry of legged robots, this checking is performed across all yaw angles at each node. 

In the first hardware experiment, the robot is tasked with autonomously reconstructing the surface of a lab. We use surface frontiers to estimate gain, allowing the robot to actively navigate toward regions with holes to complete the surface reconstruction. \Cref{fig:test_result} illustrates the reconstructed maps as time evolves. The robot simultaneously explores the environment, expands its graph, and fills in missing holes to produce a more complete surface. While following the planned path, the robot replans at a fixed frequency (omitted from the figure for simplicity). The robot successfully explores this complex and cluttered environment and reconstructs the surface densely, with minimal artifacts such as small holes. We have also conducted surface reconstruction in an outdoor space and an engine museum (\Cref{fig:hardware-fixed}) and volumetric exploration in a cafeteria and another laboratory environment (\Cref{fig:hardware-following}). These experiments demonstrate the feasibility and effectiveness of running the full planning and perception pipeline onboard the robot across diverse environments.

\section{Limitations and Future Directions}
\label{sec:limitation}

While our approach demonstrates strong performance across extensive simulation and hardware experiments, several limitations remain.
Our method is validated through empirical evaluation, and we do not provide a formal proof of optimality.
Our quantitative evaluation also focuses primarily on indoor environments, as supported by the \ac{hssd} dataset, even though we show qualitative hardware results across a broader range of settings (see \Cref{fig:test_result,fig:hardware-fixed,fig:hardware-following}).
Our current framework addresses only semantic-agnostic active perception, and extending it to more challenging semantic tasks is a natural direction for future work. Examples include object-goal navigation, where a robot must find objects in an unknown environment based on natural language instructions~\citep{chang2023goat, gervet2023navigating, shah2023navigation, qu2024ippon}, and semantic mapping~\citep{conceptfusion, takmaz2023openmask3d, gu2024conceptgraphs, zhang2025tag, Fu_2026_funfact}. In these settings, designing a semantic-aware gain function becomes nontrivial, and it would be compelling to evaluate how our planner performs under such task-driven, semantics-heavy objectives.
The present implementation also relies exclusively on ESDF-based mapping. We aim to generalize the planning framework to additional mapping paradigms, particularly 3D Gaussian Splatting~\citep{jiang2023fisherrf, kerbl20233d, goli2024bayes, chen2025activegamer}, and to incorporate learned information-gain predictors~\citep{sun2025frontiernet}.
Moreover, while we account for gain correlations in point-collection tasks, we do not model them in volumetric exploration and surface reconstruction. This is because accurately computing inter-node correlations for voxel-counting-based gains is computationally prohibitive---requiring per-voxel tracking to avoid double-counting---and would compromise real-time planning performance. Developing an efficient method to explicitly model correlations between nodes could relax the common additive node-wise gain assumption, thereby improving gain estimation and overall planning performance~\citep{yu2016correlated}.

\section{Conclusion}
\label{sec:conclusion}

We present NBS, a beam-per-node strategy designed to address the challenges of active perception in mobile robotics. Through extensive empirical evaluation, we show that NBS consistently outperforms classical \ac{tsp} and \ac{spt} baselines as well as naive beam search in graph-based environments. We further demonstrate that the expected gain selection criterion---by incorporating frontier information---achieves a better balance between exploration and exploitation, leading to higher task performance. To bridge abstract graph representations with active perception, we introduce the \ac{rrag}, which preserves orientation density and ensures graph connectivity in cluttered environments. In extensive simulations, NBS surpasses state-of-the-art planners across three representative tasks by a significant margin. We validate the practicality of our approach through hardware experiments in different scenarios using onboard computation.
All proposed algorithms are available in the open-source \openapero library to promote reproducibility and serve as a foundation for future active perception research.

\section*{Acknowledgements}
\addcontentsline{toc}{section}{Acknowledgements}

This research was supported by the Swiss National Science Foundation through the National Centre of Competence in Digital Fabrication (NCCR dfab), by Huawei Tech R\&D (UK) through a research funding agreement, and by an ETH RobotX research grant funded through the ETH Zurich Foundation. We would like to thank Jie Tan for the initial discussions, Changan Chen for helpful input, and Boyang Sun and Jiaxu Xing for proofreading the manuscript.

\bibliographystyle{SageH}
\bibliography{reference}

\end{document}